\documentclass[11pt]{article}

\usepackage[utf8]{inputenc} 
\usepackage[T1]{fontenc}    
\usepackage[margin=1in]{geometry}
\usepackage{indentfirst}

\usepackage{amsmath,amssymb,amsfonts,amsthm,mathtools}
\usepackage{nicefrac}

\usepackage[numbers]{natbib}

\usepackage[table]{xcolor} 
\usepackage{booktabs}
\usepackage{array}
\usepackage{multirow}
\usepackage{makecell}
\usepackage{arydshln}
\usepackage{rotating}

\usepackage{graphicx}
\usepackage{tikz}
\usepackage{float}
\usepackage{placeins}
\usepackage[font=footnotesize,labelformat=parens]{subcaption}

\usepackage{algorithmicx}
\usepackage{algcompatible}
\usepackage[linesnumbered,ruled,vlined]{algorithm2e}

\usepackage{enumitem}
\usepackage{listings}

\usepackage{microtype}

\usepackage{url}
\usepackage{hyperref}

\definecolor{skyblue}{HTML}{66B3FF}
\definecolor{coreflowbg}{RGB}{232,236,252}

\theoremstyle{plain}
\newtheorem{theorem}{Theorem}[section]
\newtheorem{proposition}[theorem]{Proposition}

\theoremstyle{definition}

\newtheorem{assumption}[theorem]{Assumption}
\newtheorem{remark}[theorem]{Remark}

\title{CoreFlow: Low-Rank Matrix Generative Models}
\author{Dongze Wu, Linglingzhi Zhu, Yao Xie\thanks{H. Milton Stewart School of Industrial and Systems Engineering, Georgia Institute of Technology, Atlanta, GA 30332. \texttt{<dwu381@gatech.edu>} 
\texttt{<llzzhu@gatech.edu>}
\texttt{<yao.xie@isye.gatech.edu>}}}
\date{}

\begin{document}
\maketitle
\begin{abstract}
Learning matrix-valued distributions from high-dimensional and possibly incomplete training data is challenging: ambient-space generative modeling is computationally expensive and statistically fragile when the matrix dimension is large but the sample size is limited. We propose \textit{CoreFlow}, a geometry-preserving low-rank flow model that learns shared row/column subspaces across the matrix distribution, and then trains a continuous normalizing flow only on the induced low-dimensional core. CoreFlow is designed for settings where shared low-rank matrix geometry is present, especially in high-dimensional limited-sample regimes. This separates shared matrix geometry from sample-specific variation, preserves matrix structure, and substantially improves training efficiency. The same framework also handles incomplete training matrices through masked Riemannian updates and iterative completion. Across real and synthetic benchmarks, CoreFlow substantially improves spectral and moment-level generation quality in few-sample regimes while remaining competitive in data-rich settings, even under compression to 9\% of the ambient dimension and with up to 40\% missing training entries.
\end{abstract}

\section{Introduction}

Many modern scientific and engineered datasets are naturally \emph{matrix-valued}: spatiotemporal fields in geophysics and climate reanalysis \citep{harries2020applications,store2024era5}, sensor grids \citep{ran2016tensor,wang2022spatial}, and medical imaging \citep{graff2015compressive,suetens2017fundamentals} are often represented as $M\in\mathbb{R}^{m_1\times m_2}$ arrays whose two modes encode different physical axes (e.g., space$\times$time, sensor$\times$frequency).
A core modeling goal is to learn a \emph{generative distribution} over such matrices, enabling realistic synthesis, uncertainty-aware simulation, and downstream tasks such as augmentation, denoising, and imputation.

Diffusion models \citep{ho2020denoising,song2020denoising} and flow-based generative models \citep{chen2018neural,lipmanflow2022} have become dominant for high-fidelity synthesis. However, for matrix-valued data, the ambient dimension can be tens of thousands or more, while the number of training samples may be limited. Therefore, training a diffusion or flow model in the original dimension can be not only \emph{computationally expensive}, but also statistically \emph{significantly underdetermined}: the model must estimate a complex high-dimensional distribution from too few examples.


Many matrix distributions concentrate near approximately low-rank sets: dependencies across the two modes induce shared row and column subspaces, while sample variability is expressed in a much smaller core space \citep{gupta2018matrix,mak2017information,yuchi2023bayesian}. Therefore, the intrinsic dimension can be far smaller than \(m_1m_2\). This motivates learning the shared low-rank geometry and fitting the generator only in the induced low-dimensional coordinates. Concretely, we may model each matrix as \(M_i \approx U S_i V^\top\), where \(U\) and \(V\) are shared row/column subspaces, and \(S_i\) is a sample-specific low-dimensional core. Accordingly, we position our method for settings where shared low-rank matrix geometry is present.

Real-world matrix datasets may also be incomplete due to sensor dropouts, irregular sampling, or acquisition constraints. This makes shared-geometry recovery more challenging, and a practical low-rank generative method should remain stable even when training matrices are partially observed.

To address these challenges, we propose \textit{CoreFlow}, a two-stage geometry-preserving generator for matrix distributions with shared low-rank geometry. CoreFlow first learns shared row/column subspaces, and then learns a continuous normalizing flow (CNF) only on the induced low-dimensional core distribution. Our main \emph{contributions} are: (i) \emph{Geometry-preserving low-rank matrix generation.} CoreFlow learns a shared low-rank coordinate system for the matrix distribution and trains a CNF only in the induced core space. This substantially reduces the effective generative dimension while preserving matrix structure, yielding a strong quality–efficiency tradeoff, especially in high-dimensional few-sample settings. (ii) \emph{Robust structure learning from incomplete matrices.} We develop a masked \((U,V)\) subspace estimator with alternating completion, yielding stable subspace recovery and reliable core representations under substantial missingness.

\subsection{Related Works}

\textit{Ambient- and latent-space generative modeling.} Diffusion \citep{ho2020denoising,song2020denoising} and flow-based generative models \citep{chen2018neural,lipmanflow} are powerful tools for high-fidelity synthesis and have been widely applied in areas such as statistical sampling \citep{tian2024liouville,wuannealing2025,wu2025po}, time-series generation \citep{lin2024diffusion,wu2026doflow}, and image generation \citep{peebles2023scalable,li2025back}. Most are trained in the ambient space, which can be computationally expensive and statistically fragile in high dimensions with limited samples. A related line of work improves efficiency by generating in compressed or latent representations, as in latent diffusion models \citep{rombach2022high,peebles2023scalable,zheng2025diffusion}, which typically rely on a pre-trained deep VAE. Our setting differs in two key ways: we do not rely on a generic learned latent representation, and we do not pre-train a deep VAE on possibly limited and incomplete matrices. Instead, we learn shared row and column subspaces as the geometric coordinates of the matrix distribution, and generate only in the induced core space. Closest to our setting, \citet{guo2025accelerating} studies low-dimensional diffusion through compressed sensing, but does not exploit matrix geometry, and the resulting representation may distort matrix structure. 


\textit{Generative modeling with incomplete observations.} A second line of work studies generative learning from incomplete data. MissDiff \citep{ouyang2023missdiff} enables diffusion training by masking the loss on observed entries, while other methods \citep{richardson2020mcflow,tashiro2021csdi,zheng2022diffusion,zhang2024diffputer} focus on iterative imputation with diffusion or normalizing flows. These methods are related in spirit, but primarily target missing-data imputation or ambient-space generation. Besides, they typically model the full ambient distribution, which scales poorly to high-dimensional regimes with few samples.

\textit{Generative imputation and matrix completion.}
Matrix completion is related in that it also exploits low-dimensional structure, but it targets \emph{conditional recovery} of missing entries. Existing methods include nuclear-norm and optimization approaches \citep{candes2010power,candes2012exact,koltchinskii2011nuclear,davenport2016overview,chi2019nonconvex,gulow}, as well as probabilistic models with uncertainty quantification \citep{alquier2014bayesian,mak2017information,yuchi2023bayesian,choi2024inference}. A representative example is the Singular Matrix-variate Gaussian (SMG) \citep{gupta2018matrix,mak2018maximum,mak2017information}, which uses a Gaussian low-dimensional core for completion inference. By contrast, CoreFlow targets \emph{unconditional} matrix generation: it learns shared low-rank geometry as part of generation, rather than using low-rank structure for completion or imputation.

\section{Preliminaries}

\subsection{Neural ODE and continuous normalizing flow}
\label{preliminary:flow matching}

A continuous normalizing flow (CNF) models a sample trajectory in \( \mathbb{R}^d \) by the Neural ODE \citep{chen2018neural}. Given an initial condition \( x_0 = x(0) \) at \( t=0 \), the transformation to the output \( x_1 = x(1) \) at \( t=1 \) is governed by:
\begin{equation}
\frac{dx(t)}{dt} = v_{\theta}(x(t), t), \qquad t \in [0,1],
\label{eq:ode}
\end{equation}
where \(v_{\theta}:\mathbb{R}^d\times[0,1]\to\mathbb{R}^d\) is the velocity field parametrized by a neural network. It defines an invertible map between a data distribution \(p(\cdot,0)\) at \(t=0\) and a base distribution \(p(\cdot,1)\) at \(t=1\), typically \(q:=N(0,I)\). Sampling is performed by drawing $x_1\sim q(\cdot)$ and integrating (\ref{eq:ode}) backward.

We train the CNF using Flow Matching (FM) \citep{lipmanflow2022}, which learns \(v_\theta\) by regressing it to a prescribed target velocity field \(u(x,t)\):
\begin{equation*}
\mathcal{L}_{\mathrm{FM}}
=
\mathbb{E}_{t\sim\mathcal{U}[0,1],\,x\sim p(\cdot,t)}
\bigl\|v_{\theta}(x(t),t)-u(x,t)\bigr\|_2^2.
\label{eq:fm}
\end{equation*}
Since \(u(x,t)\) is generally intractable in closed form, we use Conditional Flow Matching (CFM), which conditions on a data sample \(x_0\sim p(\cdot,0)\) and a noise sample \(x_1\sim q\), and defines an analytic reference path between them. We adopt the standard linear interpolant \citep{lipmanflow2022}:
\begin{equation*}
\phi(x_0,x_1;t)=(1-t)x_0+t x_1,
\label{eq:linear_interp}
\end{equation*}
whose reference velocity is
\begin{equation*}
\frac{d\phi}{dt}=x_1-x_0.
\label{eq:ref_velocity}
\end{equation*}
CFM then trains \(v_\theta\) by matching this reference velocity:
\begin{equation*}
\mathcal{L}_{\mathrm{CFM}}
=
\mathbb{E}_{t\sim\mathcal{U}[0,1],\,x_0\sim p(\cdot,0),\,x_1\sim q}
\left\|
v_\theta\bigl(\phi(x_0,x_1;t),t\bigr) - (x_1-x_0)
\right\|_2^2.
\label{eq:cfm}
\end{equation*}
In our method, CNF training is carried out only in the learned low-dimensional core space, substantially reducing the generative dimension while preserving geometry-aware transport.


\subsection{Optimization on the Stiefel manifold}

Low-rank matrix methods often represent row and column subspaces by matrices with orthonormal columns. Since standard Euclidean updates do not preserve this constraint, we optimize on the Stiefel manifold \citep{absil2008optimization,edelman1998geometry}, $\mathrm{St}(m,R)=\{W\in\mathbb{R}^{m\times R}:W^\top W=I_R\}.$

Given a Euclidean gradient \(G\) at \(W\in\mathrm{St}(m,R)\), we project \(G\) onto the tangent space at \(W\), take a first-order step, and retract back to the manifold. We use QR retraction, which is simple and numerically stable. Algorithm~\ref{alg:stiefel-step} summarizes this procedure as \textsc{StiefelStep}\((W,G,\eta)\).


\section{Algorithm}

\begin{figure}[h!]
    \centering
\includegraphics[width=0.9\textwidth]{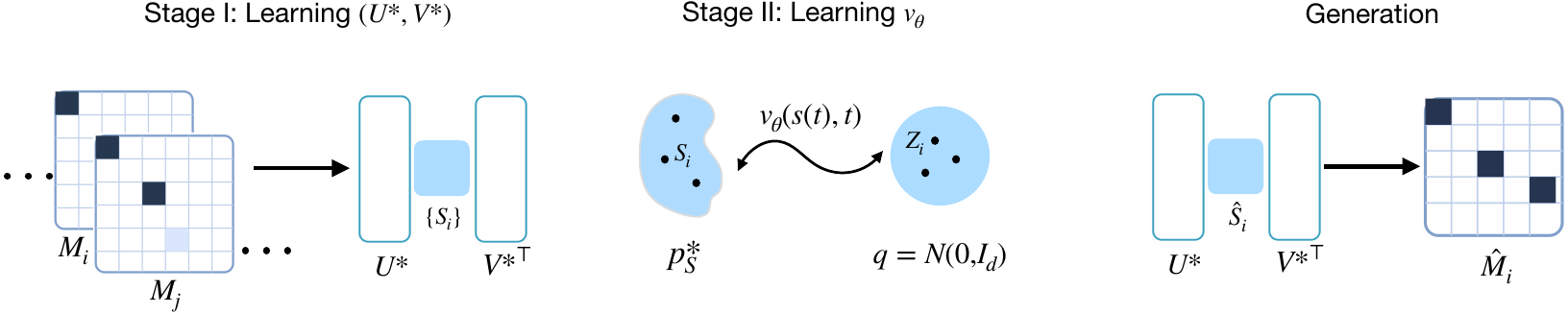}
    \caption{\small Overview of CoreFlow. \textit{Left:} learning shared row/column subspaces \((U^\ast, V^\ast)\) and constructing induced low-dimensional cores \(\{S_i\}\). \textit{Middle:} learning a continuous normalizing flow in the induced core space, matching the Stage-I cores \(\{S_i\}\sim p_S^\ast\) to a Gaussian base. \textit{Right:} generation starts from Gaussian noise, reverses the learned flow to obtain a generated core \(\hat S_i\), and decodes it with \((U^\ast, V^\ast)\) to produce a new matrix \(\hat M_i\).}
    \label{fig:illustration}
\end{figure}

\subsection{Settings and Goal}

Let \(M_i \in \mathbb{R}^{m_1 \times m_2}\) denote a matrix-valued sample drawn from an unknown data-generating distribution \(p_M\) (i.e., $M\sim p_M$). CoreFlow models each matrix through a shared low-rank parameterization
\[
M_i = U S_i V^\top,
\]
where \(U \in \mathrm{St}(m_1,R)\) and \(V \in \mathrm{St}(m_2,R)\) capture row and column subspaces shared across the matrix distribution, and \(S_i \in \mathbb{R}^{R \times R}\) is a sample-specific core matrix.

In the incomplete-data setting, we observe \(N\) i.i.d. training instances \(\{(M_i^{\mathrm{obs}},\Omega_i)\}_{i=1}^N\), where \(\Omega_i \subseteq [m_1]\times[m_2]\) indexes the observed entries and
\[
M_i^{\mathrm{obs}} = \mathcal{P}_{\Omega_i}(M_i),
\qquad
(\mathcal{P}_{\Omega}(M))_{jk} =
\begin{cases}
M_{jk}, & (j,k)\in\Omega,\\
0, & (j,k)\notin\Omega.
\end{cases}
\]
The complete-data setting is recovered when \(\Omega_i=[m_1]\times[m_2]\).


\textbf{Goal.} Our goal is to learn an unconditional generative model for complete matrices that matches \(p_M\) while avoiding ambient-space training in the full \(m_1m_2\)-dimensional space, which can be computationally costly and statistically fragile in high-dimensions. When training matrices are partially observed, we develop a Stage-I procedure that still recovers the shared matrix geometry. We first estimate \((U,V)\), and then learn a generative model only on the induced low-dimensional core \(S\).

\textbf{Notation.} Each matrix sample is denoted by $M_i\in\mathbb{R}^{m_1\times m_2}$.
For incomplete data, we observe $(M_i^{\mathrm{obs}},\Omega_i)$, where
$M_i^{\mathrm{obs}}=\mathcal{P}_{\Omega_i}(M_i)$.
Let $d=R^2$ denote the core-vector dimension.
For a core matrix $S\in\mathbb{R}^{R\times R}$, we write $s=\mathrm{vec}(S)\in\mathbb{R}^d$ for its vectorization; conversely, for any $s\in\mathbb{R}^d$, $\mathrm{mat}(s)\in\mathbb{R}^{R\times R}$ denotes its matricization, i.e., the inverse of $\mathrm{vec}(\cdot)$. We use \(p_M\) and \(\hat p_M\) to denote the true and the learned generated distribution of the ambient matrix \(M\), respectively.
The standard Gaussian base distribution in the core space is denoted by \(q=\mathcal{N}(0,I_d)\).


CoreFlow is trained in two stages. Stage~I learns \emph{shared} row and column subspaces
$(U,V)$ with orthonormal columns that explain the low-rank geometry
across the matrix distribution. Stage~II then learns a continuous normalizing
flow (CNF) only on the induced low-dimensional \emph{core} coordinates $S$, which
can greatly reduce the learnable dimensions.

\subsection{Stage I: learning shared matrix geometry on the Stiefel manifolds}
\label{subsec:stage1}

\paragraph{Complete matrices.}
In Stage I, we learn orthonormal factors $U\in\mathrm{St}(m_1,R)$ and $V\in\mathrm{St}(m_2,R)$ such that
every sample of $p_M$ can be well-approximated by the rank-$R$ model
$M_i= U S_i V^\top$.
Given $(U,V)$, the core matrix can be re-written as
$S_i\;:=\;U^\top M_i V\in\mathbb{R}^{R\times R}$,
so the reconstruction becomes $U\,S_i\,V^\top = UU^\top M_i VV^\top$.
Therefore, we can estimate $(U,V)$ by minimizing the reconstruction error:
\begin{equation}
\min_{U\in\mathrm{St}(m_1,R),\,V\in\mathrm{St}(m_2,R)}
\;\mathcal{L}_{\mathrm{rec}}(U,V)
:=\frac{1}{N}\sum_{i=1}^N \bigl\|M_i-UU^\top M_i VV^\top\bigr\|_F^2.
\label{eq:stage1_complete_revised}
\end{equation}
Because $U$ and $V$ must remain orthonormal during optimization, we use
Riemannian gradient steps on the Stiefel manifold with a QR retraction
(Alg.~\ref{alg:stiefel-step}).

\paragraph{Incomplete matrices.}
When entries are missing, directly evaluating $\|M_i-U S_i V^\top\|_F^2$ is impossible.
Therefore, we introduce a latent \emph{filled-in} matrix $\widetilde M_i$ for each training sample and
optimize a masked objective on observed entries:
\begin{equation}
\mathcal{L}_{\mathrm{miss}}(U,V;\{\widetilde M_i\})
:=\frac{1}{N}\sum_{i=1}^N
\left\|
\mathcal{P}_{\Omega_i}\!\left(
M_i^{\mathrm{obs}}-UU^\top\widetilde M_iVV^\top
\right)
\right\|_F^2.
\label{eq:missing_obj_revised}
\end{equation}
We then alternate:
(i) update $(U,V)$ using gradients of \eqref{eq:missing_obj_revised} (holding $\widetilde M_i$ fixed),
and (ii) update $\widetilde M_i$ by keeping observed entries fixed and imputing the missing ones
with the current low-rank reconstruction:
\begin{equation*}
\widetilde M_i \leftarrow
\mathcal{P}_{\Omega_i}(M_i^{\mathrm{obs}})
+\mathcal{P}_{\Omega_i^c}\!\left(UU^\top\widetilde M_iVV^\top\right).
\label{eq:fill_update_revised}
\end{equation*}
Intuitively, Step (ii) produces a completed matrix consistent with observations while staying on the
rank-$R$ manifold induced by $(U,V)$; Step (i) then refines the shared subspaces using only reliable
observed entries. The full procedure is summarized in Alg.~\ref{alg:stage1-missing} of Appendix.

\subsection{Stage II: geometry-preserving core-space flow matching}
\label{subsec:stage2}

We denote by $(U^*,V^*)$ the shared subspaces learned in Stage I. After Stage~I, $(U^*,V^*)$ are fixed and each (completed) training matrix is mapped to a
low-dimensional core:
\begin{equation*}
S \;:=\; {U^*}^\top M V^*\in\mathbb{R}^{R\times R},
\qquad
s \;:=\; \mathrm{vec}(S)\in\mathbb{R}^{d}.
\end{equation*}
Correspondingly, we denote by \(p_S^*\) the distribution of the induced core matrix
\(S\) after learning the subspaces \((U^*,V^*)\), and by
\(p_s^* := \mathrm{vec}_{\#}p_S^*\) the induced distribution of the core vector
\(s=\mathrm{vec}(S)\).

A continuous normalizing flow (CNF) defines a transformation between a data
distribution and a base distribution via a Neural ODE indexed by time $t\in[0,1]$.
It transports samples from the data distribution at $t{=}0$ (i.e., $s(0)\sim p_s^*$)
to the base distribution at $t{=}1$ (typically $\mathcal{N}(0,I_{d})$) according to
\begin{equation}
    \frac{ds(t)}{dt}=v_{\theta}(s(t),t),\quad t\in [0,1],
    \label{equation: our neural ode}
\end{equation}
where the trajectory is initialized at a data sample $s(0)=s$, and ends at
$s(1)=z$ with $z\sim \mathcal{N}(0,I_{d})$.

Flow matching trains a time-dependent velocity field $v_\theta(\cdot,t)$ on $\mathbb{R}^{d}$
by matching an analytic reference velocity along a simple coupling between
Gaussian noise and data cores, as introduced in Section \ref{preliminary:flow matching}. Specifically, we sample $s=s(0)\sim p_s^*$, $z=s(1)\sim\mathcal{N}(0,I_{d})$, $t\sim\mathrm{Unif}(0,1)$, and then define the straight-line reference path and its derivative:
\begin{equation*}
\phi(s,z;t) := (1-t)s + t z,
\qquad
\frac{d\phi}{dt} = z - s.
\end{equation*}
The flow-matching objective regresses the model velocity to this target:
\begin{equation*}
\mathcal{L}_{\mathrm{FM}}(\theta)
=
\mathbb{E}_{t\sim \mathcal{U}[0,1],\;s\sim p_s^*,\; z\sim\mathcal{N}(0,I_{d})}
\left\|
v_\theta\!\left(\phi(s,z;t),\,t\right) - (z-s)
\right\|_2^2.
\label{eq:fm-objective}
\end{equation*}

\paragraph{Generating a new matrix.} To generate a new matrix, we draw $z\sim\mathcal{N}(0,I_{d})$ and numerically solve the Neural ODE in \eqref{equation: our neural ode} backward from $t=1$ to $t=0$, producing a generated core vector $\hat s\sim \hat p_s$.
After that, we project the core vector $\hat s$ back to the ambient matrix space:
\begin{equation*}
\widehat M \;=\; U^*\,\mathrm{mat}(\hat s)\,{V^*}^\top.
\end{equation*}

\begin{algorithm}[t]
\caption{\small Two-Stage Training for \textsc{CoreFlow}}
\label{alg:coreflow}
\small
\begin{algorithmic}[1]
\REQUIRE Training set $\{M_i\}_{i=1}^N\subset\mathbb{R}^{m_1\times m_2}$ (fully observed);
rank $R$; iterative steps $T_{\mathrm{sub}},T_{\mathrm{flow}}$; learning rates $\eta_U,\eta_V,\eta_\theta$.
\ENSURE $U^*\in\mathbb{R}^{m_1\times R},V^*\in\mathbb{R}^{m_2\times R}$ (orthonormal columns), flow parameters $\theta^*$.
\STATE Initialize $U\in\mathrm{St}(m_1,R)$ and $V\in\mathrm{St}(m_2,R)$ via spectral initialization.
\vspace{0.15em}

\STATE \textbf{Stage I: learn shared subspaces $(U^*,V^*)$}
\FOR{$t=1$ to $T_{\mathrm{sub}}$}
  \STATE Sample mini-batch $\mathcal{B}\subset\{1,\dots,N\}$.
  \STATE $\mathcal{L}_{\mathrm{rec}} \gets \frac{1}{|\mathcal{B}|}\sum_{i\in\mathcal{B}}\|M_i-U U^\top M_i V V^\top\|_F^2$.
  \STATE Compute $G_U\gets\nabla_U\mathcal{L}_{\mathrm{rec}}$, $G_V\gets\nabla_V\mathcal{L}_{\mathrm{rec}}$.
  \STATE $U \gets \textsc{StiefelStep}(U,G_U,\eta_U)$;\quad
         $V \gets \textsc{StiefelStep}(V,G_V,\eta_V)$.
\ENDFOR
\STATE $U^*\gets U,\;V^*\gets V$.
\vspace{0.15em}

\STATE \textbf{Stage II: train CNF on core vectors under fixed $(U^*,V^*)$}
\FOR{$t=1$ to $T_{\mathrm{flow}}$}
  \STATE Sample mini-batch $\mathcal{B}\subset\{1,\dots,N\}$.
  \FORALL{$i\in\mathcal{B}$}
    \STATE $S_i \gets {U^*}^\top M_i V^*$;\quad $s_i\gets \mathrm{vec}(S_i)\in\mathbb{R}^{d}$.
    \STATE Sample $z_i\sim\mathcal{N}(0,I_{d})$, $t_i\sim\mathrm{Unif}(0,1)$;
    set $\phi(s_i,z_i;t_i) \gets (1-t_i)s_i+t_i z_i$.
  \ENDFOR
  \STATE $\mathcal{L}_{\mathrm{FM}} \gets \frac{1}{|\mathcal{B}|}\sum_{i\in\mathcal{B}}\|v_\theta(\phi(s_i,z_i;t_i),t_i)-(z_i-s_i)\|_2^2$.
  \STATE $\theta \gets \theta-\eta_\theta\nabla_\theta \mathcal{L}_{\mathrm{FM}}$.
\ENDFOR
\STATE $\theta^*\gets \theta$; \textbf{return} $(U^*,V^*,\theta^*)$.
\end{algorithmic}
\end{algorithm}


\begin{algorithm}[t]
\caption{\small \textsc{StiefelStep}$(W,G,\eta)$ with QR retraction (generic)}
\label{alg:stiefel-step}
\small
\begin{algorithmic}[1]
\REQUIRE $W\in\mathbb{R}^{m\times R}$ with $W^\top W=I_R$; Euclidean gradient $G$; step size $\eta$.
\STATE $\mathrm{Sym}(A)\triangleq \tfrac12(A+A^\top)$.
\STATE \textbf{Tangent projection:} $G^{\mathrm{R}} \gets G - W\,\mathrm{Sym}(W^\top G)$.
\STATE \textbf{Retract via QR:} $\widetilde W \gets W-\eta G^{\mathrm{R}}$; compute thin QR $\widetilde W=QR$; set $W\gets Q$.
\STATE \textbf{return} $W$.
\end{algorithmic}
\end{algorithm}



\begin{remark}[Optional patchification]
\label{remark:patching}
CoreFlow is most effective when the input matrices are approximately low-rank. For datasets without strong global low-rank structure but with repeated local patterns, we optionally apply a patchification step before Stage~I, inspired by patch-based vision methods \citep{buades2005non,dabov2007image,dosovitskiy2020image}. Rearranging the matrix into patches can better align local structure, increase row/column correlation, and reduce the effective rank. See details in Appendix~\ref{appendix:patching}.
\end{remark}

\section{Theoretical Properties}
\label{section:theory}

Section \ref{section:theory} analyzes the fully observed population model. Our aim is to isolate the geometry–generation decomposition of CoreFlow: ambient generation error separates into subspace error from empirical geometric recovery, core-generation error from empirical flow matching, and an irreducible approximation error. The incomplete-data Stage-I updates in Section \ref{subsec:stage1} and Appendix \ref{appendix:stage 1 incomplete} should be viewed as practical extensions, motivated by real applications and supported empirically in Section \ref{section:experiments}. Developing a statistical recovery theory under missing observations is left for future work.

\begin{assumption}[Approximate Shared Low-Rank Model]
\label{assumption:shared subspace}
The data-generating distribution $p_M$ produces matrices $M \in \mathbb{R}^{m_1 \times m_2}$ of the form
\begin{equation*}
    M = U_0 S_0 V_0^\top + E,
\end{equation*}
where (i) \textit{True Subspaces:} $U_0 \in \mathbb{R}^{m_1 \times R}$ and $V_0 \in \mathbb{R}^{m_2 \times R}$ are fixed unknown matrices with orthonormal columns.
(ii) \textit{Oracle Core Matrix:} $S_0 \in \mathbb{R}^{R \times R}$ is a random core matrix in the oracle coordinates $(U_0,V_0)$, with non-degenerate second moments, $\Sigma_L := \mathbb{E}[S_0S_0^\top] \succ 0$ and $\Sigma_R := \mathbb{E}[S_0^\top S_0] \succ 0$, and $\mathbb E\|S_0\|_F^2\le R C_S^2.$
(iii) \textit{Perturbation:} 
    $E \in \mathbb{R}^{m_1 \times m_2}$ is a random perturbation satisfying $\mathbb{E}[E \mid S_0] = 0$ and $\mathbb{E}\|E\|_2^2 \leq \tau^2$, where $\|\cdot\|_2$ denotes the spectral norm.    
\end{assumption}
\vspace{-0.5em}
Next, the proposition below states that the Stage I objective (\ref{eq:stage1_complete_revised}) has the correct global minima.

\begin{proposition}
Under Assumption \ref{assumption:shared subspace}, if the data is exactly low-rank, the global optimizers of the population reconstruction loss $\mathcal{L}_{\mathrm{rec}}(U,V)$ exactly recover the true subspaces up to orthogonal rotations:
 \begin{equation*}
\operatorname*{arg\,min}_{\substack{U^\top U = I,\ V^\top V = I}}\ \mathcal{L}_{\mathrm{rec}}(U,V)=\{(U_0R_1,\;V_0R_2): R_1\ \text{and } R_2\ \text{are orthonormal}\}.
    \end{equation*}
\label{proposition: global opt}
\end{proposition}
\vspace{-1em}
Also, Proposition~\ref{proposition:tucker init} shows that the Tucker initialization is provably close to the true subspace.

\begin{proposition}[Tucker initialization]
\label{proposition:tucker init}
Suppose that Assumption \ref{assumption:shared subspace} holds under the exact low-rank setting.
Define the population and empirical left second moments by
$C_L := \mathbb{E}[MM^\top]$ and
$\widehat{C}_L := \frac{1}{N} \sum_{i=1}^N M_i M_i^\top$,
where $M_i \stackrel{\mathrm{i.i.d.}}{\sim} p_M$.
Let $U^{(0)}\in\mathbb{R}^{m_1\times R}$ be the matrix composed of the top-$R$ eigenvectors of $\widehat{C}_L$.
Suppose $\|\widehat{C}_L - C_L\|_2 < \lambda_R(\Sigma_L)$, then the principal-angle error satisfies:
\begin{equation*}
\sin\Theta(U^{(0)},U_0)
:= \big\|\big(I - U_0 U_0^\top\big)U^{(0)}\big\|_2\le
\frac{\|\widehat{C}_L - C_L\|_2}{\lambda_{R}(\Sigma_L) - \|\widehat{C}_L - C_L\|_2}.
\end{equation*}
\end{proposition}


\begin{remark}
\label{remark:spectral_vs_iterative}
The same conclusion holds for the right subspace after replacing
\((C_L,\widehat{C}_L,U_0,U^{(0)},\Sigma_L)\) with
\((C_R,\widehat{C}_R,V_0,V^{(0)},\Sigma_R)\) in Proposition~\ref{proposition:tucker init}, where
\(C_R:=\mathbb{E}[M^\top M]\) and \(\widehat{C}_R:=\frac{1}{N}\sum_{i=1}^N M_i^\top M_i\).
The Tucker initializer is close to the true subspaces when the empirical second moments are accurate. Algorithm~\ref{alg:coreflow} then refines these subspaces through Stiefel-gradient updates. This refinement is particularly useful under missingness, where alternating completion improves the filled-in matrices and leads to more reliable subspace updates.
\end{remark}

We denote $c_d := (2\pi)^{-d/2}$ by the normalizing constant of $q := \mathcal{N}(0, I_d)$. Let $v$ and $\hat{v}$ be the target and learned velocity fields on $\mathbb{R}^d \times [0,1]$, with corresponding density paths $\rho_t$ and $\hat{\rho}_t$, $t\in [0,1]$, both initialized at $\rho_0 = \hat{\rho}_0 = p_s^*$. Note that \(p_s^*\) denotes the core-vector distribution induced by the Stage-I learned subspaces. The target field $v$ transports $p_s^*$ to $q$, i.e., $\rho_1 = q$. Let $\hat{T}: \mathbb{R}^d \to \mathbb{R}^d$ be the flow map induced by the learned ODE $\dot{x} = \hat{v}(x,t)$. We define the terminal density $\hat{\rho}_1 := \hat{T}_\# p_s^*$ and the generated core density $\hat p_s := \hat{T}^{-1}_\# q$.







\begin{assumption}
\label{asmp:fm_reg}
There exist constants $\varepsilon > 0$ and $C_1, C_2, L, L_v > 0$ 
such that for all $t \in [0,1]$: \textbf{(i)} \textit{Flow-matching error:} The learned velocity field satisfies the global $L^2$ error bound: $\int_0^1\int_{\mathbb{R}^d} \|v(x,t)-\hat v(x,t)\|^2\rho_t(x)dx dt\le \varepsilon^2.$
    \textbf{(ii)} \textit{Velocity regularity:} Both velocity fields are Lipschitz in space uniformly in time:
    $\|v(x,t) - v(y,t)\| \le L_v\|x - y\|,\ \|\hat{v}(x,t) - \hat{v}(y,t)\| \le L_v\|x - y\|, \forall\, t \in [0,1].$
    \textbf{(iii)} \textit{Density bounds:} Both density paths are strictly positive and bounded by Gaussian-like tails:
    $0 < \rho_t(x),\hat\rho_t(x) \le C_1 e^{-\|x\|^2/2}.$ \textbf{(iv)} \textit{Score regularity:} The score functions of both density paths exhibit at most linear growth:
    $\|\nabla\log\rho_t(x)\| \le L(1+\|x\|),\ \|\nabla\log\hat\rho_t(x)\| \le L(1+\|x\|).$
    \textbf{(v)} \textit{Density ratio integrability:} The density paths satisfy the bounded integral condition:  
    $\int_{\mathbb{R}^d} (1+\|x\|)^2 \frac{\rho_t(x)^3}{\hat\rho_t(x)^2}dx \le C_2.$
\end{assumption}

Following the flow-matching analysis of \citet{xu2026local}, we work under the same population-level regularity conditions and use their forward \(\chi^2\) control as the starting point for Proposition~\ref{prop:coreflow_fm_chi2}. Under these assumptions, Proposition~\ref{prop:coreflow_fm_chi2} is a conditional population guarantee rather than a finite-sample guarantee for the trained neural CNF. Our contribution is to turn this core-space control into a backward generation bound for the Stage-II core model, and then combine it with the Stage-I subspace recovery analysis to obtain an end-to-end ambient-space error decomposition for the full CoreFlow algorithm in Theorem~\ref{thm:end_to_end}.

\begin{proposition}[Core-space flow-matching and generation guarantee]
\label{prop:coreflow_fm_chi2}
Let Assumption~\ref{asmp:fm_reg} hold. Then there exists a constant $C_F>0$ depending on $(C_1,C_2,L,d)$ such that
\begin{equation*}
\chi^2(\hat\rho_1\|q) \le C_F\varepsilon,
\qquad
\chi^2(p_s^*\|\hat p_s) \le C_F\varepsilon.
\end{equation*}
\end{proposition}

\begin{theorem}[End-to-end geometry--generation guarantee]
\label{thm:end_to_end}
Let Assumption~\ref{assumption:shared subspace} and Assumption~\ref{asmp:fm_reg} hold.
Let $(U^*, V^*)$ be the learned subspaces.
Define the generated ambient distribution $\hat{p}_M$ by sampling $\hat{s} \sim \hat p_s$ and forming $\hat{M} = U^*\mathrm{mat}(\hat s)V^{*\top}$.
Then the Wasserstein-$2$ distance between the true data distribution $p_M$ and the generated distribution satisfies:
\begin{equation*}
\small
W_2(\hat p_M,p_M)
\le
\underbrace{\sqrt{R}C_S\!\left(\sin\Theta(U^*,U_0)+\sin\Theta(V^*,V_0)\right)}_{\textnormal{(a)}}
+
\underbrace{e^{L_v}\sqrt{2C_F}\,\varepsilon^{1/2}}_{\textnormal{(b)}}
+
\underbrace{\sqrt{\min(m_1,m_2)}\,\tau}_{\textnormal{(c)}}.
\end{equation*}
Here, \textnormal{(a)} is the geometric error from empirical subspace recovery mismatch, \textnormal{(b)} is the generative error from core-space flow matching, and \textnormal{(c)} is the irreducible approximation error from data noise.
\end{theorem}


\begin{remark}[Error balancing and  bottlenecks]
\label{remark:balance}
The core-flow error term is the only term that does not scale with the ambient matrix dimension $m_1m_2$, though its constant $C_F=C_F(d)$ depend on the core dimension $d$. For efficient training, it suffices to reduce the flow error until it matches the dominant geometric or noise error. From Proposition~\ref{proposition:tucker init}, combined with matrix Bernstein bounds \citep{tropp2015introduction}, the spectral estimator $(U^*,V^*)$ satisfies with high probability: \[\sin \Theta(U^*, U_0)+\sin \Theta(V^*, V_0) \lesssim \frac{C_S^2 R(\sqrt{m_1}+\sqrt{m_2})}{\sqrt{N}\min (\lambda_R(\Sigma_L), \lambda_R(\Sigma_R))}.\] Balancing the flow term against the other two yields the sufficient flow precision threshold: $\varepsilon \lesssim \max(m_1,m_2)(N^{-1}+\tau^2)/C_F(d)$. Thus the flow precision is bottlenecked by the core dimension, statistical subspace recovery, and irreducible approximation noise.
\end{remark}

\section{Experiments}
\label{section:experiments}

\textbf{Goals.}
Our experiments are designed to demonstrate that CoreFlow improves generative quality while training in a much lower-dimensional space than the original matrices, especially in high-dimensional and limited-sample settings. We additionally evaluate training matrices with substantial missing entries as robustness tests of the same framework.


\textbf{Datasets.} We consider both real and synthetic matrix-generation benchmarks. For real data, we study two sharply different regimes. Solar and Solar 2 are pseudo-color solar-flare datasets whose samples are $200\times 200$ matrices, but only 300 training matrices are available, making ambient-space generative training particularly challenging. We also evaluate on LSPF, an $80\times 80$ large-scale precipitation fraction field from ERA5, with 8760 training samples, representing a data-rich regime.

To complement the real datasets, where the underlying low-rank structure is unknown, we use four synthetic cases: \textit{Blobs} (Figure~\ref{fig:case blobs}), \textit{Bands} (Figure~\ref{fig:case bands}), \textit{Waves} (Figure~\ref{fig:case waves}), and \textit{Crosshatch} (Figure~\ref{fig:case crosshatch}), each consisting of $200\times 200$ matrices with $N=1000$ training samples. These simulations allow controlled evaluation of both distribution matching and subspace recovery. On all datasets, we assess robustness to incomplete observations by masking training entries uniformly at random with $p_{\mathrm{miss}}\in\{0,20,40\}\%$, while always evaluating generation against the complete matrices. The neural network architecture and hyperparameter settings are provided in Appendix~\ref{appendix:NN details}.

\textbf{Baselines.}
We compare against two baseline families. \textbf{(i)} \emph{Ambient-space generative models with missingness handling}: \emph{MissDiff} \citep{ouyang2023missdiff} and \emph{MissFlow}, which train diffusion or flow directly in the original \(m_1\times m_2\) space using masked losses, but still operate in the full ambient dimension. \textbf{(ii)} \emph{Low-dimensional generative models without shared matrix geometry}: \emph{CSDM} \citep{guo2025accelerating}, which uses a fixed compression operator and may distort matrix structure, and \emph{SMG-Core}, a baseline adapted from the matrix completion method BayeSMG \citep{yuchi2023bayesian}, where we fix the Stage~I subspaces learned by CoreFlow and fit the core distribution using BayeSMG variance priors; see Appendix~\ref{sec:baseline_gaussiancore_bayes} for details. We also report a PCA-Flow ablation in Appendix~\ref{appendix:ablation PCA}, replacing Stage I with flattened PCA while keeping the same CNF; since PCA requires complete vectors, it is evaluated only at $p_{\mathrm{miss}}=0$.

Notably, VAE-based latent diffusion \citep{rombach2022high,peebles2023scalable} is not well suited here, as it requires pre-training a deep VAE on the same limited and possibly incomplete matrices, introducing a substantially heavier pipeline that is often impossible in the few-sample regime.

\textbf{Metrics.} We quantify distributional similarity between generated and true matrices using: (i) entrywise moment discrepancies (mean and std), (ii) Frobenius-norm mean and std shifts,
(iii) average singular-value discrepancy, and (iv) maximum mean discrepancy (MMD).
For Simulation we also evaluate subspace recovery via principal angles. See formal definitions in Appendix~\ref{appendix:metrics}.

\subsection{Main results}
To keep the main text focused, we present representative metrics here and defer the remaining quantitative results to Appendix~\ref{appendix:additional results} (Tables~\ref{tab:real frob 1}--\ref{tab:simulation all UV}), where the same overall trends continue to hold.

\begin{table}[h!]
\centering
\caption{
\small Comparison with missing-data baselines on real datasets, measured by average singular-value discrepancy
(lower is better). Columns report the missing-entry rate $p_{\mathrm{miss}}$. Entries are $\times 10^{-2}$ mean $\pm$ standard deviation.
CoreFlow uses $\rho=36\%$ for Solar/Solar 2 and $\rho=49\%$ for LSPF, while ambient-space baselines use
$\rho=100\%$. Solar/Solar 2 are few-sample settings ($n=300$), and LSPF is data-rich ($n=8760$).
}
\label{tab:real_missingness}
\renewcommand{\arraystretch}{1.12}
\setlength{\tabcolsep}{4pt}

\scalebox{0.62}{
\begin{tabular}{l ccc ccc ccc}
\toprule
& \multicolumn{3}{c}{\textbf{Solar}}
& \multicolumn{3}{c}{\textbf{Solar 2}}
& \multicolumn{3}{c}{\textbf{LSPF}} \\
\cmidrule(lr){2-4} \cmidrule(lr){5-7} \cmidrule(lr){8-10}
$\boldsymbol{p_{\mathrm{miss}}}$
& \textbf{0\%} & \textbf{20\%} & \textbf{40\%}
& \textbf{0\%} & \textbf{20\%} & \textbf{40\%}
& \textbf{0\%} & \textbf{20\%} & \textbf{40\%} \\
\midrule
MissDiff ($\rho=100\%$)
& $9.57_{\pm 0.29}$ & $48.2_{\pm 0.31}$ & $89.3_{\pm 0.29}$
& $3.75_{\pm 0.12}$ & $60.3_{\pm 0.15}$ & $98.4_{\pm 0.25}$
& $\mathbf{2.70_{\pm 0.20}}$ & $36.2_{\pm 0.35}$ & $48.7_{\pm 0.38}$ \\

MissFlow ($\rho=100\%$)
& $7.63_{\pm 0.32}$ & $37.3_{\pm 0.20}$ & $75.7_{\pm 0.17}$
& $2.05_{\pm 0.11}$ & $41.7_{\pm 0.12}$ & $86.2_{\pm 0.17}$
& $2.71_{\pm 0.22}$ & $33.9_{\pm 0.42}$ & $45.2_{\pm 0.31}$ \\

\rowcolor{coreflowbg}
\textbf{CoreFlow} ($\rho=36\%/49\%$)
& $\mathbf{1.27_{\pm 0.49}}$ & $\mathbf{1.57_{\pm 0.04}}$ & $\mathbf{2.76_{\pm 0.05}}$
& $\mathbf{0.77_{\pm 0.25}}$ & $\mathbf{1.67_{\pm 0.18}}$ & $\mathbf{1.99_{\pm 0.11}}$
& $2.88_{\pm 0.19}$ & $\mathbf{3.73_{\pm 0.19}}$ & $\mathbf{5.78_{\pm 0.10}}$ \\
\bottomrule
\end{tabular}
}
\label{tab:real data sv 1}
\end{table}

\begin{table}[h!]
\centering
\caption{
\small Comparison with low-dimensional baselines on the Solar datasets, measured by average singular-value discrepancy (lower is better).
Columns report the missing-entry rate $p_{\mathrm{miss}}$.
Entries are shown in $\times 10^{-2}$ as mean $\pm$ standard deviation.
Parentheses indicate the learned-dimension ratio $\rho$.
}
\label{tab:solar_lowdim}
\renewcommand{\arraystretch}{1.12}
\setlength{\tabcolsep}{5pt}

\scalebox{0.61}{
\begin{tabular}{l ccc ccc}
\toprule
& \multicolumn{3}{c}{\textbf{Solar}}
& \multicolumn{3}{c}{\textbf{Solar 2}} \\
\cmidrule(lr){2-4} \cmidrule(lr){5-7}
$\boldsymbol{p_{\mathrm{miss}}}$
& \textbf{0\%} & \textbf{20\%} & \textbf{40\%}
& \textbf{0\%} & \textbf{20\%} & \textbf{40\%} \\
\midrule
CSDM ($\rho=9\%$)
& $74.2_{\pm 1.1}$ & $79.1_{\pm 1.0}$ & $83.4_{\pm 0.90}$
& $75.5_{\pm 1.8}$ & $80.4_{\pm 1.4}$ & $85.2_{\pm 1.1}$ \\

CSDM ($\rho=36\%$)
& $51.5_{\pm 1.2}$ & $61.8_{\pm 1.2}$ & $71.5_{\pm 1.0}$
& $54.0_{\pm 1.9}$ & $64.3_{\pm 0.15}$ & $73.5_{\pm 1.2}$ \\

SMG-Core ($\rho=9\%$)
& $3.21_{\pm 0.90}$ & $4.95_{\pm 1.1}$ & $5.01_{\pm 1.2}$
& $2.51_{\pm 0.98}$ & $4.50_{\pm 1.0}$ & $8.43_{\pm 1.3}$ \\

\rowcolor{coreflowbg}
\textbf{CoreFlow} ($\rho=9\%$)
& $\mathbf{0.87_{\pm 0.10}}$ & $\mathbf{0.93_{\pm 0.07}}$ & $\mathbf{1.01_{\pm 0.10}}$
& $\mathbf{0.56_{\pm 0.16}}$ & $\mathbf{0.64_{\pm 0.12}}$ & $\mathbf{0.79_{\pm 0.07}}$ \\
\bottomrule
\end{tabular}
}
\label{tab:solar sv}
\end{table}

\paragraph{Solar Flare ($200{\times}200$, few-sample).} Solar Flare is our most challenging real-data setting, with $40{,}000$ ambient dimensions but only $300$ training matrices.
Figures~\ref{fig:cross method solar no missing} and~\ref{fig:cross_method_comparison} show visual comparisons, while Tables~\ref{tab:real data sv 1}, \ref{tab:real frob 1}, and~\ref{tab:real entrywise 1} show that CoreFlow, using $\rho=36\%$ for Solar and $\rho=49\%$ for LSPF, consistently outperforms full-space MissDiff and MissFlow. Compared to \textit{low-dimensional baselines}, CoreFlow achieves substantially better metrics than both CSDM and SMG-Core on Solar (Tables~\ref{tab:solar sv}, \ref{tab:solar frob}, and \ref{tab:solar entrywise}). CSDM uses a fixed compression operator that is not geometry-preserving, while SMG-Core models the core with a Gaussian posterior that may underestimate the expressiveness.




\textbf{LSPF ($80{\times}80$, data-rich).} On LSPF, where training data are abundant, the gap naturally narrows. Nevertheless, Table~\ref{tab:real data sv 1} shows that CoreFlow remains competitive while operating in only $49\%$ of the ambient dimension, compared with ambient-space baselines trained at full dimension. As missingness increases, its spectral fidelity also degrades more gracefully, with the same pattern reflected in the Frobenius and entry-wise results in Tables~\ref{tab:real frob 1} and \ref{tab:real entrywise 1}.

\begin{table}[t]
\centering
\caption{
\small Comparison with low-dimensional generative baselines on simulation datasets, measured by average singular-value discrepancy (lower is better).
Columns report the missing-entry rate $p_{\mathrm{miss}}$.
Entries are shown in $\times 10^{-2}$ as mean $\pm$ standard deviation.
Parentheses indicate the learned-dimension ratio $\rho$.
}
\label{tab:simulation_lowdim}
\renewcommand{\arraystretch}{1.08}
\setlength{\tabcolsep}{3pt}

\scalebox{0.58}{
\begin{tabular}{l ccc ccc ccc ccc}
\toprule
& \multicolumn{3}{c}{\textbf{Blobs}}
& \multicolumn{3}{c}{\textbf{Bands}}
& \multicolumn{3}{c}{\textbf{Waves}}
& \multicolumn{3}{c}{\textbf{Crosshatch}} \\
\cmidrule(lr){2-4} \cmidrule(lr){5-7} \cmidrule(lr){8-10} \cmidrule(lr){11-13}
$\boldsymbol{p_{\mathrm{miss}}}$
& \textbf{0\%} & \textbf{20\%} & \textbf{40\%}
& \textbf{0\%} & \textbf{20\%} & \textbf{40\%}
& \textbf{0\%} & \textbf{20\%} & \textbf{40\%}
& \textbf{0\%} & \textbf{20\%} & \textbf{40\%} \\
\midrule
CSDM ($\rho=1.44\%$)
& $409_{\pm 7.4}$ & $484_{\pm 7.9}$ & $559_{\pm 9.1}$
& $392_{\pm 7.2}$ & $461_{\pm 8.3}$ & $529_{\pm 9.6}$
& $89.4_{\pm 3.2}$ & $91.4_{\pm 2.5}$ & $93.5_{\pm 1.9}$
& $88.7_{\pm 2.8}$ & $90.8_{\pm 3.2}$ & $92.9_{\pm 4.8}$ \\

CSDM ($\rho=64\%$)
& $197_{\pm 5.4}$ & $230_{\pm 6.0}$ & $261_{\pm 6.3}$
& $196_{\pm 5.3}$ & $226_{\pm 6.1}$ & $254_{\pm 6.5}$
& $43.3_{\pm 1.6}$ & $57.8_{\pm 1.4}$ & $69.6_{\pm 1.7}$
& $41.7_{\pm 1.4}$ & $56.1_{\pm 1.4}$ & $67.8_{\pm 1.6}$ \\

SMG-Core ($\rho=1.44\%$)
& $195_{\pm 4.0}$ & $219_{\pm 5.1}$ & $254_{\pm 5.7}$
& $124_{\pm 3.3}$ & $179_{\pm 4.0}$ & $201_{\pm 4.9}$
& $18.9_{\pm 1.0}$ & $23.2_{\pm 1.2}$ & $27.7_{\pm 1.4}$
& $19.5_{\pm 0.90}$ & $24.1_{\pm 1.1}$ & $26.8_{\pm 1.3}$ \\

SMG-Core ($\rho=64\%$)
& $20.7_{\pm 1.2}$ & $28.3_{\pm 1.8}$ & $34.8_{\pm 2.3}$
& $16.8_{\pm 1.3}$ & $19.7_{\pm 1.7}$ & $24.1_{\pm 2.5}$
& $6.20_{\pm 0.50}$ & $10.8_{\pm 1.1}$ & $15.6_{\pm 1.2}$
& $8.12_{\pm 0.70}$ & $12.9_{\pm 1.2}$ & $17.8_{\pm 1.7}$ \\

\rowcolor{coreflowbg}
\textbf{CoreFlow} ($\rho=1.44\%$)
& $\mathbf{0.43_{\pm 0.12}}$ & $\mathbf{0.63_{\pm 0.14}}$ & $\mathbf{1.28_{\pm 0.20}}$
& $\mathbf{4.45_{\pm 0.74}}$ & $\mathbf{8.87_{\pm 0.68}}$ & $\mathbf{9.09_{\pm 0.32}}$
& $\mathbf{3.47_{\pm 0.17}}$ & $\mathbf{5.25_{\pm 0.07}}$ & $\mathbf{5.71_{\pm 0.31}}$
& $\mathbf{3.23_{\pm 0.50}}$ & $\mathbf{4.33_{\pm 0.21}}$ & $\mathbf{4.70_{\pm 0.26}}$ \\
\bottomrule
\end{tabular}
}
\label{tab:simulation data sv}
\end{table}

\textbf{Simulation ($200\times 200$, extreme compression).}
As shown in Table~\ref{tab:simulation data sv}, the simulation study considers an even harsher compression setting, where CoreFlow is trained in only $1.44\%$ of the ambient dimension. Even under this extreme reduction, it preserves the matrix spectrum well across all datasets, while both CSDM and SMG-Core degrade noticeably at the same dimensionality. CSDM and SMG-Core improve only when allowed a much larger compressed space.





\begin{figure}[t]
  \centering
  \captionsetup[subfigure]{font=footnotesize, justification=centering, skip=2pt}

  \begin{subfigure}[t]{0.48\linewidth}
    \centering
    \includegraphics[width=\linewidth]{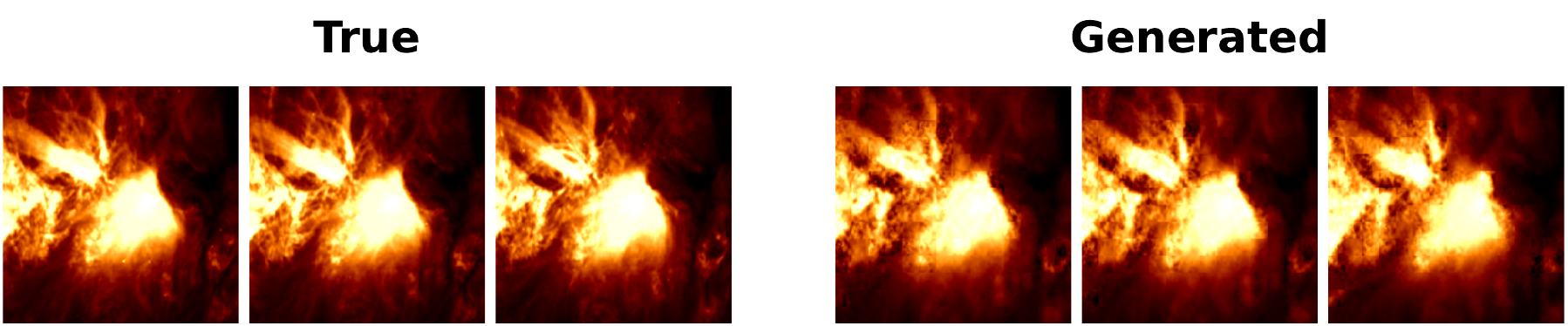}
    \caption{Solar (9\% dim., 40\% missing)}
  \end{subfigure}
  \hfill
  \begin{subfigure}[t]{0.48\linewidth}
    \centering
    \includegraphics[width=\linewidth]{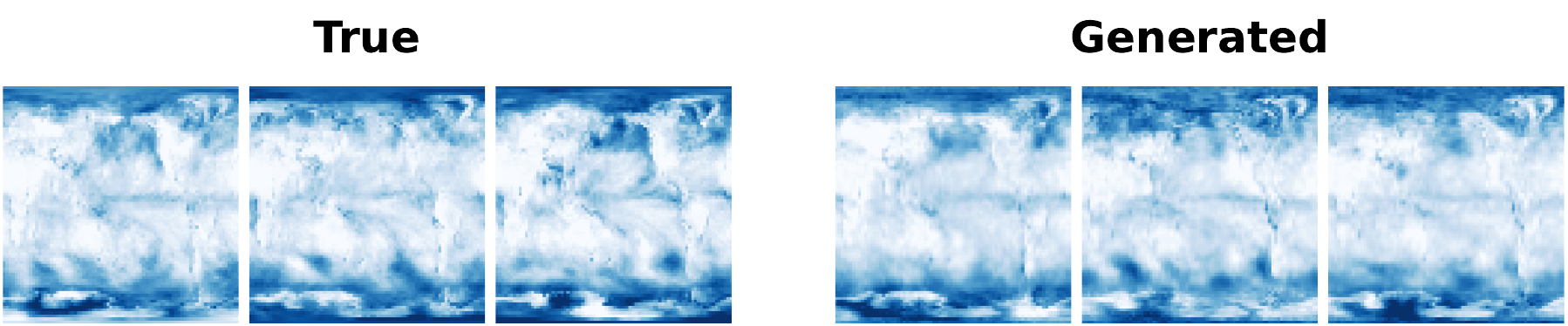}
    \caption{LSPF (49\% dim., 0\% missing)}
  \end{subfigure}

  \vspace{-0.3mm}

  \begin{subfigure}[t]{0.48\linewidth}
    \centering
    \includegraphics[width=\linewidth]{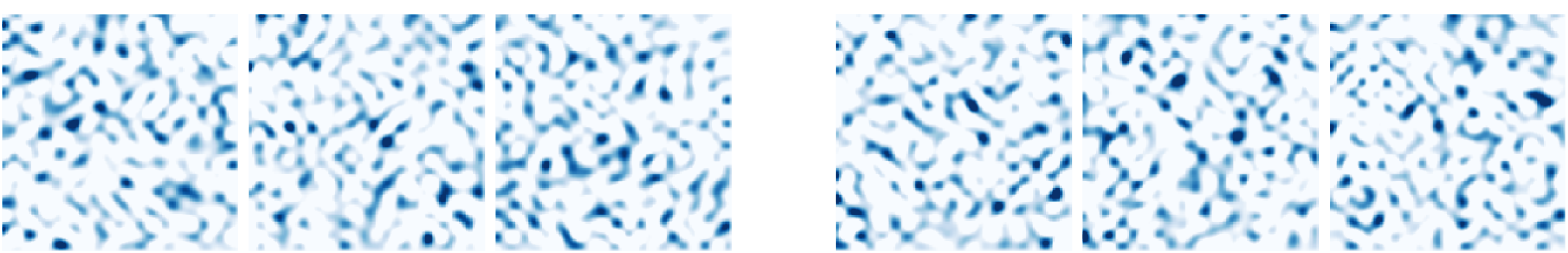}
    \caption{Blobs (1.44\% dim., 40\% missing)}
  \end{subfigure}
  \hfill
  \begin{subfigure}[t]{0.48\linewidth}
    \centering
    \includegraphics[width=\linewidth]{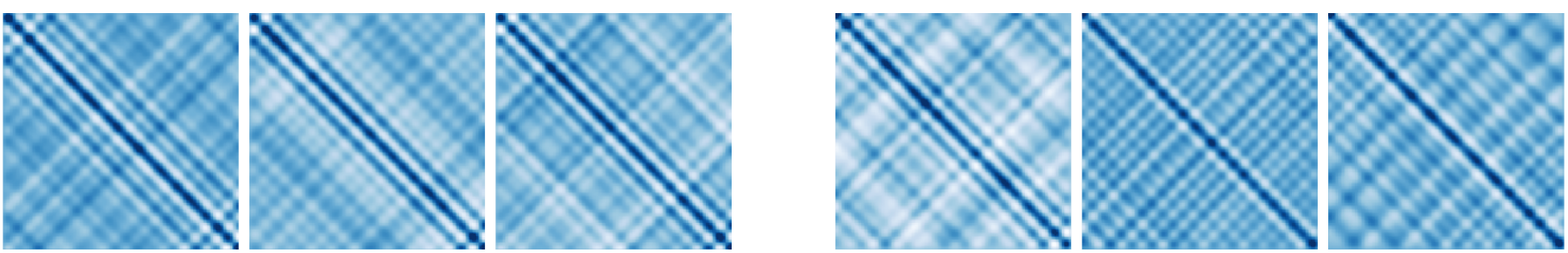}
    \caption{Crosshatch (1.44\% dim., 40\% missing)}
  \end{subfigure}

  \caption{\small True samples (left) and CoreFlow-generated samples (right) on real and synthetic benchmarks. The close visual agreement shows that CoreFlow captures the matrix distributions well even under aggressive dimensionality reduction and substantial training missingness.}
  \label{fig:true_vs_gen_all}
\end{figure}

\begin{figure}[H]
  \centering
  \begin{minipage}[t]{0.43\linewidth}
    \textbf{Computational efficiency.} Figure~\ref{fig:solar_efficiency} shows that CoreFlow achieves the best tradeoff between \emph{generative-model cost} and generation quality by learning only in the geometry-aware core space. On Solar, it uses just 9\% of the ambient dimension (about 11$\times$ lower cost proxy than ambient-space models) while still achieving the best spectral fidelity. On LSPF, it reduces the training dimension to 49\% (2.1$\times$ lower proxy cost) without sacrificing robustness. Full training times, including Stage~I for CoreFlow and converged training times for the baselines, are reported in Appendix~\ref{appendix:NN details}.
  \end{minipage}\hfill
  \begin{minipage}[t]{0.54\linewidth}
    \vspace{-1.0em}
    \centering
    \includegraphics[width=\linewidth]{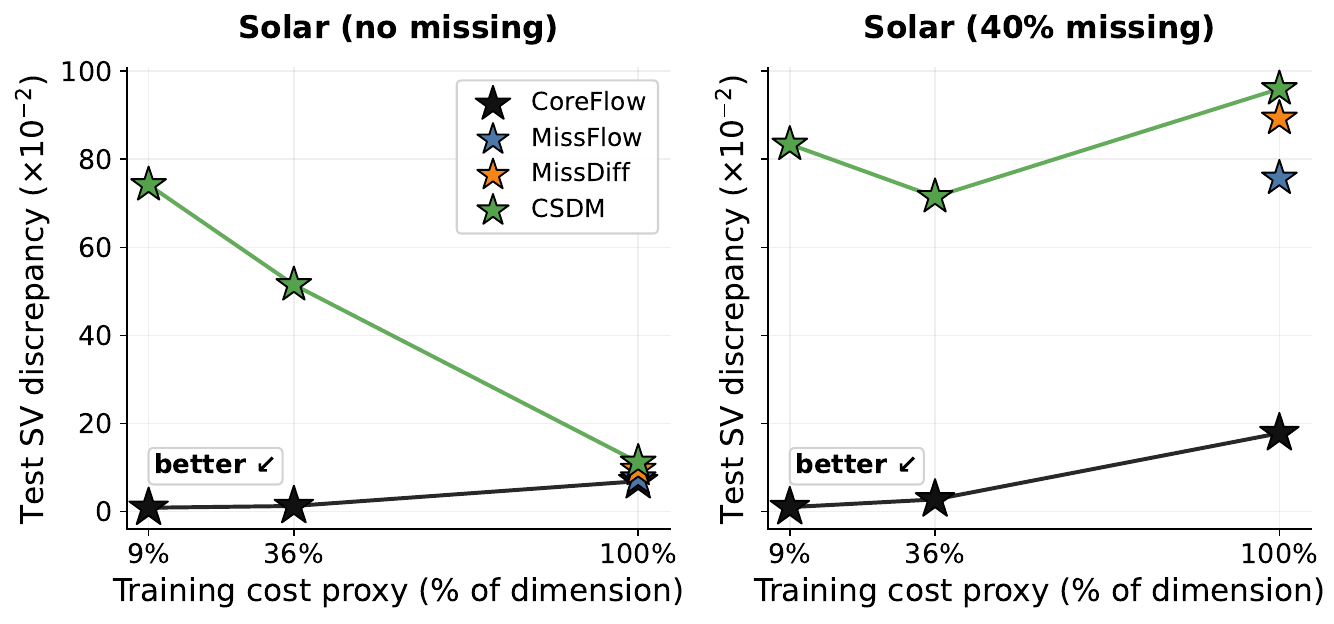}
    \caption{\small Performance--efficiency tradeoff on Solar among generative models. The horizontal axis is the learned-dimension ratio as a proxy for \emph{generative-model cost}, and the vertical axis is test singular-value discrepancy (lower is better). Full training times are reported in Appendix~\ref{appendix:NN details}.}
    \label{fig:solar_efficiency}
  \end{minipage}
\end{figure}
\textbf{Robustness to missing entries.}
Across all benchmarks, CoreFlow degrades only mildly as \(p_{\mathrm{miss}}\) increases, and substantially less than the baselines. This is because Stage~I learns \((U,V)\) with masked reconstruction and iterative completion (Alg.~\ref{alg:stage1-missing}). Table~\ref{tab:simulation all UV} further confirms reliable subspace recovery under missingness in simulations.

\subsection{Discussion}
\textit{Structure-aware Stage I matters.} CoreFlow learns shared row/column subspaces and models only the induced core distribution, reducing dimension while preserving matrix geometry. This explains its robustness under aggressive compression and missingness, where structure-agnostic compressors such as CSDM degrade. The PCA-Flow ablation in Appendix~\ref{appendix:ablation PCA} further shows that simply replacing Stage I with flattened PCA can be competitive on complete data, but is weaker overall and cannot handle missing entries. This highlights the benefit of our matrix-structured \(USV^\top\) representation.

\textit{CoreFlow enables few-sample high-dimensional generation.}
On Solar Flare, learning a \(40{,}000\)-dimensional distribution from \(300\) samples is highly underdetermined. By recovering shared low-rank geometry and learning only the core distribution, we make the generative task more stable. On data-rich LSPF, the gap narrows, but CoreFlow maintains a favorable efficiency--robustness tradeoff.

\section{Conclusions}
We presented CoreFlow, a geometry-preserving low-rank generative model for matrix generation that first learns shared row and column subspaces, and then fits a continuous normalizing flow only in the induced core space. This preserves matrix structure while substantially reducing the effective generative dimension. The PCA-Flow ablation further shows that low-dimensional flow modeling is useful, but that the matrix-structure-aware representation is crucial for geometry-preserving generation and for naturally handling incomplete training matrices. Theoretically, we establish an explicit separation between geometric error from subspace recovery and generative error from core-space modeling. Empirically, CoreFlow improves generation quality over ambient-space and geometry-agnostic baselines, especially in high-dimensional few-sample regimes, while remaining robust under substantial compression and missingness. A limitation is that CoreFlow is most effective when shared low-rank matrix geometry is present. While patchification can help reduce the effective rank, extending the method to more heterogeneous matrix distributions remains an important direction.

\section*{Acknowledgment}
This work is partially supported by NSF  CMMI-2112533, Emory Healthcare A1230749, and the Coca-Cola Foundation.

\bibliographystyle{plainnat}
\bibliography{reference}

\clearpage
\appendix
\thispagestyle{empty}

\section{Proofs}

\begin{proof}[Proof of Proposition \ref{proposition: global opt}]
Let $U\in\mathbb{R}^{m_1\times R}$ and $V\in\mathbb{R}^{m_2\times R}$ satisfy $U^\top U=I_R$ and $V^\top V=I_R$. Define the orthogonal projectors
\begin{equation*}
P_U := UU^\top,\qquad P_V := VV^\top.
\end{equation*}
Note that $P_U^\top P_U = P_U$ and $P_V^\top P_V = P_V$. The population loss becomes
\begin{equation*}
\mathcal{L}(U,V) = \mathbb{E}\|M-U(U^\top M V)V^\top\|_F^2
=\mathbb{E}\|M-P_U M P_V\|_F^2.
\end{equation*}
Expanding the Frobenius norm, we have
\begin{equation}
\|M-P_U M P_V\|_F^2
= \|M\|_F^2 - 2\langle M,\,P_U M P_V\rangle + \|P_U M P_V\|_F^2.
\label{eq:intermediate_loss}
\end{equation}
Using the cyclic property of the trace and the idempotence of projectors, the inner product term simplifies to
\begin{equation}
\begin{aligned}
\langle M,\,P_U M P_V\rangle &= \mathrm{tr}\big(M^\top P_U M P_V\big) \\
&= \mathrm{tr}\big(M^\top P_U^\top P_U M P_V^\top P_V\big) \\
&= \mathrm{tr}\big((P_U M P_V)^\top (P_U M P_V)\big) = \|P_U M P_V\|_F^2.
\end{aligned}
\label{eq:inner_product_M}
\end{equation}
Substituting \eqref{eq:inner_product_M} into \eqref{eq:intermediate_loss}, we obtain
\begin{equation*}
    \mathcal{L}(U,V)=\mathbb{E}\|M\|_F^2-\mathbb{E}\|P_U M P_V\|_F^2.
\end{equation*}
Since $\mathbb{E}\|M\|_F^2$ is independent of $(U,V)$, minimizing $\mathcal{L}(U,V)$ is equivalent to maximizing
\begin{equation*}
\mathcal{J}(U,V) := \mathbb{E}\|P_U M P_V\|_F^2.
\end{equation*}
Define matrices $A := U^\top U_0 \in\mathbb{R}^{R\times R}$ and $B := V^\top V_0 \in\mathbb{R}^{R\times R}$. Since $M = U_0 S_0 V_0^\top$ (under $\tau=0$), we have
\begin{equation*}
\mathcal{J}(U,V)
= \mathbb{E}\|P_U U_0 S_0 V_0^\top P_V\|_F^2
= \mathbb{E}\|U A S_0 B^\top V^\top\|_F^2
= \mathbb{E}\|A S_0 B^\top\|_F^2,
\end{equation*}
where the last equality holds due to the orthogonal invariance of the Frobenius norm. Further defining $P := A^\top A = U_0^\top P_U U_0$ and $Q := B^\top B = V_0^\top P_V V_0$, we can rewrite the objective as:
\begin{equation*}
\begin{aligned}
\mathcal{J}(U,V)
&= \mathbb{E}\,\mathrm{tr}(A S_0 B^\top B S_0^\top A^\top) 
= \mathbb{E}\,\mathrm{tr}(PS_0 Q S_0^\top) = \mathrm{tr}(P\,\mathbb{E}\,[S_0 Q S_0^\top]).
\end{aligned}
\end{equation*}
It is immediate that $P \succeq 0$ and $Q \succeq 0$. Furthermore, for any vector $x \in \mathbb{R}^R$, since multiplication by $U^\top$ (where $U$ has orthonormal columns) cannot increase the $\ell_2$-norm of a vector, we know that
\begin{equation*}
x^\top P x
= x^\top U_0^\top U U^\top U_0 x
= \|U^\top (U_0 x)\|_2^2
\le \|U_0 x\|_2^2
= \|x\|_2^2.
\end{equation*}
The identical logic applies to $Q$. Therefore, the eigenvalues of $P$ and $Q$ are bounded:
\begin{equation*}
0 \preceq P \preceq I_R, \qquad 0 \preceq Q \preceq I_R.
\end{equation*}
Using $P \preceq I_R$ and the fact that $\mathbb{E}[S_0QS_0^\top] \succeq 0$ it follows that
\begin{equation}
\mathcal{J}(U,V)
= \mathrm{tr}(P\,\mathbb{E}[S_0QS_0^\top])
\le \mathrm{tr}(\mathbb{E}[S_0QS_0^\top]).
\label{eq:upper_bound_1}
\end{equation}
Using the cyclicity of the trace and recall that $\Sigma_R = \mathbb{E}[S_0^\top S_0]$, we have
\begin{equation*}
\mathrm{tr}(\mathbb{E}[S_0QS_0^\top])
= \mathrm{tr}(Q\,\mathbb{E}[S_0^\top S_0])
= \mathrm{tr}(Q\Sigma_R).
\end{equation*}
Similarly, since $Q \preceq I_R$ and $\Sigma_R \succ 0$ (by Assumption \ref{assumption:shared subspace}), we further bound this by
\begin{equation}
\mathrm{tr}(Q\Sigma_R) \le \mathrm{tr}(\Sigma_R)
= \mathbb{E}\,\mathrm{tr}(S_0^\top S_0)
= \mathbb{E}\|S_0\|_F^2.
\label{eq:upper_bound_2}
\end{equation}
Combining \eqref{eq:upper_bound_1} and \eqref{eq:upper_bound_2} yields the global upper bound $\mathcal{J}(U,V)\le \mathbb{E}\|S_0\|_F^2$.

This bound is tight and strictly achievable. If we set $(U,V)=(U_0,V_0)$, then $P_{U_0}MP_{V_0} = U_0 U_0^\top M V_0 V_0^\top = M$. Thus
\begin{equation*}
\mathcal{J}(U_0,V_0) = \mathbb{E}\|P_{U_0} M P_{V_0}\|_F^2 = \mathbb{E}\|M\|_F^2 = \mathbb{E}\|U_0S_0V_0^\top\|_F^2 = \mathbb{E}\|S_0\|_F^2.
\end{equation*}
If a pair $(U,V)$ is a global optimizer, it must achieve the maximal upper bound. Consequently, the inequalities in \eqref{eq:upper_bound_1} and \eqref{eq:upper_bound_2} must hold as equalities, meaning:
\begin{equation*}
\mathrm{tr}((I_R - Q)\Sigma_R) = 0 \quad \text{and} \quad \mathrm{tr}((I_R - P)\mathbb{E}[S_0QS_0^\top]) = 0.
\end{equation*}
Since $\Sigma_R \succ 0$ and $I_R - Q \succeq 0$, the trace of their product is zero if and only if $I_R - Q = 0$, implying $Q = I_R$. Consequently, $\mathbb{E}[S_0QS_0^\top] = \mathbb{E}[S_0S_0^\top] = \Sigma_L \succ 0$. By the same logic, $I_R - P \succeq 0$ combined with $\Sigma_L \succ 0$ implies $P = I_R$. 

The conditions $P = I_R$ and $Q = I_R$ explicitly translate to
\begin{equation*}
    U_0^\top P_U U_0 = I_R \quad \text{and} \quad V_0^\top P_V V_0 = I_R.
\end{equation*}
Focusing on $U$, we can rewrite $U_0^\top P_U U_0 = I_R$ as $U_0^\top (I_{m_1} - P_U) U_0 = 0$. Since $I_{m_1} - P_U$ is a positive semi-definite projection matrix, this requires $(I_{m_1} - P_U)U_0 = 0$, or equivalently, $P_U U_0 = U_0$. 
This means $U (U^\top U_0) = U_0$, which strictly implies that $\text{span}(U_0) \subseteq \text{span}(U)$. Since both $U$ and $U_0$ have exactly rank $R$, their column spaces must be identical: $\text{span}(U_0) = \text{span}(U)$. By symmetric reasoning, $\text{span}(V_0) = \text{span}(V)$. Thus, the global optimizers are equivalent to the true subspaces up to orthogonal rotations. 
\end{proof}

\begin{proof}[Proof of Proposition \ref{proposition:tucker init}]
Under the exact low-rank setting ($\tau=0$) of Assumption \ref{assumption:shared subspace}, we have
\begin{equation*}
MM^\top = U_0 S_0 V_0^\top V_0 S_0^\top U_0^\top = U_0(S_0S_0^\top)U_0^\top.
\end{equation*}
Taking the expectation gives the population left second moment
\begin{equation*}
C_L = \mathbb{E}[MM^\top] = U_0 \mathbb{E}[S_0S_0^\top] U_0^\top = U_0\Sigma_L U_0^\top.
\end{equation*}
Since $\Sigma_L \succ 0$, $C_L$ is positive semi-definite with exactly rank $R$. Furthermore, $\Sigma_L$ and $C_L$ share the same non-zero eigenvalues. For any eigen-pair $(\mu,q)$ of $\Sigma_L$ such that $\Sigma_L q = \mu q$, we can define $u = U_0 q$. It follows that
\begin{equation*}
    C_L u = U_0\Sigma_L U_0^\top (U_0 q) = U_0\Sigma_L q = U_0 (\mu q) = \mu (U_0 q) = \mu u.
\end{equation*}
Therefore, the $R$-th largest eigenvalue of $C_L$ perfectly matches that of $\Sigma_L$
\begin{equation}    
\lambda_R(C_L) = \lambda_{R}(\Sigma_L).
\label{eq:proof_lambda_R_equal}
\end{equation}

Let $\Delta_L := \widehat{C}_L - C_L$ denote the empirical estimation error.
By the Courant–Fischer min-max principle \citep{horn2012matrix}, we know
\begin{equation*}
\lambda_k(\widehat{C}_L)
=\max_{\dim(\mathcal{S})=k}\;\min_{\substack{x\in\mathcal{S}\\ \|x\|_2=1}} x^\top \widehat{C}_L x
=\max_{\dim(\mathcal{S})=k}\;\min_{\substack{x\in\mathcal{S}\\ \|x\|_2=1}} x^\top (C_L + \Delta_L)x.
\end{equation*}
Since $x^\top \Delta_L x \ge \lambda_{\min}(\Delta_L)$ for all unit vectors $x$, we have
\begin{equation}
\lambda_k(\widehat{C}_L) \ge
\max_{\dim(\mathcal{S})=k}\;\min_{\substack{x\in\mathcal{S}\\ \|x\|_2=1}}
\big(x^\top C_L x + \lambda_{\min}(\Delta_L)\big)
= \lambda_k(C_L) + \lambda_{\min}(\Delta_L).
\label{eq:proof_lambda_ge}
\end{equation}
Similarly, bounding from above using $x^\top \Delta_L x \le \lambda_{\max}(\Delta_L)$ yields
\begin{equation*}
    \lambda_k(\widehat{C}_L) \le \lambda_k(C_L) + \lambda_{\max}(\Delta_L).
\end{equation*}
Since both $C_L$ and $\widehat{C}_L$ are symmetric, the perturbation $\Delta_L$ is symmetric, which directly implies
\begin{equation}
    \|\Delta_L\|_2 = \sqrt{\lambda_{\max}(\Delta_L^\top \Delta_L)} = \sqrt{\max_{i} \lambda_i^2(\Delta_L)} \ge -\lambda_{\min}(\Delta_L).
    \label{eq:proof_lambda_hat_C_L}
\end{equation}
By combining \eqref{eq:proof_lambda_R_equal}, \eqref{eq:proof_lambda_ge}, and \eqref{eq:proof_lambda_hat_C_L}, we establish a lower bound on the $R$-th eigenvalue of the empirical matrix
\begin{equation}
\lambda_R(\widehat{C}_L) \ge \lambda_R(C_L) - \|\Delta_L\|_2
= \lambda_{R}(\Sigma_L) - \|\Delta_L\|_2.
\label{eq:proof_step_2_lower_bound}
\end{equation}

Recall that $U^{(0)}$ is the matrix of the top-$R$ eigenvectors of $\widehat{C}_L$ and $\Lambda^{(0)}\in\mathbb{R}^{R\times R}$ is their corresponding diagonal eigenvalue matrix satisfying:
\begin{equation*}
\widehat{C}_L U^{(0)} = U^{(0)}\Lambda^{(0)},\qquad U^{(0)\top}U^{(0)} = I_R.
\end{equation*}
Let $U_{0,\perp}\in\mathbb{R}^{m_1\times(m_1-R)}$ denote an orthonormal complement of $U_0$. Left-multiplying the eigen-equation by $U_{0,\perp}^\top$ gives
\begin{equation}
    U_{0,\perp}^\top \widehat{C}_L U^{(0)} = U_{0,\perp}^\top U^{(0)}\Lambda^{(0)}.
    \label{eq:proof_U_0_hat_C_L_U_equal}
\end{equation}
Alternatively, expanding $\widehat{C}_L = C_L + \Delta_L$ and noting that $U_{0,\perp}^\top C_L = 0$ (since $C_L = U_0\Sigma_L U_0^\top$), we obtain
\begin{equation}
U_{0,\perp}^\top \widehat{C}_L U^{(0)}
= U_{0,\perp}^\top (C_L + \Delta_L)U^{(0)}
= U_{0,\perp}^\top C_L U^{(0)} + U_{0,\perp}^\top \Delta_L U^{(0)}
= U_{0,\perp}^\top \Delta_L U^{(0)}.
\label{eq:proof_U_0_hat_C_L_U_equal_2}
\end{equation}
Equating \eqref{eq:proof_U_0_hat_C_L_U_equal} and \eqref{eq:proof_U_0_hat_C_L_U_equal_2} yields
\begin{equation*}
U_{0,\perp}^\top \Delta_L U^{(0)} = U_{0,\perp}^\top U^{(0)}\Lambda^{(0)}.
\end{equation*}
Since $\|\widehat{C}_L - C_L\|_2 < \lambda_R(\Sigma_L)$, equation \eqref{eq:proof_step_2_lower_bound} ensures that $\lambda_R(\widehat{C}_L) > 0$. Therefore, $\Lambda^{(0)}$ is invertible, and we have 
\begin{equation*}
U_{0,\perp}^\top U^{(0)} = U_{0,\perp}^\top \Delta_L U^{(0)}\big(\Lambda^{(0)}\big)^{-1}.
\end{equation*}
This together with the facts that $\|U_{0,\perp}^\top\|_2 = 1$ and $\|U^{(0)}\|_2 = 1$ implies that
\begin{equation*}
\|U_{0,\perp}^\top U^{(0)}\|_2
\le \|\Delta_L\|_2 \cdot \|(\Lambda^{(0)})^{-1}\|_2
= \frac{\|\Delta_L\|_2}{\lambda_R(\widehat{C}_L)}\le
\frac{\|\Delta_L\|_2}{\lambda_{R}(\Sigma_L) - \|\Delta_L\|_2},
\end{equation*}
where the last inequality is from the lower bound \eqref{eq:proof_step_2_lower_bound}.
Finally, by the definition of principal angles, we have $\|U_{0,\perp}^\top U^{(0)}\|_2 = \|(I - U_0 U_0^\top)U^{(0)}\|_2 = \sin\Theta(U^{(0)},U_0)$, which concludes the proof. 
\end{proof}

\begin{proof}[Proof of Proposition \ref{prop:coreflow_fm_chi2}]
By the conditional-to-marginal equivalence of flow matching~\citep[Theorem~2]{lipmanflow}, the population minimizer of Algorithm~\ref{alg:coreflow}'s CFM loss for the linear interpolant \(x_t=(1-t)s+tz\), \(s\sim p_s^\ast\), \(z\sim q\), is the marginal velocity
$v(x,t)=\mathbb{E}[z-s\mid x_t=x].$
This velocity induces the density path \(\rho_t\) that satisfies the continuity equation and transports \(\rho_0=p_s^\ast\) to \(\rho_1=q\). Therefore, Assumption~\ref{asmp:fm_reg}(i) is a population approximation condition for \(\hat v\) relative to this marginal velocity \(v\).

Define the time-dependent divergences
\begin{equation*}
G(t):=\chi^2(\rho_t\|q)=\left\|\frac{\rho_t}{q}-1\right\|_{L^2(q)}^2,\qquad
\hat G(t):=\chi^2(\hat\rho_t\|q)=\left\|\frac{\hat\rho_t}{q}-1\right\|_{L^2(q)}^2.
\end{equation*}
Since the target vector field $v$ is valid, we have $\rho_1=q$, which implies $G(1)=0$.
By the reverse triangle inequality in the normed space $L^2(q)$,
\begin{equation*}
\left|\sqrt{\hat G(1)}-\sqrt{G(1)}\right|
\le \left\|\left(\frac{\hat\rho_1}{q}-1\right)-\left(\frac{\rho_1}{q}-1\right)\right\|_{L^2(q)}
= \left\|\frac{\hat\rho_1-\rho_1}{q}\right\|_{L^2(q)}.
\end{equation*}
Expanding the squared norm yields
\begin{equation*}
\left\|\frac{\hat\rho_1-\rho_1}{q}\right\|_{L^2(q)}^2
=\int_{\mathbb{R}^d}\Big(\frac{\hat\rho_1-\rho_1}{q}\Big)^2 q\,dx
=\int_{\mathbb{R}^d}\frac{(\hat\rho_1-\rho_1)^2}{q}\,dx
=\int_{\mathbb{R}^d}\frac{(\hat\rho_1-\rho_1)^2}{\hat\rho_1}\cdot\frac{\hat\rho_1}{q}\,dx.
\end{equation*}
By Assumption~\ref{asmp:fm_reg} (iii), both density paths satisfy the upper bound $\hat\rho_t(x) \le C_1 e^{-\|x\|^2/2}$. Consequently, the density ratio is bounded by
\begin{equation*}
\frac{\hat\rho_t(x)}{q(x)} \le \frac{C_1 e^{-\|x\|^2/2}}{c_d\, e^{-\|x\|^2/2}}=\frac{C_1}{c_d},
\qquad \forall\,t\in[0,1],\; x\in\mathbb{R}^d.
\label{eq:density_ratio_bound}
\end{equation*}
Next, we define the auxiliary function
\begin{equation*}
F(t):=\int_{\mathbb{R}^d}\frac{(\hat\rho_t-\rho_t)^2}{\hat\rho_t}\,dx 
=\int_{\mathbb{R}^d}\left(\frac{\rho_t}{\hat\rho_t}-1\right)^2 \hat\rho_t\,dx \ge 0.
\end{equation*}
Combining these observations, we obtain
\begin{equation}
\label{eq:chi2_to_F}
\sqrt{\hat G(1)}
\le
\Biggl(\int_{\mathbb{R}^d}\frac{(\hat\rho_1-\rho_1)^2}{\hat\rho_1}\cdot\frac{\hat\rho_1}{q}\,dx\Biggr)^{1/2}
\le
\left(\frac{C_1}{c_d}\right)^{1/2}\sqrt{F(1)}.
\end{equation}
 
Since $\rho_0=\hat\rho_0=p$, we initially have $F(0)=0$. Recall that the Neural ODE induces continuity equations governing both density paths
\begin{equation}
    \partial_t \rho_t+\nabla\!\cdot(\rho_t\, v)=0, \qquad
    \partial_t \hat\rho_t+\nabla\!\cdot(\hat\rho_t\, \hat v)=0.
    \label{equation:continuity}
\end{equation}
Differentiating $F(t)$ with respect to time yields
\begin{equation}
\begin{aligned}
\frac{d}{dt}F(t)
&= \int_{\mathbb{R}^d}\left(\frac{\rho_t}{\hat\rho_t}-1\right)^2\,\partial_t \hat\rho_t
+2\left(\frac{\rho_t}{\hat\rho_t}-1\right)\left(\partial_t\rho_t-\frac{\rho_t}{\hat\rho_t}\,\partial_t\hat\rho_t\right)dx
\\
&= \int_{\mathbb{R}^d} 2\,\partial_t\rho_t\left(\frac{\rho_t}{\hat\rho_t}-1\right)
-\partial_t\hat\rho_t\left(\left(\frac{\rho_t}{\hat\rho_t}\right)^2-1\right)dx
\\
&= \int_{\mathbb{R}^d} -2\,\nabla\!\cdot\!(\rho_t\,v)\left(\frac{\rho_t}{\hat\rho_t}-1\right)
\;+\;\nabla\!\cdot\!\left(\hat\rho_t\,\hat v\right)\left(\left(\frac{\rho_t}{\hat\rho_t}\right)^2-1\right)dx
\\
&= \int_{\mathbb{R}^d} 2(\rho_t\,v)\cdot\nabla\left(\frac{\rho_t}{\hat\rho_t}\right)
-(\hat\rho_t\,\hat v)\cdot \frac{2\rho_t}{\hat\rho_t}\,\nabla\left(\frac{\rho_t}{\hat\rho_t}\right)dx
\\
&= 2\int_{\mathbb{R}^d} \rho_t(v-\hat v)\cdot\nabla\left(\frac{\rho_t}{\hat\rho_t}\right)dx
\\
&= 2\int_{\mathbb{R}^d} (v-\hat v)\cdot\frac{\rho_t^2}{\hat\rho_t}
(\nabla\log\rho_t-\nabla\log\hat\rho_t)dx,
\end{aligned}
\label{equation:F(t) derivative}
\end{equation}
where the third equality substitutes the continuity equations~\eqref{equation:continuity}, and the fourth equality follows from integration by parts. Note that the boundary terms strictly vanish due to the Gaussian tail decay specified in Assumption~\ref{asmp:fm_reg}(iii). Integrating~\eqref{equation:F(t) derivative} over $t\in[0,1]$ and applying $F(0)=0$ gives
\begin{equation}
F(1)=2\int_0^1\!\!\int_{\mathbb{R}^d} \bigl(v-\hat v\bigr)\cdot\frac{\rho_t^2}{\hat\rho_t}\,
\bigl(\nabla\log\rho_t-\nabla\log\hat\rho_t\bigr)dx dt.
\label{equation:F(1)}
\end{equation}
By factoring the integrand as $(\|v-\hat v\|\sqrt{\rho_t})\cdot(\|\nabla\log\rho_t-\nabla\log\hat\rho_t\|\rho_t^{3/2}/\hat\rho_t)$ and applying the Cauchy--Schwarz inequality in $L^2(dxdt)$, we can bound~\eqref{equation:F(1)} as follows
\begin{equation*}
\label{eq:CS_bound}
\frac12 F(1)
\le
\left(\int_0^1\!\!\int_{\mathbb{R}^d} \|v-\hat v\|^2\,\rho_t dx dt\right)^{1/2}
\cdot\left(\int_0^1\!\!\int_{\mathbb{R}^d} \|\nabla\log\rho_t-\nabla\log\hat\rho_t\|^2
\frac{\rho_t^3}{\hat\rho_t^2}dx dt\right)^{1/2}.
\end{equation*}
For the first factor, Assumption~\ref{asmp:fm_reg} (i) directly implies that the square root is bounded by $\varepsilon$. For the second factor, the score-difference bound in Assumption~\ref{asmp:fm_reg} (iv) provides
\begin{equation*}
\|\nabla\log\rho_t-\nabla\log\hat\rho_t\|
\le \|\nabla\log\rho_t\|+\|\nabla\log\hat\rho_t\|
\le 2L(1+\|x\|).
\end{equation*}
Consequently, for any $t \in [0,1]$, we have
\begin{equation*}
\int_{\mathbb{R}^d} \|\nabla\log\rho_t-\nabla\log\hat\rho_t\|^2\frac{\rho_t^3}{\hat\rho_t^2}\,dx
\le (2L)^2 \int_{\mathbb{R}^d} (1+\|x\|)^2\frac{\rho_t^3}{\hat\rho_t^2}\,dx
\le (2L)^2 C_2,
\end{equation*}
where the final inequality leverages the moment bound from Assumption~\ref{asmp:fm_reg} (v). Combining both factors, we conclude
\begin{equation}
\label{eq:F1_bound}
F(1)\le 4L\sqrt{C_2}\varepsilon =: C_3\varepsilon.
\end{equation}
 
Plugging \eqref{eq:F1_bound} into \eqref{eq:chi2_to_F} and recalling $G(1)=0$, we obtain
the forward divergence bound
\begin{equation*}
\label{eq:chi2_forward}
\chi^2(\hat\rho_1\|q)=\hat G(1)\le \frac{C_1}{c_d}\cdot C_3\,\varepsilon = C_F\,\varepsilon,
\end{equation*}
where $C_F := (C_1/c_d)\,C_3 = 4L\sqrt{C_2}\cdot C_1/c_d$.

By definition, $\hat p_s=\hat T^{-1}_\# q$, which yields $\hat T_\# \hat p_s=q$. Furthermore, $\hat\rho_1=\hat T_\# p$. We now prove the pushforward identity $\chi^2(p_s^*\|\hat p_s)=\chi^2(\hat\rho_1\|q)$. Since the estimated vector field $\hat v$ is $L_v$-Lipschitz (Assumption~\ref{asmp:fm_reg} (ii)), the resulting flow map $\hat T:\mathbb{R}^d\to\mathbb{R}^d$ is a valid $C^1$-diffeomorphism with a strictly positive Jacobian $J(x):=|\det\nabla\hat T(x)|>0$ everywhere. For $y=\hat T(x)$, the pushforward densities satisfy
\begin{equation*}
(\hat T_\# p_s^*)(y)=\frac{p_s^*(x)}{J(x)},\qquad (\hat T_\# \hat p_s)(y)=\frac{\hat p_s(x)}{J(x)},
\end{equation*}
ensuring that the density ratio is preserved: $(\hat T_\# p_s^*)(y)/(\hat T_\# \hat p_s)(y)=p_s^*(x)/\hat p_s(x)$. Applying the change of variables $y=\hat T(x)$ (where $dy=J(x)\,dx$), we evaluate
\begin{equation*}
\begin{aligned}
\chi^2(\hat T_\# p_s^* \,\|\, \hat T_\# \hat p_s)
&=\int_{\mathbb{R}^d}\left(\frac{(\hat T_\# p_s^*)(y)}{(\hat T_\# \hat p_s)(y)}-1\right)^2(\hat T_\# \hat p_s)(y)dy\\
&=\int_{\mathbb{R}^d}\left(\frac{p_s^*(x)}{\hat p_s(x)}-1\right)^2\frac{\hat p_s(x)}{J(x)}J(x)dx=\chi^2(p_s^*\|\hat p_s).
\end{aligned}
\end{equation*}
Therefore, we conclude the proof with
\begin{equation*}
\chi^2(p_s^*\|\hat p_s)=\chi^2(\hat T_\# p_s^* \,\|\, \hat T_\# \hat p_s)=\chi^2(\hat\rho_1\|q)\le C_F\varepsilon.
\end{equation*}
\end{proof}

\begin{proof}[Proof of Theorem~\ref{thm:end_to_end}]

Condition on the fixed Stage-I output \((U^*,V^*)\), and define
\[
P_U := U^*U^{*\top},\qquad P_V := V^*V^{*\top}.
\]
For \(M\sim p_M\), let
\[
S := U^{*\top}MV^*,\qquad s:=\mathrm{vec}(S),
\]
so that \(s\sim p_s^*\). Let \(\hat S:=\mathrm{mat}(\hat s)\), where
\(\hat s\sim \hat p_s\), and define
\[
\hat M := U^*\hat S V^{*\top}.
\]
We couple \((s,\hat s)\) via the optimal \(W_2\)-coupling between
\(p_s^*\) and \(\hat p_s\), so that
\[
\mathbb{E}\|S-\hat S\|_F^2
=
\mathbb{E}\|s-\hat s\|_2^2
=
W_2^2(p_s^*,\hat p_s).
\]
Given \(S\), draw \(M\) from its true conditional distribution under the
constraint \(S=U^{*\top}MV^*\). This gives a valid joint distribution over
\((M,\hat M)\) with marginals \(p_M\) and \(\hat p_M\). Therefore,
\begin{equation*}
W_2^2(\hat p_M,p_M)
\le
\mathbb{E}\|M-\hat M\|_F^2.
\end{equation*}

Using \(S=U^{*\top}MV^*\), we have
\[
U^*SV^{*\top}
=
U^*U^{*\top}MV^*V^{*\top}
=
P_UMP_V.
\]
To avoid overloading notation, write the oracle low-rank representation in
Assumption~\ref{assumption:shared subspace} as
\[
M=U_0S_0V_0^\top+E,
\]
where \(S_0\) is the oracle core and \(\|S_0\|_2\le C_S\). Adding and subtracting
the intermediate term \(U^*SV^{*\top}=P_UMP_V\) gives
\begin{equation*}
\begin{aligned}
M-\hat M
&=
\bigl(M-U^*SV^{*\top}\bigr)
+
\bigl(U^*SV^{*\top}-\hat M\bigr)
\\&= U_0S_0V_0^\top+E-P_U(U_0S_0V_0^\top+E)P_V + U^*SV^{*\top}- U^*\hat S V^*
\\
&=
\underbrace{
\bigl(U_0S_0V_0^\top-P_UU_0S_0V_0^\top P_V\bigr)
}_{\displaystyle =:\,\Delta_{\mathrm{sub}}}
+
\underbrace{
U^*(S-\hat S)V^{*\top}
}_{\displaystyle =:\,\Delta_{\mathrm{core}}}
+
\underbrace{
\bigl(E-P_UEP_V\bigr)
}_{\displaystyle =:\,\Delta_{\mathrm{noise}}}.
\end{aligned}
\end{equation*}
Applying Minkowski's inequality in \(L^2\) over the joint distribution of
\((M,S,\hat S)\),
\begin{equation}
\label{eq:minkowski}
\bigl(\mathbb{E}\|M-\hat M\|_F^2\bigr)^{1/2}
\le
\bigl(\mathbb{E}\|\Delta_{\mathrm{sub}}\|_F^2\bigr)^{1/2}
+
\bigl(\mathbb{E}\|\Delta_{\mathrm{core}}\|_F^2\bigr)^{1/2}
+
\bigl(\mathbb{E}\|\Delta_{\mathrm{noise}}\|_F^2\bigr)^{1/2}.
\end{equation}

We first bound the subspace mismatch term. Since
\[
\Delta_{\mathrm{sub}}
=
U_0S_0V_0^\top-P_UU_0S_0V_0^\top P_V,
\]
we decompose
\begin{equation*}
\Delta_{\mathrm{sub}}
=
(I-P_U)U_0S_0V_0^\top
+
P_UU_0S_0V_0^\top(I-P_V).
\end{equation*}
By the triangle inequality and mixed-norm submultiplicativity, and the fact that $\|V_0\|_2=1$ and $\|P_UU_0\|_2\le 1$,
\begin{equation*}
\begin{aligned}
\|\Delta_{\mathrm{sub}}\|_F
&\le
\|(I-P_U)U_0S_0V_0^\top\|_F
+
\|P_UU_0S_0V_0^\top(I-P_V)\|_F \\
&\le
\|(I-P_U)U_0\|_2\|S_0\|_F\|V_0\|_2
+
\|P_UU_0\|_2\|S_0\|_F\|V_0^\top(I-P_V)\|_2 \\
&\le
\bigl(
\sin\Theta(U^*,U_0)+\sin\Theta(V^*,V_0)
\bigr)\|S_0\|_F.
\end{aligned}
\end{equation*}
Taking the \(L^2\)-norm gives
\begin{equation}
\label{eq:bound_subspace}
\begin{aligned}
\bigl(\mathbb{E}\|\Delta_{\mathrm{sub}}\|_F^2\bigr)^{1/2}
&\le
\bigl(\mathbb{E}\|S_0\|_F^2\bigr)^{1/2}
\bigl(
\sin\Theta(U^*,U_0)+\sin\Theta(V^*,V_0)
\bigr) \\
&\le
\sqrt{R}\,C_S
\bigl(
\sin\Theta(U^*,U_0)+\sin\Theta(V^*,V_0)
\bigr),
\end{aligned}
\end{equation}
where we used Assumption~\ref{assumption:shared subspace}(ii): $\mathbb E\|S_0\|_F^2\le R C_S^2.$


Next, since \(U^*\) and \(V^*\) have orthonormal columns,
\begin{equation*}
\mathbb{E}\|\Delta_{\mathrm{core}}\|_F^2
=
\mathbb{E}\|U^*(S-\hat S)V^{*\top}\|_F^2
=
\mathbb{E}\|S-\hat S\|_F^2
=
W_2^2(p_s^*,\hat p_s).
\end{equation*}
By Proposition~\ref{prop:coreflow_fm_chi2} and \(\mathrm{KL}\le \chi^2\), we have
\[
\mathrm{KL}(p_s^*\|\hat p_s)\le C_F\varepsilon.
\]
Since \(\hat p_s=\hat T^{-1}_{\#}\mathcal{N}(0,I)\) and \(\hat T^{-1}\) is
\(e^{L_v}\)-Lipschitz by Gr\"onwall and
Assumption~\ref{asmp:fm_reg}(ii), \(\hat p_s\) satisfies Talagrand's
\(T_2(e^{2L_v})\) inequality~\citep{otto2000generalization,villani2009optimal}.
Therefore,
\begin{equation}
\label{eq:bound_flow}
W_2(p_s^*,\hat p_s)
\le
e^{L_v}\sqrt{2\,\mathrm{KL}(p_s^*\|\hat p_s)}
\le
e^{L_v}\sqrt{2C_F\varepsilon}.
\end{equation}


Finally, the noise term is bounded directly by Assumption~\ref{assumption:shared subspace}.
Since \(E\mapsto P_UEP_V\) is an orthogonal projection in Frobenius norm,
\[
\|E-P_UEP_V\|_F\le \|E\|_F.
\]
Hence
\begin{equation}
\label{eq:bound_noise}
\begin{aligned}
\bigl(\mathbb{E}\|\Delta_{\mathrm{noise}}\|_F^2\bigr)^{1/2}
&=
\bigl(\mathbb{E}\|E-P_UEP_V\|_F^2\bigr)^{1/2} \\
&\le
\bigl(\mathbb{E}\|E\|_F^2\bigr)^{1/2} \\
&\leq 
\sqrt{\min(m_1,m_2)}\cdot
\bigl(\mathbb{E}\|E\|_2^2\bigr)^{1/2} \leq
\sqrt{\min(m_1,m_2)}\;\tau .
\end{aligned}
\end{equation}
Substituting \eqref{eq:bound_subspace}, \eqref{eq:bound_flow}, and
\eqref{eq:bound_noise} into \eqref{eq:minkowski} gives
\[
W_2(\hat p_M,p_M)
\le
\sqrt{R}C_S
\bigl(
\sin\Theta(U^*,U_0)+\sin\Theta(V^*,V_0)
\bigr)
+
e^{L_v}\sqrt{2C_F\varepsilon}
+
\sqrt{\min(m_1,m_2)}\,\tau .
\]
This proves the theorem.
\end{proof}

\begin{remark}[Error metric alignment]
\label{remark:metric_alignment}
Theorem~\ref{thm:end_to_end} unifies errors across 
disparate spaces: ambient Frobenius norm (subspace mismatch 
and data noise) and density-space $\chi^2$-divergence (core 
generation). We bridge these via standard optimal transport 
inequalities, bounding the $W_2$ cost through 
$\mathrm{KL} \le \chi^2$ and Talagrand's $T_2$ inequality. 
Because the inverse flow $\hat{T}^{-1}$ is 
$e^{L_v}$-Lipschitz (by the Gr\"onwall inequality), the 
classical Gaussian $T_2(1)$ 
property~\citep{otto2000generalization} transfers to the 
generated measure $\hat{p}_s$ with constant 
$e^{2L_v}$~\citep{villani2009optimal}. Finally, the 
orthonormality of $(U^*, V^*)$ allows this 
core-space $W_2$ coupling to lift isometrically to the 
ambient Frobenius distance, thus unifying all terms.
\end{remark}

\section{Stage I: Learn $(U,V)$ from Incomplete Matrices}
\label{appendix:stage 1 incomplete}

The Stage I training procedure for incomplete matrices is provided in Algorithm~\ref{alg:stage1-missing}.

\begin{algorithm}[h!]
\caption{Stage I for \textsc{CoreFlow}: Learn $(U,V)$ from Incomplete Matrices}
\label{alg:stage1-missing}
\begin{algorithmic}[1]
\REQUIRE Observations $\{(M_i^{\mathrm{obs}},\Omega_i)\}_{i=1}^N$;
rank $R$; outer iterative steps $E$; inner iterative steps $T_{\mathrm{sub}}$; learning rates $\eta_U,\eta_V$.
\STATE Initialize $U\in\mathrm{St}(m_1,R)$, $V\in\mathrm{St}(m_2,R)$ via spectral initialization.
\STATE Initialize filled-in states $\widetilde M_i \gets \mathcal{P}_{\Omega_i}(M_i^{\mathrm{obs}})$ for all $i$.
\FOR{$e=1$ to $E$}
  \STATE \textbf{(a) Subspace update with masked loss}
  \FOR{$t=1$ to $T_{\mathrm{sub}}$}
    \STATE Sample mini-batch $\mathcal{B}\subset\{1,\dots,N\}$.
    \STATE $\mathcal{L}_{\mathrm{miss}} \gets \frac{1}{|\mathcal{B}|}\sum_{i\in\mathcal{B}}
    \left\|\mathcal{P}_{\Omega_i}\!\left(M_i^{\mathrm{obs}}-U U^\top \widetilde M_i V V^\top\right)\right\|_F^2$.
    \STATE Compute $G_U\gets\nabla_U\mathcal{L}_{\mathrm{miss}}$, $G_V\gets\nabla_V\mathcal{L}_{\mathrm{miss}}$.
    \STATE $U \gets \textsc{StiefelStep}(U,G_U,\eta_U)$;\quad
           $V \gets \textsc{StiefelStep}(V,G_V,\eta_V)$.
  \ENDFOR

  \STATE \textbf{(b) Completion update: keep observed entries, impute missing ones}
  \FOR{$i=1$ to $N$}
    \STATE $\widehat M_i \gets U\,U^\top \widetilde M_i V\,V^\top$.
    \STATE $\widetilde M_i \gets \mathcal{P}_{\Omega_i}(M_i^{\mathrm{obs}})+\mathcal{P}_{\Omega_i^c}(\widehat M_i)$.
  \ENDFOR
\ENDFOR
\STATE \textbf{return} $U^*\gets U,\;V^*\gets V,\; \{\widetilde M_i\}_{i=1}^N$.
\end{algorithmic}
\end{algorithm}

\section{Experimental Details}

\subsection{Metrics}
\label{appendix:metrics}

Let $\{M^{\text{true}}_{b}\}_{b=1}^{B_{\text{true}}}$ denote the set of ground-truth matrices and $\{M^{\text{gen}}_{b}\}_{b=1}^{B_{\text{gen}}}$ denote the set of generated matrices, where each $M \in \mathbb{R}^{m_1 \times m_2}$. We calculate the metrics below that compare the empirical distributions induced by these two sets.

\paragraph{Entrywise moment discrepancy (mean).}
For each entry $j \in \{1,\dots,m_1 m_2\}$ (after flattening the matrix), we define the per-entry empirical means as:
\begin{equation*}
\mu^{\text{true}}_{j} \;=\; \frac{1}{B_{\text{true}}}\sum_{b=1}^{B_{\text{true}}} M^{\text{true}}_{b,j},
\qquad
\mu^{\text{gen}}_{j} \;=\; \frac{1}{B_{\text{gen}}}\sum_{b=1}^{B_{\text{gen}}} M^{\text{gen}}_{b,j},
\end{equation*}
and the absolute per-entry mean difference is $\Delta^{\mu}_{j} = \lvert \mu^{\text{gen}}_{j} - \mu^{\text{true}}_{j} \rvert$.
We report the average of $\{\Delta^\mu_j\}_{j=1}^{m_1m_2}$ across all entries:
\begin{equation*}
\mathrm{AbsEntryMeanDiff} \;=\; \frac{1}{m_1m_2}\sum_{j=1}^{m_1m_2} \Delta^{\mu}_{j}.
\end{equation*}

\paragraph{Entrywise moment discrepancy (standard deviation).}
Analogously, we define the per-entry empirical standard deviations as:
\begin{equation*}
\sigma^{\text{true}}_{j} \;=\; \mathrm{Std}\!\left(\{M^{\text{true}}_{b,j}\}_{b=1}^{B_{\text{true}}}\right),
\qquad
\sigma^{\text{gen}}_{j} \;=\; \mathrm{Std}\!\left(\{M^{\text{gen}}_{b,j}\}_{b=1}^{B_{\text{gen}}}\right),
\end{equation*}
and the absolute per-entry standard-deviation difference $\Delta^{\sigma}_{j} = \lvert \sigma^{\text{gen}}_{j} - \sigma^{\text{true}}_{j} \rvert$.
We report the average of it across all entries:
\begin{equation*}
\mathrm{AbsEntryStdDiff} \;=\; \frac{1}{m_1m_2}\sum_{j=1}^{m_1m_2} \Delta^{\sigma}_{j}.
\end{equation*}
These entrywise metrics quantify how well the generator matches the first two moments \emph{at each spatial location}.

\paragraph{Frobenius-norm distribution shift.}
For each matrix, we compute its Frobenius norm:
\begin{equation*}
r^{\text{true}}_{b} \;=\; \lVert M^{\text{true}}_{b} \rVert_F,
\qquad
r^{\text{gen}}_{b} \;=\; \lVert M^{\text{gen}}_{b} \rVert_F.
\end{equation*}
Let $\bar r^{\text{true}}$ and $\bar r^{\text{gen}}$ denote the empirical means of $\{r^{\text{true}}_{b}\}_{b=1}^{B_{\text{true}}}$ and $\{r^{\text{gen}}_{b}\}_{b=1}^{B_{\text{gen}}}$, and we let $s^{\text{true}}$ and $s^{\text{gen}}$ denote the empirical standard deviations. We report absolute differences
\begin{equation*}
\mathrm{FrobMeanDiff} \;=\; |\bar r^{\text{gen}} - \bar r^{\text{true}}|,
\qquad
\mathrm{FrobStdDiff} \;=\;| s^{\text{gen}} - s^{\text{true}}|.
\end{equation*}
These statistics summarize distributional shifts in overall matrix energy and its variability.

\paragraph{Average singular-value spectrum distance.}
Let $\sigma_i(M)$ denote the $i$-th singular value of $M$, for $i=1,\dots,k$ with $k=\min(m_1,m_2)$.
We define the average singular-value spectra as:
\begin{equation*}
\bar{\boldsymbol{\sigma}}^{\text{true}} \;=\; \frac{1}{B_{\text{true}}}\sum_{b=1}^{B_{\text{true}}} 
\bigl(\sigma_1(M^{\text{true}}_b),\dots,\sigma_k(M^{\text{true}}_b)\bigr),
\qquad
\bar{\boldsymbol{\sigma}}^{\text{gen}} \;=\; \frac{1}{B_{\text{gen}}}\sum_{b=1}^{B_{\text{gen}}} 
\bigl(\sigma_1(M^{\text{gen}}_b),\dots,\sigma_k(M^{\text{gen}}_b)\bigr).
\end{equation*}
We then compute the relative $\ell_2$ discrepancy as:
\begin{equation*}
\mathrm{SVRelL2} \;=\; 
\frac{\left\lVert \bar{\boldsymbol{\sigma}}^{\text{true}} - \bar{\boldsymbol{\sigma}}^{\text{gen}} \right\rVert_2}
{\left\lVert \bar{\boldsymbol{\sigma}}^{\text{true}} \right\rVert_2 + \varepsilon},
\qquad \varepsilon = 10^{-8}.
\end{equation*}
This metric compares the global spectral structure of the two sets of matrices.

We also evaluate how good the estimates are about $U$, $V$, as well as generated low-rank matrix $M$.

\paragraph{Principal Angles.} For simulation cases where $U_{\mathrm{true}}$ and $V_{\mathrm{true}}$ are available, we also measure subspace recovery for both \(U\) and \(V\). Given the ground-truth and estimated factors \(U_{\text{true}},U^*\in\mathbb{R}^{m\times R}\) (analogously \(V_{\text{true}},V^*\)), we first orthonormalize the columns via reduced QR:
\begin{equation*}
Q_{\text{true}}=\mathrm{qr}(U_{\text{true}}),\qquad Q^*=\mathrm{qr}(U^*),
\end{equation*}
and form the overlap matrix \(M=Q_{\text{true}}^\top Q^*\in\mathbb{R}^{R\times R}\). As in our implementation, we compute singular values \(\{\sigma_r\}_{r=1}^R=\mathrm{svdvals}(M)\), clamp them to \([-1,1]\), and define principal angles by
\begin{equation*}
\theta_r=\arccos(\sigma_r)\in\Big[0,\tfrac{\pi}{2}\Big]
\quad(\text{equivalently } \theta_r\in[0,90^\circ]).
\end{equation*}
We report the mean principal angle (in degrees) for \(U\) and \(V\) separately:
\begin{equation*}
\overline{\theta}_U=\frac1R\sum_{r=1}^R \theta_r^{(U)},\qquad
\overline{\theta}_V=\frac1R\sum_{r=1}^R \theta_r^{(V)}.
\end{equation*}

\paragraph{Maximum Mean Discrepancy (MMD).}
To compare the generated and true matrix distributions directly in ambient matrix space, we also report the maximum mean discrepancy (MMD) with an RBF kernel on Frobenius distance. Given true and generated samples $\{M_i^{\mathrm{true}}\}_{i=1}^n$ and $\{\widehat M_j\}_{j=1}^m$, we define
\begin{equation*}
k(M,\widetilde M)=\exp\!\left(-\frac{\|M-\widetilde M\|_F^2}{2\sigma^2}\right),
\end{equation*}
where the bandwidth is chosen by the pooled-sample median,
\begin{equation*}
\sigma^2=\operatorname{median}\Bigl\{\|Z_a-Z_b\|_F^2:\ a<b,\ Z\in\{M_i^{\mathrm{true}}\}_{i=1}^n\cup\{\widehat M_j\}_{j=1}^m\Bigr\}.
\end{equation*}
As in our implementation, we use the unbiased U-statistic estimator
\begin{equation*}
\widehat{\mathrm{MMD}}_u^2
=
\frac{1}{n(n-1)}\sum_{i\neq i'} k(M_i^{\mathrm{true}},M_{i'}^{\mathrm{true}})
+\frac{1}{m(m-1)}\sum_{j\neq j'} k(\widehat M_j,\widehat M_{j'})
-\frac{2}{nm}\sum_{i=1}^n\sum_{j=1}^m k(M_i^{\mathrm{true}},\widehat M_j),
\end{equation*}
and report $\mathrm{MMD}=\sqrt{\max(\widehat{\mathrm{MMD}}_u^2,0)}$, where smaller values indicate better distributional alignment.

\subsection{Patchification details and illustrations}
\label{appendix:patching}

Many real-world matrices are not globally low-rank, yet they often exhibit strong \emph{local} redundancy.
Motivated by patch-based computer vision methods that exploit self-similarity through patch extraction and grouping (e.g., \citep{buades2005non,dabov2007image,dosovitskiy2020image}), we can apply a simple \emph{patchification} operator before learning shared subspaces, which yields an effectively lower-rank representation even when the original matrix is not inherently low-rank.

Specifically, given a matrix $M\in\mathbb{R}^{m_1\times m_2}$, we choose a patch size $p_1\times p_2$ and extract $n_p$ local submatrices (patches), denoted by $\{P_k(M)\in\mathbb{R}^{p_1\times p_2}\}_{k=1}^{n_p}$.
We then form the \emph{patch matrix} by vectorizing each patch and stacking the resulting row vectors:
\begin{equation}
M_{\mathrm{patch}}\;\triangleq\;
\begin{bmatrix}
\mathrm{vec}(P_1(M))^\top\\
\vdots\\
\mathrm{vec}(P_{n_p}(M))^\top
\end{bmatrix}
\in\mathbb{R}^{n_p\times (p_1p_2)}.
\label{equation:patch}
\end{equation}
Because repeated local patterns induce highly correlated rows or columns in $M_{\mathrm{patch}}$, it typically has a much faster singular-value decay (i.e., a lower effective rank) than $M$. See Figures \ref{fig:solar_original_vs_patch} and \ref{fig:rank-drop} for an illustrative spectrum comparison demonstrating the rank reduction by patchification.

\begin{figure}[h!]
  \centering

  \begin{subfigure}[t]{0.98\linewidth}
    \centering
    \includegraphics[width=0.60\textwidth]{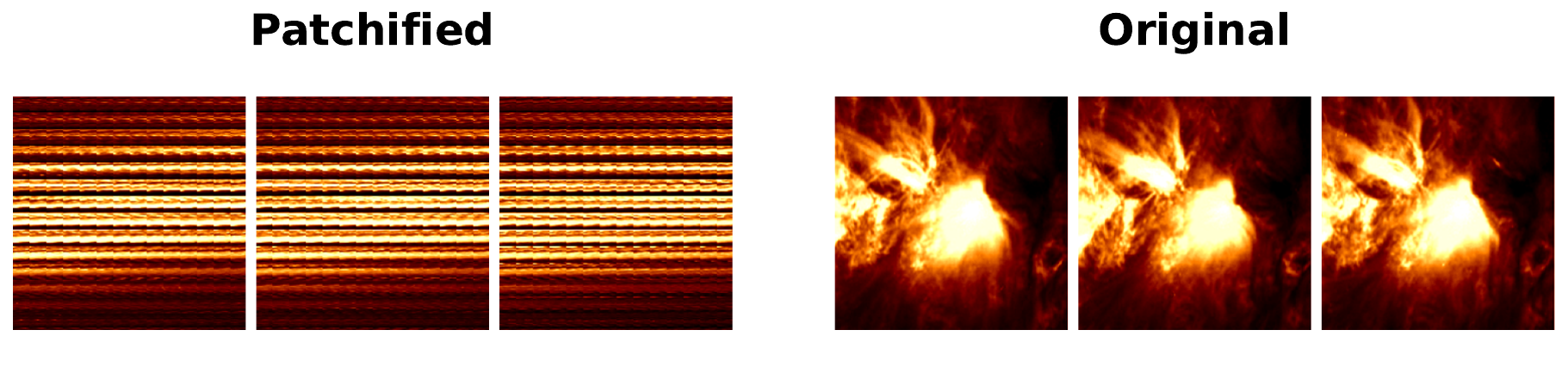}
    \label{fig:solar_original_vs_patch_a}
  \end{subfigure}

  \begin{subfigure}[t]{0.98\linewidth}
    \centering
    \includegraphics[width=0.60\textwidth]{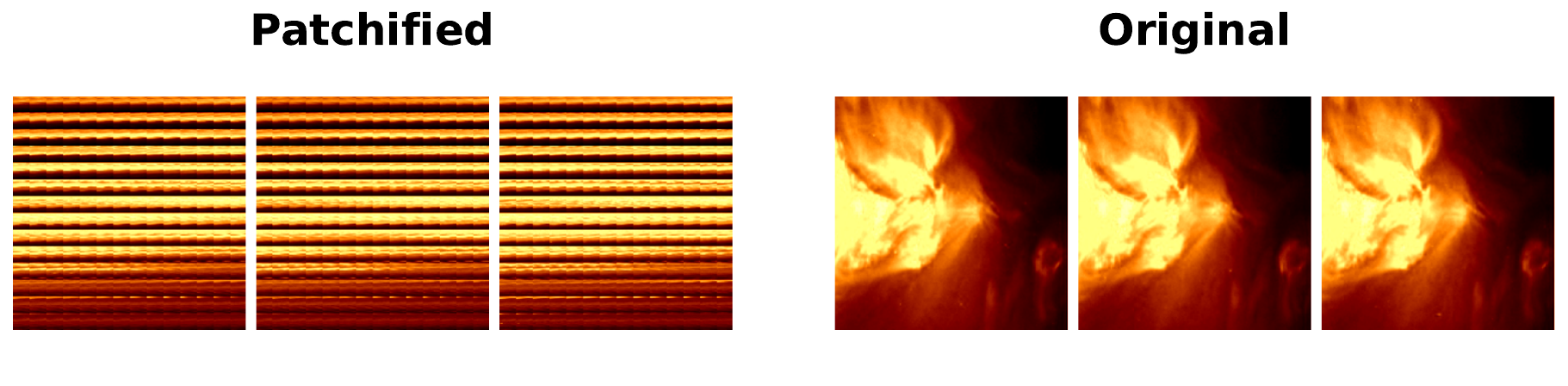}    \label{fig:solar_original_vs_patch_b}
  \end{subfigure}

  \caption{\textbf{Patchified vs.\ original matrices (Solar).} Patchification (Eq.~\eqref{equation:patch}) reshapes local patches into rows of a patch-matrix. In the patchified view (left), many rows/columns become visually near-repetitions due to recurring local textures, while the original matrices (right) remain globally diverse.}
  \label{fig:solar_original_vs_patch}
\end{figure}

\begin{figure}[h!]
  \centering
  \begin{subfigure}[t]{0.30\linewidth}
    \centering
    \includegraphics[width=\linewidth]{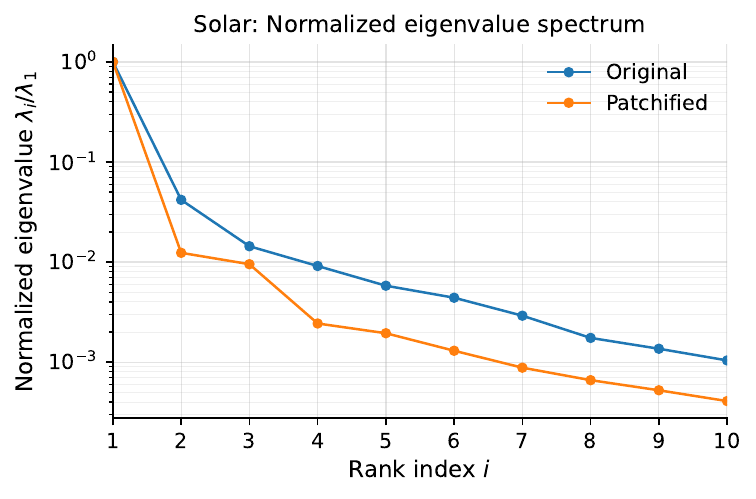}
    \caption{Solar.}
    \label{fig:solar_rank-norm}
  \end{subfigure}%
  \begin{subfigure}[t]{0.30\linewidth}
    \centering
    \includegraphics[width=\linewidth]{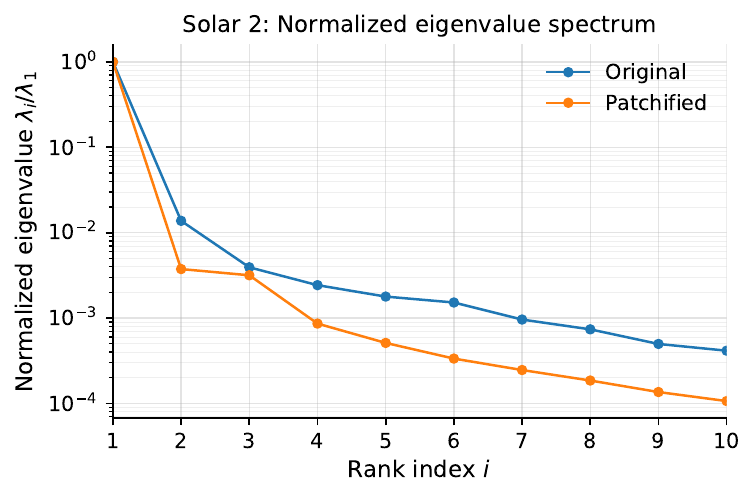}
    \caption{Solar 2.}
    \label{fig:solar2_rank-norm}
  \end{subfigure}
  \caption{\textbf{Consistent rank reduction via patchification.} Normalized spectra ($\lambda_i/\lambda_1$) show that the patchified representation decays faster: for $i\ge2$, the patchified curve lies consistently below the original, indicating that more energy concentrates in the leading components and the representation has a lower effective rank.}
  \label{fig:rank-drop}
\end{figure}

By reorganizing local patches into a patch-matrix via Eq.~\eqref{equation:patch}, patchification can convert matrices that are not \emph{globally} low-rank into an \emph{effectively} lower-rank representation by leveraging local redundancy. As shown in Fig.~\ref{fig:solar_original_vs_patch}, stacking vectorized patches makes recurring local structures (e.g., horizontal bands) appear as near-repeated rows/columns, which increases correlation and reduces the effective rank. Figure~\ref{fig:rank-drop} reports the normalized spectrum $\lambda_i/\lambda_1$ for $i=1,\ldots,10$. Across both Solar and Solar 2, the patchified curve is consistently below the original, indicating faster spectral decay and a lower effective rank.

\textbf{Patchification for Solar datasets.}
For Solar Flare experiments, we apply deterministic non-overlapping patchification before Stage~I. Given \(X_i\in\mathbb{R}^{H\times W}\), we set
\[
p=\operatorname{round}\!\big((HW)^{1/4}\big),\qquad
H_c=\lfloor H/p\rfloor p,\qquad
W_c=\lfloor W/p\rfloor p .
\]
Each matrix is cropped to \(H_c\times W_c\), divided into non-overlapping \(p\times p\) patches, and rearranged into a matrix of size
\[
m_1=\frac{H_cW_c}{p^2},\qquad m_2=p^2 .
\]
For the \(200\times200\) Solar/Solar 2 data, this gives \(p=14\), \(H_c=W_c=196\), and patched matrices of size \(196\times196\). Generated samples are mapped back to the cropped image domain by the inverse patch operation before visualization and metric computation. Patchification is not used for LSPF or synthetic datasets. In all reported comparisons where patchification is used, CoreFlow and all baselines use the same patch size, cropping rule, inverse mapping, and evaluation domain.

\subsection{Synthetic data generation}
\label{appendix:sim_data}

In all simulated data synthesis, we set $m_1 = m_2 = 200$ and rank $R=24$. We generate all synthetic benchmarks from a common form
\begin{equation*}
M \;=\; U S V^\top,\qquad U,V\in\mathbb{R}^{200\times 24},\; S\in\mathbb{R}^{24\times 24}
\end{equation*}
where $U$ and $V$ are fixed within each case and the core $S$ is sampled independently per matrix. We generate $N=1000$ matrices per case and save $(U,V)$ as ground truth for subspace recovery evaluation.

\paragraph{Discrete Cosine Transform (DCT) basis.}
When a case uses a DCT basis, we build an \emph{orthonormal} discrete cosine transform matrix $B\in\mathbb{R}^{m_1\times m_2}$ and take its first $R$ columns. Concretely, for indices $i\in\{0,\dots,m_1{-}1\}$ and $k\in\{0,\dots,m_2{-}1\}$,
\[
B_{k,i} \;=\; \alpha_k \cos\Bigl(\frac{\pi}{n}(i+\tfrac12)k\Bigr),
\qquad
\alpha_0=\sqrt{\tfrac1n},\;\; \alpha_k=\sqrt{\tfrac{2}{n}}\ \ (k\ge 1),
\]
 using the corresponding $n=m_1$ or $n=m_2$. And we set:
 \[
 U = B_{[:,\,0{:}R]},\ V = B_{[:,\,0{:}R]}.
 \]
The DCT-II columns correspond to cosine modes with increasing spatial frequency. Also, restricting to the first $R$ modes biases $M$ toward smooth and low-frequency structure.

\begin{figure}[h!]
    \centering
    \begin{minipage}[t]{0.48\textwidth}
        \centering
        \includegraphics[width=\textwidth]{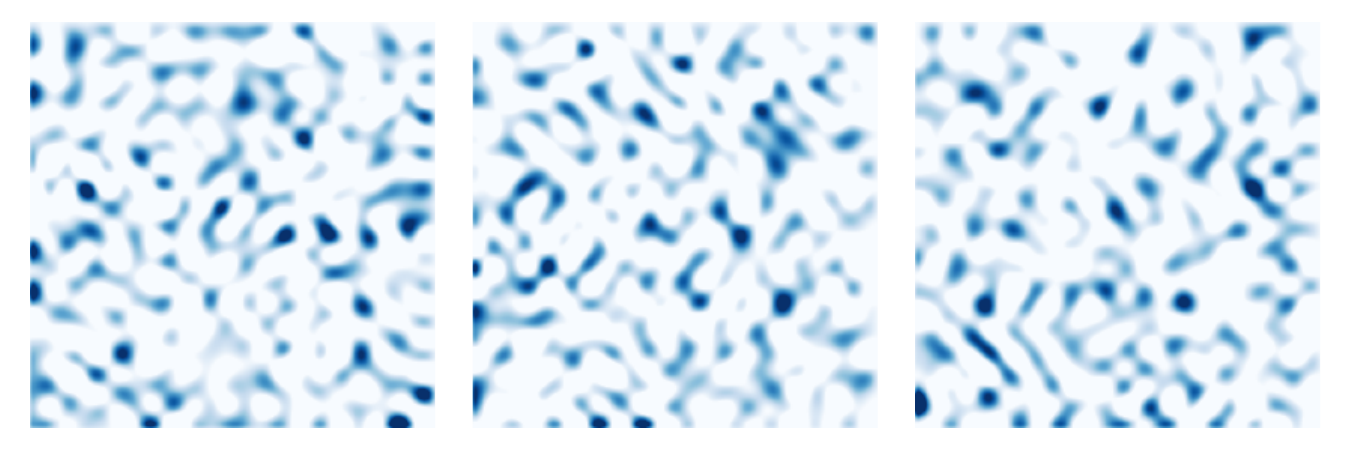}
        \caption{Case Blobs}
        \label{fig:case blobs}
    \end{minipage}
    \hfill
    \begin{minipage}[t]{0.48\textwidth}
        \centering
        \includegraphics[width=\textwidth]{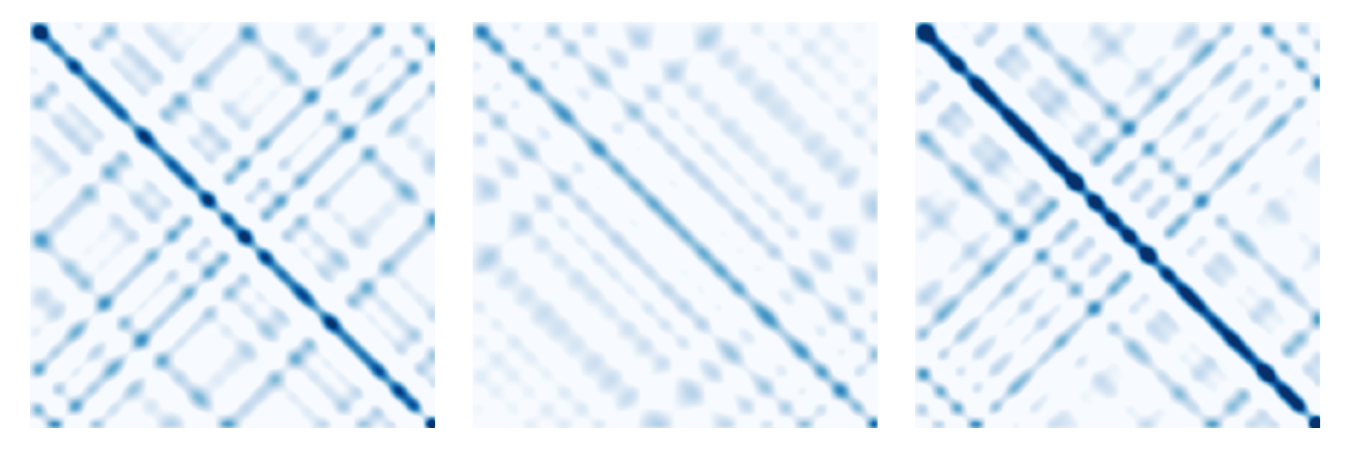}
        \caption{Case Bands}
        \label{fig:case bands}
    \end{minipage}
\end{figure}


\paragraph{Case Blobs.}
We use DCT bases for $U,V$ and sample an i.i.d.\ Gaussian core
\[
S \;=\; 1.5\,Z,\quad Z\overset{\text{i.i.d.}}{\sim}\mathcal{N}(0,I).
\]
Because energy is expressed in low-frequency DCT modes, the resulting matrices appear as smooth blob-like fields. See Figure \ref{fig:case blobs}.


\paragraph{Case Bands.}
We use DCT bases for $U,V$ and sample a \emph{diagonal} core with random mode activation, resulting in structured band patterns and varying effective rank. For each sample, we draw strengths
\[
s_r \sim 1.5\,\mathcal{N}(0,1)+3.0,\qquad r=1,\dots,R,
\]
and independent activations $a_r\sim\mathrm{Bern}(p_r)$ with $p_r$ increasing linearly from $0.2$ to $0.9$ across modes. We then set
\[
S=\mathrm{diag}(s_1a_1,\dots,s_Ra_R).
\]
This yields coherent banded and diagonal structure. Also, the number of active modes (i.e., effective rank) varies across samples. See Figure \ref{fig:case bands}.


\paragraph{Case Waves.}
We use DCT bases for $U,V$ but sample a structured core that concentrates energy on a few low-frequency interactions and adds a mild ``band blur'' in core space. Specifically, let $k_{\max}=\max(4,\lfloor R/3\rfloor)$ and $n_{\mathrm{comp}}=4$. For each sample we initialize $S=0$, then:
(i) draw $(k_j,\ell_j)$ uniformly from $\{0,\dots,k_{\max}{-}1\}^2$ for $j=1,\dots,n_{\mathrm{comp}}$ and add amplitudes $A_j\sim \mathcal{N}(0,1.2^2)$ via $S_{k_j\ell_j}\mathrel{+}=A_j$;
(ii) for band width $b=2$, add Gaussian noise to the first two super- and sub-diagonals:
\[
S_{p,p+d}\mathrel{+}=0.15\,\xi_{p,d},\quad
S_{p+d,p}\mathrel{+}=0.15\,\tilde{\xi}_{p,d},
\qquad d=1,2,
\]
where $\xi_{p,d},\tilde{\xi}_{p,d}\overset{\text{i.i.d.}}{\sim}\mathcal{N}(0,1)$. This produces smooth wave-like patterns in the generated matrices. See Figure \ref{fig:case waves}.

\begin{figure}[h!]
    \centering
    \begin{minipage}[t]{0.48\textwidth}
        \centering
        \includegraphics[width=\textwidth]{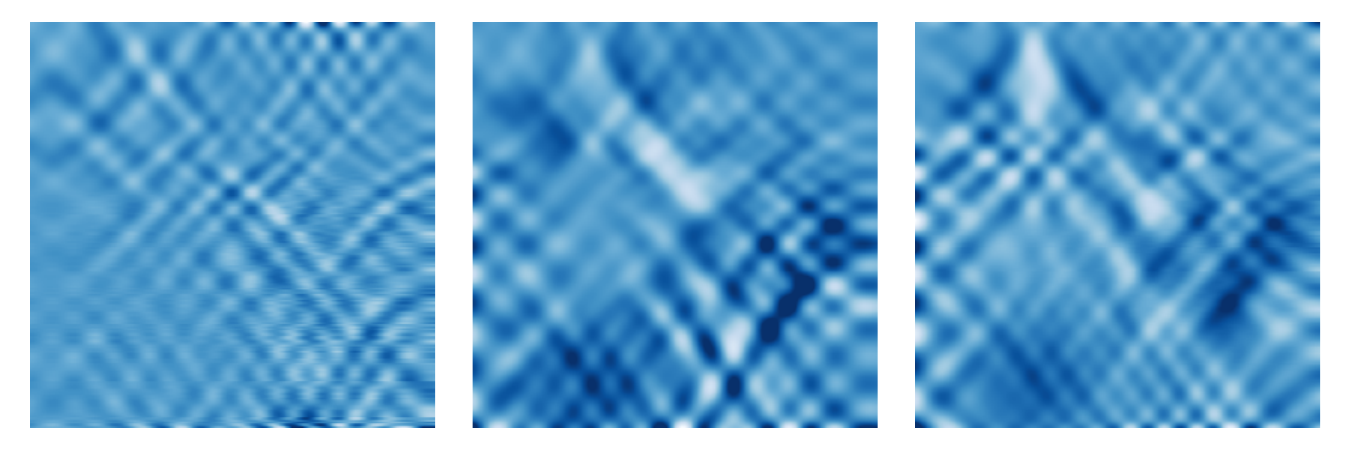}
        \caption{Case Waves}
        \label{fig:case waves}
    \end{minipage}
    \hfill
    \begin{minipage}[t]{0.48\textwidth}
        \centering
        \includegraphics[width=\textwidth]{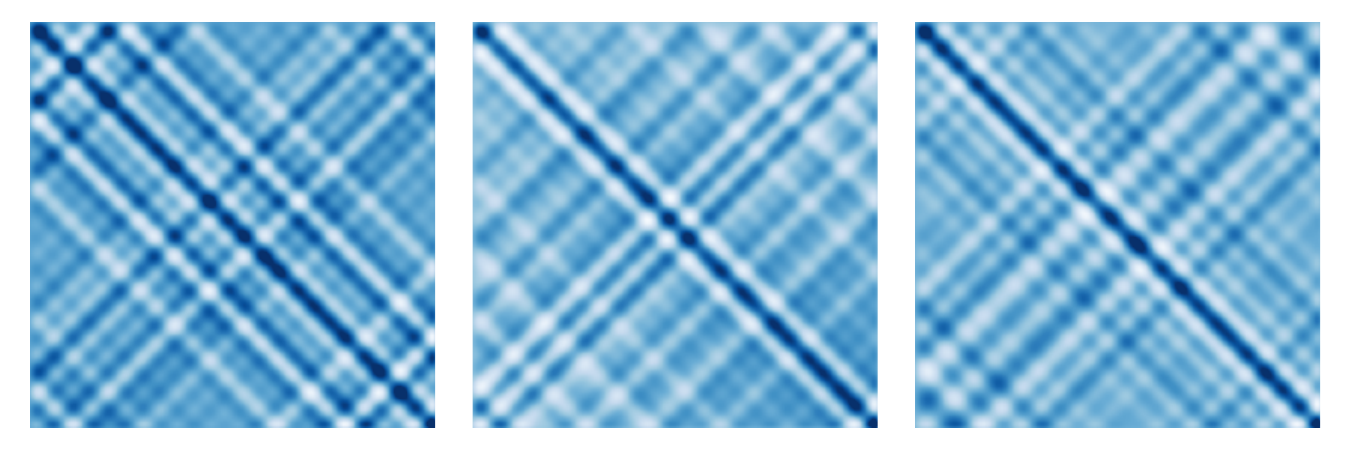}
        \caption{Case crosshatch}
        \label{fig:case crosshatch}
    \end{minipage}
\end{figure}


\paragraph{Case Crosshatch.}
We use DCT bases for $U,V$ and sample diagonal cores with stochastic activation (similar to Case Bands, but with different parameters), yielding strong crosshatched textures and varying effective rank. Specifically, we draw
\[
s_r \sim \mathcal{N}(2.5,1.1^2),\qquad a_r\sim\mathrm{Bern}(p_r),\qquad
p_r \text{ increases linearly from } 0.15 \text{ to } 0.85,
\]
and set $S=\mathrm{diag}(s_1a_1,\dots,s_Ra_R)$. See Figure \ref{fig:case crosshatch}.

\subsection{Neural Network Architecture and Computing Resources}
\label{appendix:NN details}

\begin{table}[h!]
\centering
\caption{Neural Networks parameter counts across different datasets and rank settings. Notably, all models compared share the exact same network structure.}
\renewcommand{\arraystretch}{1.15}
\setlength{\tabcolsep}{8pt}
\begin{tabular}{lcc}
\toprule
\textbf{Setting} & \textbf{Rank} & \textbf{Parameters} \\
\midrule
Case A (Blobs) & 24 & 195,169 \\
Case B (Bands) & 24 & 195,169 \\
Case C (Waves) & 24 & 195,169 \\
Case E (Crosshatch) & 24 & 195,169 \\
LSPF & 56 & 564,481 \\
LSPF & 80 & 614,657 \\
Solar / Solar~2 & 60 & 564,481 \\
Solar / Solar~2 & 120 & 564,481 \\
Solar / Solar~2 & 200 & 614,657 \\
\bottomrule
\end{tabular}
\label{tab:param_count}
\end{table}

Our generative model in Stage II parameterizes the continuous-time velocity field
$v_\theta(x,t)$ on the low-dimensional core space using a lightweight time-conditioned U-Net. 
Stage I does not use a neural network: the shared row and column subspaces $U \in \mathrm{St}(m_1,R)$ and $V \in \mathrm{St}(m_2,R)$ are optimized directly on the Stiefel manifold, and each matrix is represented in the core space by
$s=\mathrm{vec}(U^\top M V) \in \mathbb{R}^{R^2}$. Therefore, only Stage II velocity field is implemented using a Neural Network.

\textbf{Velocity network.}
Given a core vector $x \in \mathbb{R}^{d}$ with $d=R^2$ and time $t \in [0,1]$, we reshape $x$ into a single-channel square feature map and predict the velocity field $v_\theta(x,t) \in \mathbb{R}^{d}$ using a compact U-Net. Time is encoded by a two-layer MLP and injected into every residual block through an additive bias. Each residual block uses
GroupNorm, SiLU activations, two $3\times 3$ convolutions, and a residual skip connection. When the input and output channel dimensions differ, the skip path uses a $1\times 1$ convolution. Downsampling is implemented by average pooling followed by a $3\times 3$ convolution, and upsampling uses nearest-neighbor interpolation followed by a $3\times 3$ convolution. The network includes encoder skip connections, two bottleneck residual blocks, and a final GroupNorm--SiLU--$3\times 3$ output layer. If the reshaped core image size is not divisible by the downsampling factor, we pad symmetrically before the U-Net and crop back to the original size at the output.

\textbf{Architecture hyperparameters.}
The architecture is determined by the following quantities:
\begin{itemize}
    \item \textit{Core rank} $R$, which sets the core dimension $d=R^2$.
    \item \textit{U-Net base width}: $\texttt{base\_channels}=32$.
    \item \textit{Channel multipliers}: either $\texttt{[1,1]}$ for smaller core grids or $\texttt{[1,1,2]}$ for larger ones.
    \item \textit{Residual blocks per resolution}: $\texttt{num\_res\_blocks}=1$.
    \item \textit{Time embedding width}: $\texttt{time\_emb\_dim}=128$.
\end{itemize}

\textbf{Sampling.}
To generate samples, we solve the learned ODE $\dot{x}(t)=v_\theta(x,t)$ from a Gaussian initial state in the core space using an RK4 solver with 101 time steps, and then map the generated core back to the matrix space through the learned subspaces.

\textbf{Baseline architectures.} To ensure a fully comparable evaluation, all baselines except SMG-Core use exactly the same neural network architecture and shared hyperparameters in each experiment. The corresponding parameter counts are reported in Table~\ref{tab:param_count}.


\textbf{Compute.}
All experiments were conducted on a single NVIDIA V100 GPU.
Because the velocity network is shared across methods, Figure~3 in the main text uses the learned-dimension ratio only as a proxy for \emph{Stage~II CNF training cost}.
However, CoreFlow is a two-stage method, so its full training cost also includes Stage~I subspace learning.
Table~\ref{tab:compute_full_pipeline} reports end-to-end training times.

\begin{table*}[h!]
\centering
\caption{
Measured total training time on a single NVIDIA V100 GPU.
For CoreFlow, the total includes both Stage~I subspace training and Stage~II core-space CNF training.
For MissDiff, MissFlow, and CSDM, we report the actual generative model's training time until validation-loss convergence.
}
\label{tab:compute_full_pipeline}
\small
\resizebox{0.97\textwidth}{!}{%
\begin{tabular}{llccccc}
\toprule
Dataset & Method & \(\rho\) & \(p_{\mathrm{miss}}\) & Stage I (min) & Stage II to convergence (min) & Total (min) \\
\midrule
\multirow{9}{*}{Solar / Solar 2}
& CoreFlow & 9\%  & 0\%  & 4.42 & 4.05  & 8.47  \\
& CoreFlow & 9\%  & 20\% & 8.85 & 4.05  & \textbf{12.90} \\
& CoreFlow & 9\%  & 40\% & 8.85 & 4.05  & \textbf{12.90} \\
& CoreFlow & 36\% & 0\%  & 3.17 & 13.23 & \textbf{16.40} \\
& CoreFlow & 36\% & 20\% & 4.76 & 13.23 & \textbf{17.99} \\
& CoreFlow & 36\% & 40\% & 6.35 & 13.23 & \textbf{19.58} \\
\cmidrule(lr){2-7}
& MissDiff / MissFlow & 100\% & all & -- & 37.70 & 37.70 \\
& CSDM & 9\%  & all & -- & 8.29  & \textbf{8.29}  \\
& CSDM & 36\% & all & -- & 23.81 & 23.81 \\
\midrule
\multirow{4}{*}{LSPF}
& CoreFlow & 49\% & 0\%  & 11.36 & 30.20 & \textbf{41.56} \\
& CoreFlow & 49\% & 20\% & 22.73 & 30.20 & \textbf{52.93} \\
& CoreFlow & 49\% & 40\% & 22.73 & 30.20 & \textbf{52.93} \\
\cmidrule(lr){2-7}
& MissDiff / MissFlow & 100\% & all & -- & 83.20 & 83.20 \\
\bottomrule
\end{tabular}%
}
\end{table*}

On Solar and Solar~2, the measured Stage-I training speed is 10.5 epochs/s for \(R=120\) (\(\rho=36\%\)) and 11.3 epochs/s for \(R=60\) (\(\rho=9\%\)). Based on the actual number of Stage-I epochs required under each compression and missingness setting, this corresponds to Stage-I costs of 3.17/4.76/6.35 minutes at \(p_{\mathrm{miss}}=0/20/40\%\) for \(\rho=36\%\), and 4.42/8.85/8.85 minutes for \(\rho=9\%\). Adding the corresponding measured Stage-II times then gives the total CoreFlow training time.

On LSPF, Stage~I runs at 2.5 epochs/s on average. Based on the actual number of Stage-I epochs required under each missingness setting, this corresponds to 10.0/20.0/20.0 minutes for \(p_{\mathrm{miss}}=0/20/40\%\), respectively. Adding the corresponding measured Stage-II time then gives the total CoreFlow training time. For the other baselines, Table~\ref{tab:compute_full_pipeline} reports their actual Stage-II wall-clock time to convergence directly. For the simulations, Stage~I is negligible at \(p_{\mathrm{miss}}=0\%\) because Tucker spectral initialization is used directly, and costs about 8.62 minutes under missingness (1500 epochs at 2.9 epochs/s).

\subsection{SMG-Core implementation details}
\label{sec:baseline_gaussiancore_bayes}

\paragraph{Why an adaptation is needed.}
The original BayeSMG framework is designed for \emph{matrix completion}: it performs conditional inference for a partially observed matrix under an SMG prior, using the observed entries to infer posterior uncertainty about the unobserved ones. Our task here is different: we study \emph{unconditional generation} of new matrices from a training distribution. Therefore, the original BayeSMG procedure is not directly applicable as a generative baseline, since at generation time there is no partially observed test matrix to condition on. To obtain a controlled low-dimensional baseline for our setting, we fix the shared subspaces learned in Stage~I and use the BayeSMG Gaussian variance-prior structure only to model the core distribution. This yields SMG-Core, which isolates the effect of replacing CoreFlow's nonlinear core generator by a Gaussian core model under the same shared geometry.

\textit{Normal--Inverse--Wishart prior.}
We place a conjugate Normal--Inverse--Wishart (NIW) prior on the unknown mean $\mu\in\mathbb{R}^d$
and covariance $\Sigma\in\mathbb{R}^{d\times d}$:
\begin{equation*}
\Sigma \sim \mathrm{InvWishart}(\nu_0, \Psi_0),
\qquad
\mu \mid \Sigma \sim \mathcal{N}\!\left(\mu_0,\; \frac{1}{\kappa_0}\Sigma\right),
\end{equation*}
where $\kappa_0>0$, $\nu_0>d-1$, and $\Psi_0\succ 0$.

\textit{Posterior.}
Given data $\{x_i\}_{i=1}^N$, define the empirical mean $\bar{x}=\frac{1}{N}\sum_i x_i$ and scatter matrix
\begin{equation*}
S \;=\; \sum_{i=1}^N (x_i-\bar{x})(x_i-\bar{x})^\top.
\end{equation*}
The NIW posterior is again NIW with parameters
\begin{align*}
\kappa_N &= \kappa_0 + N,
&
\mu_N &= \frac{\kappa_0 \mu_0 + N\bar{x}}{\kappa_N},
\\
\nu_N &= \nu_0 + N,
&
\Psi_N &= \Psi_0 + S + \frac{\kappa_0 N}{\kappa_N}(\bar{x}-\mu_0)(\bar{x}-\mu_0)^\top.
\end{align*}

\textit{Posterior predictive sampling.}
A Bayesian posterior predictive sample can be drawn in either of the following equivalent ways:

(i) Two-stage sampling:
\begin{equation*}
\Sigma \sim \mathrm{InvWishart}(\nu_N,\Psi_N),
\qquad
\mu \sim \mathcal{N}\!\left(\mu_N,\frac{1}{\kappa_N}\Sigma\right),
\qquad
x_{\mathrm{new}} \sim \mathcal{N}(\mu,\Sigma).
\end{equation*}

(ii) Collapsed predictive distribution:
Integrating out $(\mu,\Sigma)$ yields a multivariate Student-$t$ distribution:
\begin{equation*}
x_{\mathrm{new}} \sim t_{\nu_N-d+1}\!\left(\mu_N,\; \frac{\kappa_N+1}{\kappa_N(\nu_N-d+1)}\Psi_N\right).
\end{equation*}

We then reshape $x_{\mathrm{new}}$ into a core matrix and decode:
\begin{equation*}
T_{\mathrm{new}} = \mathrm{unvec}(x_{\mathrm{new}}),
\qquad
M_{\mathrm{new}} = U\, T_{\mathrm{new}}\, V^\top.
\end{equation*}

\section{Additional Results}
This appendix provides additional qualitative and quantitative results that complement the main paper.


\subsection{Qualitative cross-method sample comparisons}
\label{appendix:additional results}

\begin{figure}[h!]
    \centering
    \includegraphics[width=0.60\textwidth]{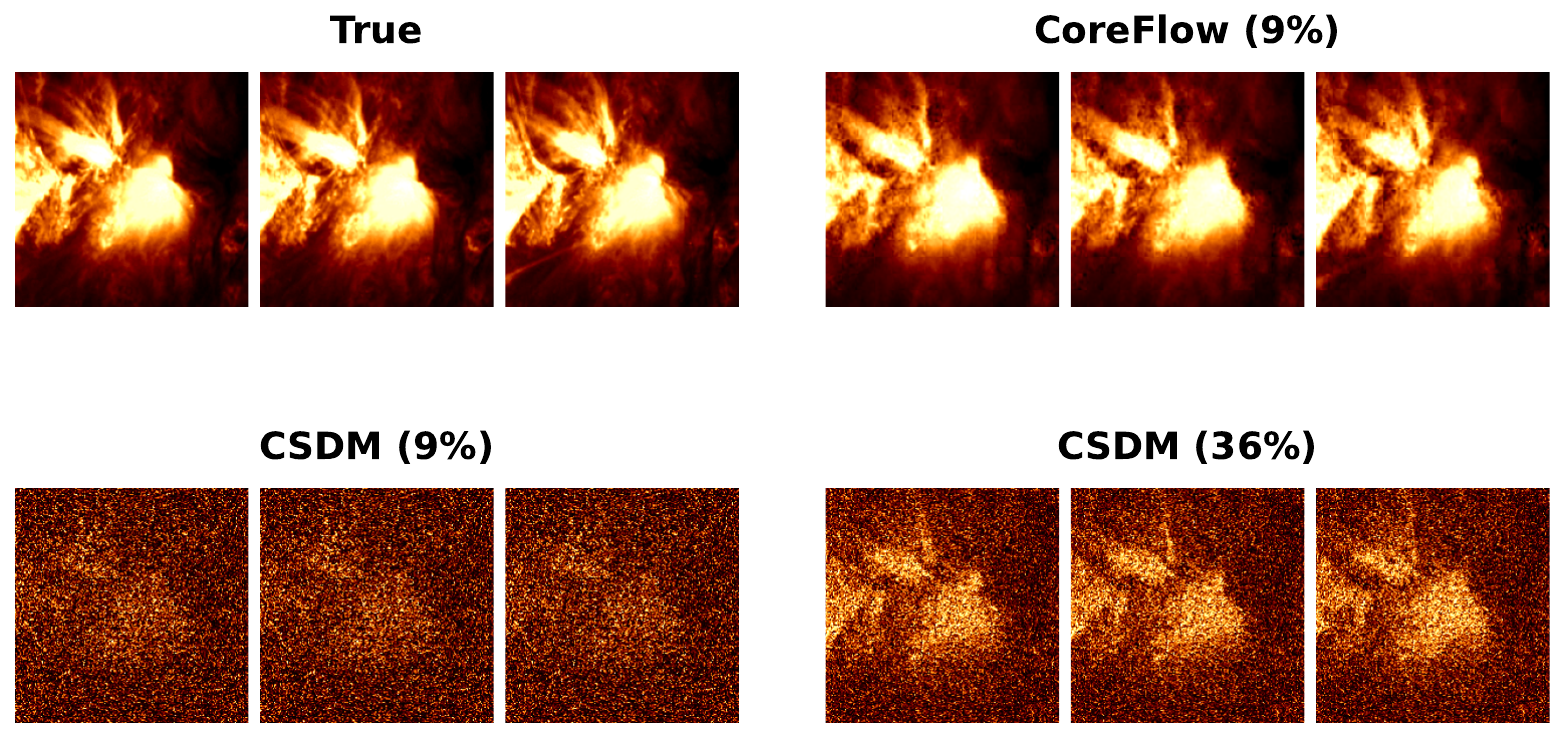}
    \caption{Comparisons of generated samples on Solar Flare with $p_{\mathrm{miss}}=0\%$ (training on complete matrices). CoreFlow is trained at $\mathbf{\rho=9\%}$ of the original dimension, and is compared against CSDM trained at $\rho\in\{9\%,36\%\}$.}
    \label{fig:cross method solar no missing}
\end{figure}



\begin{figure}[h!]
    \centering
    \begin{minipage}[t]{0.48\textwidth}
        \centering
        \includegraphics[width=\textwidth]{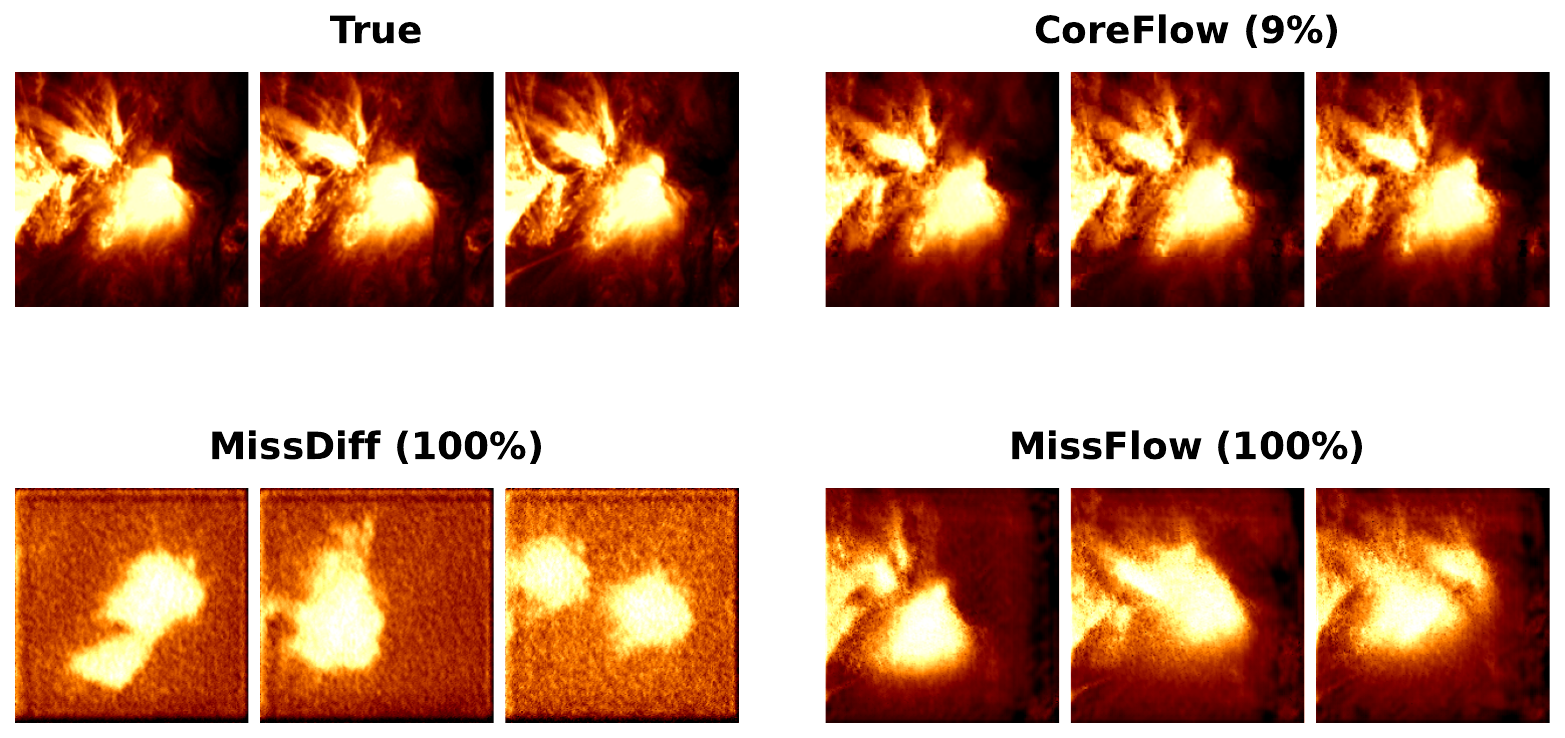}
    \end{minipage}
    \hfill
    \begin{minipage}[t]{0.48\textwidth}
        \centering
        \includegraphics[width=\textwidth]{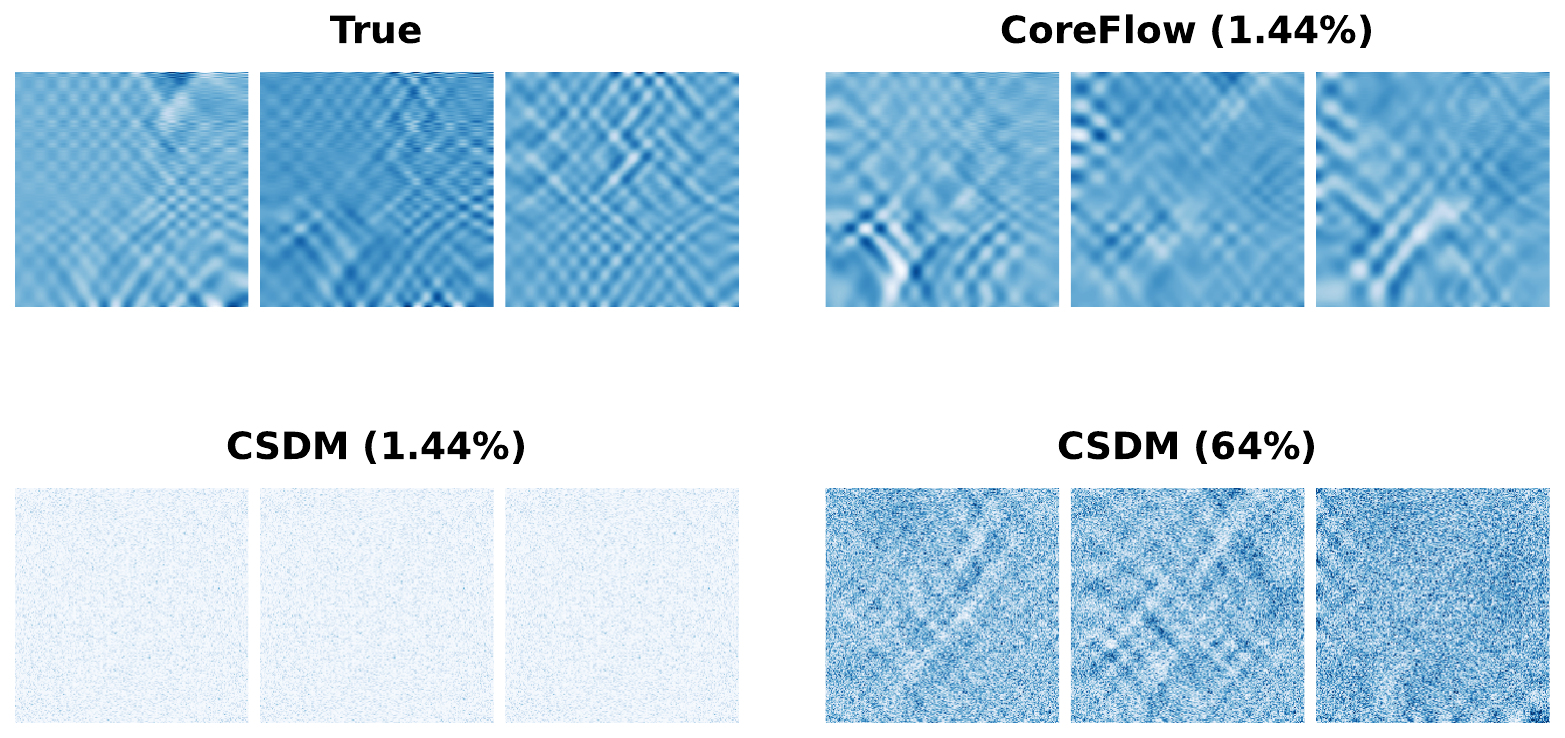}
    \end{minipage}
    
    \caption{
    Comparison of generated samples across methods.
    \textit{Left:} Solar Flare with $p_{\mathrm{miss}}=40\%$. CoreFlow is trained at $\rho=9\%$ of the original dimension, and is compared against MissDiff and MissFlow trained at $\rho=100\%$.
    \textit{Right:} Case Waves with $p_{\mathrm{miss}}=0\%$ (complete training matrices). CoreFlow is trained at $\rho=1.44\%$ of the original dimension, and is compared against CSDM trained at $\rho\in\{1.44\%,64\%\}$.
    }
    \label{fig:cross_method_comparison}
\end{figure}

Figures~\ref{fig:cross method solar no missing} and \ref{fig:cross_method_comparison} visualize representative generated matrices under both complete training ($p_{\mathrm{miss}}=0\%$) and missing-entry training.
On Solar Flare with complete training (Figure~\ref{fig:cross method solar no missing}), CoreFlow at $\rho=9\%$ reproduces the salient flare morphology, matching the visual statistics of real samples.
In contrast, CSDM at the same compression ($\rho=9\%$) loses the global structure.
This gap is primarily due to information loss induced by CSDM's compression operator: its compressed measurements may not preserve spatial locality and can discard geometry information at aggressive $\rho$.

Under missing entries (Figure~\ref{fig:cross_method_comparison} \textit{Left}), CoreFlow maintains coherent spatial patterns and realistic intensity variation despite operating in a reduced core space, while full-space baselines (MissDiff and MissFlow) fail to learn the matrix distribution in the $200\times200$ ambient dimension with limited samples.
Finally, in the Waves simulation (Figure~\ref{fig:cross_method_comparison}, \textit{right}), CoreFlow preserves the characteristic smooth oscillatory textures even at extreme compression ($\rho=1.44\%$).

\subsection{Real data: robustness to missing entries}

\begin{table}[h!]
\centering
\caption{Comparisons to baselines with \textit{missingness handling} on real data (Frobenius-norm distribution shift: FrobMeanDiff and FrobStdDiff; lower is better). Column headers denote missing-entry rate. Parentheses report each method's learned-dimension ratio (CoreFlow: $\rho{=}36\%$ for Solar and $\rho{=}49\%$ for LSPF; all ambient-space baselines use $\rho{=}100\%$). Solar is a few-sample regime (300 samples) and LSPF a standard regime (8760 samples).}
\renewcommand{\arraystretch}{1.2}
\setlength{\tabcolsep}{4pt}

\scalebox{0.55}{
\begin{tabular}{lcccccc}
\toprule
& \multicolumn{3}{c}{\textbf{Solar}}
& \multicolumn{3}{c}{\textbf{LSPF}} \\
\cmidrule(lr){2-4}\cmidrule(lr){5-7}

$p_{\mathrm{miss}}$
& 0\% & 20\% & 40\%
& 0\% & 20\% & 40\% \\
\midrule

MissDiff ($\rho=100\%$)      & $(0.85_{\pm .13}, 0.62_{\pm .11})$ & $(37.4_{\pm .19}, 0.72_{\pm .24})$ & $(71.2_{\pm .30}, 0.91_{\pm .26})$ & $(\mathbf{0.21_{\pm .11}}, 0.18_{\pm .08})$ & $(13.5_{\pm .17}, 0.61_{\pm .16})$ & $(16.7_{\pm .21}, 0.82_{\pm .19})$ \\
MissFlow ($\rho=100\%$)      & $(\mathbf{0.13_{\pm .10}}, 0.53_{\pm .10})$ & $(21.2_{\pm .13}, 0.32_{\pm .11})$ & $(54.3_{\pm .14}, 0.41_{\pm .13})$ & $(\mathbf{0.20_{\pm .12}}, \mathbf{0.19_{\pm .09}})$ & $(11.9_{\pm .15}, 0.44_{\pm .16})$ & $(15.9_{\pm .11}, 0.72_{\pm .12})$ \\
\rowcolor{coreflowbg}
\textbf{CoreFlow} ($\rho=36\%/49\%$) & $(0.56_{\pm .28}, \mathbf{0.39_{\pm .21}})$ & $(\mathbf{0.30_{\pm .14}}, \mathbf{0.45_{\pm .09}})$ & $(\mathbf{0.84_{\pm .15}}, \mathbf{0.40_{\pm .08}})$ & $(0.24_{\pm .09}, 0.12_{\pm .06})$ & $(\mathbf{0.67_{\pm .09}}, \mathbf{0.20_{\pm .06}})$ & $(\mathbf{0.98_{\pm .08}}, \mathbf{0.44_{\pm .05}})$ \\
\bottomrule
\end{tabular}}
\label{tab:real frob 1}
\end{table}

\begin{table}[h!]
\centering
\caption{Comparisons to baselines with \textit{missingness handling} on real data (entry-wise moment discrepancies: AbsEntryMeanDiff and AbsEntryStdDiff; lower is better). Column headers denote missing-entry rate. Parentheses report each method's learned-dimension ratio (CoreFlow: $\rho{=}36\%$ for Solar and $\rho{=}49\%$ for LSPF; all ambient-space baselines use $\rho{=}100\%$). Solar is a few-sample regime (300 samples) and LSPF a standard regime (8760 samples). All values are in $\times 10^{-2}$.}
\renewcommand{\arraystretch}{1.2}
\setlength{\tabcolsep}{4pt}

\scalebox{0.55}{
\begin{tabular}{lcccccc}
\toprule
& \multicolumn{3}{c}{\textbf{Solar}}
& \multicolumn{3}{c}{\textbf{LSPF}} \\
\cmidrule(lr){2-4}\cmidrule(lr){5-7}

$p_{\mathrm{miss}}$
& 0\% & 20\% & 40\%
& 0\% & 20\% & 40\% \\
\midrule

MissDiff ($\rho=100\%$)      & $(12.3_{\pm .28}, 5.91_{\pm .29})$ & $(29.3_{\pm .20}, 6.00_{\pm .24})$ & $(39.1_{\pm .25}, 6.96_{\pm .28})$ & $(\mathbf{1.69_{\pm .22}}, 1.74_{\pm .20})$ & $(15.6_{\pm .18}, 2.53_{\pm .20})$ & $(21.7_{\pm .14}, 3.07_{\pm .19})$ \\
MissFlow ($\rho=100\%$)      & $(7.16_{\pm .39}, 5.41_{\pm .38})$ & $(16.4_{\pm .18}, 5.69_{\pm .25})$ & $(33.4_{\pm .19}, 6.68_{\pm .27})$ & $(1.71_{\pm .35}, \mathbf{1.70_{\pm .23}})$ & $(14.8_{\pm .17}, 2.65_{\pm .28})$ & $(20.4_{\pm .12}, 3.23_{\pm .21})$ \\
\rowcolor{coreflowbg}
\textbf{CoreFlow} ($\rho=36\%/49\%$) 
& $(\mathbf{1.02_{\pm .17}}, \mathbf{0.92_{\pm .33}})$ 
& $(\mathbf{0.87_{\pm .08}}, \mathbf{1.19_{\pm .13}})$ 
& $(\mathbf{0.97_{\pm .03}}, \mathbf{1.32_{\pm .04}})$ 
& $(1.87_{\pm .11}, 1.89_{\pm .12})$ 
& $(\mathbf{1.69_{\pm .10}}, \mathbf{1.96_{\pm .09}})$ 
& $(\mathbf{1.62_{\pm .07}}, \mathbf{2.38_{\pm .05}})$ \\
\bottomrule
\end{tabular}}
\label{tab:real entrywise 1}
\end{table}

Tables~\ref{tab:real frob 1} and~\ref{tab:real entrywise 1} report distribution-shift metrics on three real datasets (Solar, Solar~2, and LSPF) under $p_{\mathrm{miss}}\in\{0,20,40\}\%$. On Solar and Solar~2 (few-data regime), full-space baselines can be competitive at $p_{\mathrm{miss}}=0\%$ but degrade sharply as missingness increases, with FrobMeanDiff growing by orders of magnitude at $20\%$. On LSPF (standard regime), all methods are closer at $p_{\mathrm{miss}}=0\%$, but CoreFlow achieves clearly smaller shifts once missingness is introduced ($20\%$/$40\%$).

\subsection{Simulations: distribution matching under controlled structure}





\begin{table}[h!]
\centering
\caption{Comparisons to \textit{low-dimensional} generative baselines on Blobs and Bands simulations (Frobenius-norm distribution shift: FrobMeanDiff and FrobStdDiff; lower is better). Column headers denote missing-entry rate. Parentheses report each method's learned-dimension ratio.}
\renewcommand{\arraystretch}{1.2}
\setlength{\tabcolsep}{4pt}

\scalebox{0.55}{
\begin{tabular}{lcccccc}
\toprule
& \multicolumn{3}{c}{\textbf{Blobs}}
& \multicolumn{3}{c}{\textbf{Bands}} \\
\cmidrule(lr){2-4}\cmidrule(lr){5-7}

$p_{\mathrm{miss}}$
& 0\% & 20\% & 40\%
& 0\% & 20\% & 40\% \\
\midrule

CSDM ($\rho=1.44\%$)
& $(130_{\pm 6.9}, 0.70_{\pm .38})$
& $(155_{\pm 5.1}, 0.76_{\pm .36})$
& $(178_{\pm 4.1}, 0.84_{\pm .31})$
& $(40.5_{\pm 1.3}, 1.61_{\pm .18})$
& $(48.7_{\pm 2.3}, 1.61_{\pm .13})$
& $(55.9_{\pm 1.8}, 1.62_{\pm .18})$ \\
CSDM ($\rho=36\%$)
& $(76.1_{\pm 5.5}, 0.69_{\pm .44})$
& $(93.4_{\pm 4.3}, 0.72_{\pm .34})$
& $(108_{\pm 3.9}, 0.76_{\pm .34})$
& $(23.4_{\pm 1.5}, 1.57_{\pm .33})$
& $(28.9_{\pm 1.8}, 1.56_{\pm .22})$
& $(33.7_{\pm 2.4}, 1.56_{\pm .29})$ \\
CSDM ($\rho=64\%$)
& $(50.6_{\pm 5.1}, 0.71_{\pm .51})$
& $(64.1_{\pm 4.4}, 0.74_{\pm .25})$
& $(75.8_{\pm 5.3}, 0.77_{\pm .24})$
& $(15.6_{\pm 1.1}, 1.56_{\pm .29})$
& $(19.9_{\pm 1.3}, 1.55_{\pm .37})$
& $(23.6_{\pm 1.4}, 1.53_{\pm .34})$ \\
SMG-Core ($\rho=1.44\%$)
& $(70.2_{\pm 6.0}, 2.30_{\pm .94})$
& $(79.7_{\pm 6.3}, 2.15_{\pm .90})$
& $(80.2_{\pm 6.2}, 2.03_{\pm .88})$
& $(14.2_{\pm .80}, 2.28_{\pm .34})$
& $(16.3_{\pm 0.8}, 2.14_{\pm .37})$
& $(19.9_{\pm 1.0}, 2.37_{\pm .36})$ \\
\rowcolor{coreflowbg}
\textbf{CoreFlow} ($\rho=1.44\%$)
& $(\mathbf{0.09_{\pm .06}}, \mathbf{0.20_{\pm .05}})$
& $(\mathbf{0.19_{\pm .06}}, \mathbf{0.12_{\pm .06}})$
& $(\mathbf{0.44_{\pm .07}}, \mathbf{0.15_{\pm .05}})$
& $(\mathbf{0.08_{\pm .06}}, \mathbf{0.12_{\pm .06}})$
& $(\mathbf{0.85_{\pm .08}}, \mathbf{0.23_{\pm .06}})$
& $(\mathbf{0.98_{\pm .07}}, \mathbf{0.47_{\pm .04}})$ \\
\bottomrule
\end{tabular}}
\label{tab:simulation blobs bands frob}
\end{table}

\begin{table}[h!]
\centering
\caption{Comparisons to \textit{low-dimensional} generative baselines on Waves and Crosshatch simulations (Frobenius-norm distribution shift: FrobMeanDiff and FrobStdDiff; lower is better). Column headers denote missing-entry rate. Parentheses report learned-dimension ratio.}
\renewcommand{\arraystretch}{1.2}
\setlength{\tabcolsep}{4pt}

\scalebox{0.55}{
\begin{tabular}{lcccccc}
\toprule
& \multicolumn{3}{c}{\textbf{Waves}}
& \multicolumn{3}{c}{\textbf{Crosshatch}} \\
\cmidrule(lr){2-4}\cmidrule(lr){5-7}

$p_{\mathrm{miss}}$
& 0\% & 20\% & 40\%
& 0\% & 20\% & 40\% \\
\midrule

CSDM ($\rho=1.44\%$)
& $(83.5_{\pm 5.6}, 9.84_{\pm 1.3})$
& $(85.4_{\pm 5.0}, 10.1_{\pm 1.3})$
& $(87.4_{\pm 4.3}, 10.3_{\pm 1.2})$
& $(62.4_{\pm 3.4}, 7.67_{\pm 1.2})$
& $(63.8_{\pm 3.0}, 7.82_{\pm 1.1})$
& $(65.31_{\pm 2.5}, 8.00_{\pm 1.2})$ \\
CSDM ($\rho=36\%$)
& $(34.8_{\pm 2.2}, 4.23_{\pm .96})$
& $(41.8_{\pm 2.9}, 5.03_{\pm .85})$
& $(49.6_{\pm 2.7}, 5.94_{\pm .72})$
& $(25.6_{\pm 1.3}, 3.32_{\pm .64})$
& $(30.9_{\pm 1.2}, 3.94_{\pm .66})$
& $(36.8_{\pm 1.2}, 4.63_{\pm .56})$ \\
CSDM ($\rho=64\%$)
& $(17.8_{\pm 2.7}, 2.22_{\pm .32})$
& $(26.5_{\pm 3.4}, 3.23_{\pm .51})$
& $(36.4_{\pm 2.1}, 4.38_{\pm .72})$
& $(12.88_{\pm 1.6}, 1.79_{\pm .45})$
& $(19.46_{\pm 1.5}, 2.56_{\pm .43})$
& $(26.9_{\pm 1.6}, 3.45_{\pm .51})$ \\
SMG-Core ($\rho=1.44\%$)
& $(11.6_{\pm 1.3}, 4.10_{\pm .47})$
& $(14.8_{\pm 1.9}, 3.79_{\pm .29})$
& $(15.7_{\pm 2.0}, 4.97_{\pm .87})$
& $(5.94_{\pm .90}, 2.14_{\pm .40})$
& $(7.28_{\pm 1.1}, 2.31_{\pm .44})$
& $(9.75_{\pm 1.0}, 2.76_{\pm .48})$ \\
\rowcolor{coreflowbg}
\textbf{CoreFlow} ($\rho=1.44\%$)
& $(\mathbf{0.55_{\pm .38}}, \mathbf{0.53_{\pm .32}})$
& $(\mathbf{0.47_{\pm .38}}, \mathbf{2.87_{\pm .37}})$
& $(\mathbf{3.53_{\pm .63}}, \mathbf{0.87_{\pm .39}})$
& $(\mathbf{0.96_{\pm .31}}, \mathbf{0.75_{\pm .30}})$
& $(\mathbf{1.07_{\pm .39}}, \mathbf{0.95_{\pm .38}})$
& $(\mathbf{1.60_{\pm .48}}, \mathbf{0.96_{\pm .36}})$ \\
\bottomrule
\end{tabular}}
\label{tab:simulation waves crosshatch frob}
\end{table}

Tables~\ref{tab:simulation blobs bands frob} and~\ref{tab:simulation waves crosshatch frob} report Frobenius-norm distribution shift on synthetic families; Tables~\ref{tab:simulation blobs bands entrywise} and~\ref{tab:simulation waves crosshatch entrywise} report entry-wise moment discrepancies. Across Blobs/Bands/Waves/Crosshatch, CoreFlow at $\rho=1.44\%$ achieves near-zero discrepancies and remains small as $p_{\mathrm{miss}}$ increases.
In contrast, low-dimensional generative baselines exhibit substantially larger shifts.


\begin{table}[h!]
\centering
\caption{Comparisons to \textit{low-dimensional} generative baselines on Blobs and Bands simulations (entry-wise moment discrepancies: AbsEntryMeanDiff and AbsEntryStdDiff; lower is better). Column headers denote missing-entry rate. Parentheses report each method's learned-dimension ratio. All values are in $\times 10^{-2}$.}
\renewcommand{\arraystretch}{1.2}
\setlength{\tabcolsep}{4pt}

\scalebox{0.55}{
\begin{tabular}{lcccccc}
\toprule
& \multicolumn{3}{c}{\textbf{Blobs}}
& \multicolumn{3}{c}{\textbf{Bands}} \\
\cmidrule(lr){2-4}\cmidrule(lr){5-7}

$p_{\mathrm{miss}}$
& 0\% & 20\% & 40\%
& 0\% & 20\% & 40\% \\
\midrule

CSDM ($\rho=1.44\%$)
& $(63.6_{\pm 2.0}, 34.1_{\pm 1.8})$
& $(84.8_{\pm 2.4}, 24.3_{\pm 1.3})$
& $(106_{\pm 2.2}, 15.5_{\pm 1.2})$
& $(20.5_{\pm 1.2}, 12.2_{\pm .14})$
& $(27.2_{\pm 1.1}, 9.10_{\pm .09})$
& $(34.0_{\pm 1.2}, 3.67_{\pm .12})$ \\
CSDM ($\rho=36\%$)
& $(37.5_{\pm 1.9}, 17.2_{\pm 1.1})$
& $(50.0_{\pm 1.7}, 12.3_{\pm 1.3})$
& $(62.6_{\pm 1.5}, 6.25_{\pm 1.6})$
& $(12.3_{\pm 1.2}, 6.37_{\pm .18})$
& $(16.2_{\pm 1.1}, 4.70_{\pm .19})$
& $(20.2_{\pm 1.2}, 1.53_{\pm .20})$ \\
CSDM ($\rho=64\%$)
& $(26.9_{\pm .75}, 11.0_{\pm 1.4})$
& $(35.8_{\pm .81}, 7.87_{\pm 1.3})$
& $(106_{\pm 1.2}, 15.5_{\pm .20})$
& $(8.98_{\pm .91}, 4.21_{\pm .08})$
& $(11.7_{\pm .04}, 3.12_{\pm .13})$
& $(14.5_{\pm 1.3}, 0.87_{\pm .05})$ \\
SMG-Core ($\rho=1.44\%$)
& $(34.9_{\pm 1.2}, 3.92_{\pm .82})$
& $(36.5_{\pm 1.2}, 3.84_{\pm .86})$
& $(41.3_{\pm 1.6}, 4.02_{\pm 1.0})$
& $(8.26_{\pm .57}, 1.02_{\pm .07})$
& $(10.3_{\pm .63}, 1.74_{\pm .11})$
& $(12.9_{\pm .89}, 1.55_{\pm .12})$ \\
\rowcolor{coreflowbg}
\textbf{CoreFlow} ($\rho=1.44\%$)
& $(\mathbf{0.98_{\pm .04}}, \mathbf{0.74_{\pm .02}})$
& $(\mathbf{0.99_{\pm .03}}, \mathbf{0.72_{\pm .01}})$
& $(\mathbf{1.01_{\pm .04}}, \mathbf{0.72_{\pm .02}})$
& $(\mathbf{0.29_{\pm .03}}, \mathbf{0.19_{\pm .01}})$
& $(\mathbf{0.45_{\pm .03}}, \mathbf{0.67_{\pm .03}})$
& $(\mathbf{0.52_{\pm .02}}, \mathbf{0.86_{\pm .02}})$ \\
\bottomrule
\end{tabular}}
\label{tab:simulation blobs bands entrywise}
\end{table}

\begin{table}[h!]
\centering
\caption{Comparisons to \textit{low-dimensional} generative baselines on Waves and Crosshatch simulations (entry-wise moment discrepancies: AbsEntryMeanDiff and AbsEntryStdDiff; lower is better). Column headers denote missing-entry rate. Parentheses report each method's learned-dimension ratio. All values are in $\times 10^{-2}$.}
\renewcommand{\arraystretch}{1.2}
\setlength{\tabcolsep}{4pt}

\scalebox{0.55}{
\begin{tabular}{lcccccc}
\toprule
& \multicolumn{3}{c}{\textbf{Waves}}
& \multicolumn{3}{c}{\textbf{Crosshatch}} \\
\cmidrule(lr){2-4}\cmidrule(lr){5-7}

$p_{\mathrm{miss}}$
& 0\% & 20\% & 40\%
& 0\% & 20\% & 40\% \\
\midrule

CSDM ($\rho=1.44\%$)
& $(44.1_{\pm 1.8}, 8.22_{\pm .75})$
& $(45.2_{\pm 1.4}, 7.14_{\pm .54})$
& $(46.2_{\pm 1.0}, 6.91_{\pm .45})$
& $(31.9_{\pm 1.1}, 7.97_{\pm .17})$
& $(32.7_{\pm 1.0}, 7.21_{\pm .15})$
& $(33.5_{\pm 1.5}, 7.00_{\pm .16})$ \\
CSDM ($\rho=36\%$)
& $(26.0_{\pm .86}, 3.78_{\pm .68})$
& $(29.2_{\pm 1.1}, 6.40_{\pm .75})$
& $(33.7_{\pm .99}, 7.96_{\pm .81})$
& $(18.7_{\pm 1.1}, 10.9_{\pm .36})$
& $(20.9_{\pm 1.2}, 4.72_{\pm .11})$
& $(24.2_{\pm 1.2}, 5.67_{\pm .15})$ \\
CSDM ($\rho=64\%$)
& $(18.5_{\pm .79}, 2.36_{\pm .53})$
& $(21.9_{\pm 1.1}, 8.06_{\pm .92})$
& $(27.8_{\pm 1.1}, 10.6_{\pm .95})$
& $(13.2_{\pm 1.4}, 2.23_{\pm .08})$
& $(15.7_{\pm 1.4}, 5.27_{\pm .13})$
& $(20.0_{\pm 1.4}, 6.84_{\pm .15})$ \\
SMG-Core ($\rho=1.44\%$)
& $(13.1_{\pm .40}, 1.70_{\pm .21})$
& $(17.6_{\pm .47}, 1.99_{\pm .28})$
& $(20.6_{\pm .57}, 2.04_{\pm .39})$
& $(12.7_{\pm .85}, 1.53_{\pm .05})$
& $(14.3_{\pm .94}, 1.87_{\pm .11})$
& $(17.3_{\pm 1.2}, 1.84_{\pm .19})$ \\
\rowcolor{coreflowbg}
\textbf{CoreFlow} ($\rho=1.44\%$)
& $(\mathbf{0.97_{\pm .09}}, \mathbf{1.11_{\pm .13}})$
& $(\mathbf{0.82_{\pm .09}}, \mathbf{2.65_{\pm .13}})$
& $(\mathbf{1.59_{\pm .28}}, \mathbf{1.37_{\pm .14}})$
& $(\mathbf{0.93_{\pm .22}}, \mathbf{0.74_{\pm .07}})$
& $(\mathbf{0.85_{\pm .23}}, \mathbf{0.96_{\pm .09}})$
& $(\mathbf{0.89_{\pm .14}}, \mathbf{0.97_{\pm .07}})$ \\
\bottomrule
\end{tabular}}
\label{tab:simulation waves crosshatch entrywise}
\end{table}

\subsection{Compression-matched comparison on Solar Flare}

\begin{table}[h!]
\centering
\caption{Comparisons to \textit{low-dimensional} baselines on Solar (Frobenius-norm distribution shift: FrobMeanDiff and FrobStdDiff; lower is better). Column headers denote missing-entry rate. Parentheses report each method's learned-dimension ratio.}
\renewcommand{\arraystretch}{1.2}
\setlength{\tabcolsep}{4pt}

\scalebox{0.55}{
\begin{tabular}{lcccccc}
\toprule
& \multicolumn{3}{c}{\textbf{Solar}}
& \multicolumn{3}{c}{\textbf{Solar 2}} \\
\cmidrule(lr){2-4}\cmidrule(lr){5-7}

$p_{\mathrm{miss}}$
& 0\% & 20\% & 40\%
& 0\% & 20\% & 40\% \\
\midrule

CSDM ($\rho=9\%$)        & $(52.9_{\pm 1.8}, 0.98_{\pm .30})$ & $(56.4_{\pm 1.6}, 1.03_{\pm .40})$ & $(60.3_{\pm 1.5}, 1.08_{\pm .62})$ & $(65.9_{\pm 3.8}, 1.81_{\pm 1.0})$ & $(70.2_{\pm 3.3}, 1.92_{\pm 1.1})$ & $(75.1_{\pm 3.1}, 2.04_{\pm 1.1})$ \\
CSDM ($\rho=36\%$)       & $(26.6_{\pm 3.5}, 0.52_{\pm .24})$ & $(32.8_{\pm 3.3}, 0.62_{\pm .25})$ & $(39.9_{\pm 3.0}, 0.72_{\pm .29})$ & $(33.7_{\pm 6.6}, 1.06_{\pm .48})$ & $(41.4_{\pm 6.1}, 1.25_{\pm .47})$ & $(50.2_{\pm 5.5}, 1.43_{\pm .58})$ \\
SMG-Core ($\rho=9\%$)    & $(2.15_{\pm .87}, 0.20_{\pm .06})$ & $(3.02_{\pm .90}, 0.28_{\pm .08})$ & $(3.79_{\pm 1.0}, 0.36_{\pm .18})$ & $(1.03_{\pm .31}, 1.09_{\pm .32})$ & $(1.92_{\pm .34}, 1.30_{\pm .37})$ & $(2.03_{\pm 45}, 1.54_{\pm .58})$ \\
\rowcolor{coreflowbg}
\textbf{CoreFlow} ($\rho=9\%$) & $(\mathbf{0.33_{\pm .19}}, \mathbf{0.14_{\pm .11}})$ & $(\mathbf{0.53_{\pm .08}}, \mathbf{0.05_{\pm .04}})$ & $(\mathbf{0.65_{\pm .10}}, \mathbf{0.11_{\pm .06}})$ & $(\mathbf{0.44_{\pm .29}}, \mathbf{0.93_{\pm .22}})$ & $(\mathbf{0.61_{\pm 13}}, \mathbf{0.40_{\pm .11}})$ & $(\mathbf{0.70_{\pm .13}}, \mathbf{0.74_{\pm .12}})$ \\
\bottomrule
\end{tabular}}
\label{tab:solar frob}
\end{table}

\begin{table}[h!]
\centering
\caption{Comparisons to \textit{low-dimensional} baselines on Solar (entry-wise moment discrepancies: AbsEntryMeanDiff and AbsEntryStdDiff; lower is better). Column headers denote missing-entry rate. Parentheses report each method's learned-dimension ratio. All values are in $\times 10^{-2}$.}
\renewcommand{\arraystretch}{1.2}
\setlength{\tabcolsep}{4pt}

\scalebox{0.55}{
\begin{tabular}{lcccccc}
\toprule
& \multicolumn{3}{c}{\textbf{Solar}}
& \multicolumn{3}{c}{\textbf{Solar 2}} \\
\cmidrule(lr){2-4}\cmidrule(lr){5-7}

$p_{\mathrm{miss}}$
& 0\% & 20\% & 40\%
& 0\% & 20\% & 40\% \\
\midrule

CSDM ($\rho=9\%$)        & $(26.8_{\pm 1.3}, 2.91_{\pm .40})$ & $(28.1_{\pm 1.4}, 4.18_{\pm .22})$ & $(29.6_{\pm 1.4}, 4.82_{\pm .29})$ & $(34.9_{\pm 1.7}, 2.21_{\pm .38})$ & $(36.6_{\pm 1.9}, 4.64_{\pm .45})$ & $(38.6_{\pm 1.6}, 5.74_{\pm .57})$ \\
CSDM ($\rho=36\%$)       & $(19.6_{\pm 1.4}, 2.30_{\pm .31})$ & $(21.4_{\pm 1.5}, 6.53_{\pm .35})$ & $(23.9_{\pm 1.6}, 8.41_{\pm .47})$ & $(25.3_{\pm 1.6}, 1.76_{\pm .50})$ & $(27.9_{\pm 1.1}, 8.55_{\pm .71})$ & $(31.4_{\pm 1.1}, 11.4_{\pm .92})$ \\
SMG-Core ($\rho=9\%$)    & $(1.54_{\pm .49}, 1.44_{\pm .26})$ & $(1.76_{\pm .46}, 1.47_{\pm .21})$ & $(1.81_{\pm .51}, 1.36_{\pm .34})$ & $(1.30_{\pm .15}, 1.27_{\pm .19})$ & $(1.34_{\pm 16}, 1.32_{\pm .19})$ & $(1.48_{\pm .19}, 1.38_{\pm .20})$ \\
\rowcolor{coreflowbg}
\textbf{CoreFlow} ($\rho=9\%$) & $(\mathbf{0.77_{\pm .21}}, \mathbf{0.83_{\pm .10}})$ & $(\mathbf{0.80_{\pm .08}}, \mathbf{0.86_{\pm .03}})$ & $(\mathbf{0.91_{\pm .10}}, \mathbf{0.93_{\pm .15}})$ & $(\mathbf{0.86_{\pm .05}}, \mathbf{0.71_{\pm .06}})$ & $(\mathbf{0.98_{\pm .07}}, \mathbf{0.69_{\pm .04}})$ & $(\mathbf{1.05_{\pm .09}}, \mathbf{0.78_{\pm .04}})$ \\
\bottomrule
\end{tabular}}
\label{tab:solar entrywise}
\end{table}

Tables~\ref{tab:solar frob} and~\ref{tab:solar entrywise} directly compare CoreFlow and low-dimensional baselines at the \emph{same} compression ratio $\rho=9\%$ on Solar and Solar~2, isolating the contribution of the modeling pipeline beyond dimensionality reduction.
At matched $\rho$, CoreFlow dramatically reduces both Frobenius-norm distribution shift (Table~\ref{tab:solar frob}) and entry-wise moment discrepancies (Table~\ref{tab:solar entrywise}) across all missing rates.

\subsection{Stage-I subspace recovery diagnostic via principal angles}
Finally, Table~\ref{tab:simulation all UV} shows that CoreFlow recovers the ground-truth row/column subspaces accurately, with small mean and maximum principal angles even at $p_{\mathrm{miss}}=40\%$.

\begin{table}[h!]
\centering
\caption{Simulation: \textbf{CoreFlow} principal-angle discrepancy in degrees. Each entry reports
$(\overline{\theta}_U/\theta_{U,\max},\ \overline{\theta}_V/\theta_{V,\max})$.
All values lie in $[0,90]$; lower is better. Columns denote the missing-entry rate $p_{\mathrm{miss}}$.}
\renewcommand{\arraystretch}{1.15}
\setlength{\tabcolsep}{6pt}
\scalebox{0.60}{
\begin{tabular}{l:ccc}
\toprule
& \multicolumn{3}{c}{\textbf{$p_{\mathrm{miss}}$}} \\
\cmidrule(lr){2-4}
\textbf{Case} & 0\% & 20\% & 40\% \\
\midrule
\textbf{Blobs} $(\rho=1.44\%)$
& $(0.00/0.00,\ 0.00/0.00)$
& $(0.01/0.07,\ 0.04/0.13)$
& $(0.04/0.14,\ 0.03/0.15)$ \\
\textbf{Bands} $(\rho=1.44\%)$
& $(0.00/0.00,\ 0.00/0.00)$
& $(0.28/0.74,\ 0.27/0.74)$
& $(0.58/1.98,\ 0.58/1.97)$ \\
\textbf{Waves} $(\rho=1.44\%)$
& $(0.00/0.00,\ 0.00/0.00)$
& $(0.04/0.33,\ 0.05/0.89)$
& $(0.05/0.40,\ 0.08/0.89)$ \\
\textbf{Crosshatch} $(\rho=1.44\%)$
& $(0.00/0.00,\ 0.00/0.00)$
& $(0.20/4.63,\ 0.21/4.62)$
& $(0.24/4.63,\ 0.22/4.63)$ \\
\bottomrule
\end{tabular}}
\label{tab:simulation all UV}
\end{table}

\subsection{Maximum Mean Discrepancy (MMD) metrics}

\begin{table}[h!]
\centering
\caption{
\small Comparison with baselines that explicitly handle missing data on real datasets, measured by MMD (lower is better).
Columns report the missing-entry rate $p_{\mathrm{miss}}$.
Entries report MMD values.
CoreFlow uses a learned-dimension ratio of $\rho=36\%$ for Solar and Solar 2 and $\rho=49\%$ for LSPF, while all ambient-space baselines use $\rho=100\%$.
Solar and Solar 2 correspond to few-sample settings ($n=300$), whereas LSPF is a standard-sample setting ($n=8760$).
Empty entries indicate results not yet collected.
}
\label{tab:real_mmd}
\renewcommand{\arraystretch}{1.12}
\setlength{\tabcolsep}{4pt}

\scalebox{0.66}{
\begin{tabular}{l ccc ccc ccc}
\toprule
& \multicolumn{3}{c}{\textbf{Solar}}
& \multicolumn{3}{c}{\textbf{Solar 2}}
& \multicolumn{3}{c}{\textbf{LSPF}} \\
\cmidrule(lr){2-4} \cmidrule(lr){5-7} \cmidrule(lr){8-10}
$\boldsymbol{p_{\mathrm{miss}}}$
& \textbf{0\%} & \textbf{20\%} & \textbf{40\%}
& \textbf{0\%} & \textbf{20\%} & \textbf{40\%}
& \textbf{0\%} & \textbf{20\%} & \textbf{40\%} \\
\midrule
MissDiff ($\rho=100\%$)
& $0.893$ & $0.917$ & $1.044$
& $0.593$ & $0.915$ & $0.976$
& $\mathbf{0.116}$ & $0.685$ & $0.811$ \\

MissFlow ($\rho=100\%$)
& $0.882$ & $0.786$ & $0.866$
& $0.517$ & $0.932$ & $0.923$
& $\mathbf{0.118}$ & $0.661$ & $0.814$ \\

\rowcolor{coreflowbg}
\textbf{CoreFlow} ($\rho=36\%/49\%$)
& $\mathbf{0.177}$ & $\mathbf{0.193}$ & $\mathbf{0.249}$
& $\mathbf{0.187}$ & $\mathbf{0.267}$ & $\mathbf{0.360}$
& $0.122$ & $\mathbf{0.115}$ & $\mathbf{0.120}$ \\
\bottomrule
\end{tabular}
}
\end{table}

\begin{table}[h!]
\centering
\caption{
\small Comparison with low-dimensional baselines on the Solar datasets, measured by MMD (lower is better).
Columns report the missing-entry rate $p_{\mathrm{miss}}$.
Entries report MMD values.
Parentheses indicate the learned-dimension ratio $\rho$.
Empty entries indicate results not yet collected.
}
\label{tab:solar_mmd}
\renewcommand{\arraystretch}{1.12}
\setlength{\tabcolsep}{5pt}

\scalebox{0.66}{
\begin{tabular}{l ccc ccc}
\toprule
& \multicolumn{3}{c}{\textbf{Solar}}
& \multicolumn{3}{c}{\textbf{Solar 2}} \\
\cmidrule(lr){2-4} \cmidrule(lr){5-7}
$\boldsymbol{p_{\mathrm{miss}}}$
& \textbf{0\%} & \textbf{20\%} & \textbf{40\%}
& \textbf{0\%} & \textbf{20\%} & \textbf{40\%} \\
\midrule
CSDM ($\rho=9\%$)
& $1.197$ & $1.314$ & $1.367$
& $1.206$ & $1.351$ & $1.373$ \\

CSDM ($\rho=36\%$)
& $0.842$ & $0.873$ & $0.905$
& $0.857$ & $0.913$ & $0.970$ \\

SMG-Core ($\rho=36\%$)
& $0.224$ & $0.265$ & $0.293$
& $0.219$ & $0.285$ & $0.426$ \\

\rowcolor{coreflowbg}
\textbf{CoreFlow} ($\rho=36\%$)
& $\mathbf{0.177}$ & $\mathbf{0.193}$ & $\mathbf{0.249}$
& $\mathbf{0.187}$ & $\mathbf{0.267}$ & $\mathbf{0.360}$ \\
\bottomrule
\end{tabular}
}
\end{table}

\begin{table}[h!]
\centering
\caption{
\small Comparison with low-dimensional generative baselines on simulation datasets, measured by MMD (lower is better).
Columns report the missing-entry rate $p_{\mathrm{miss}}$.
Entries report MMD values.
Parentheses indicate the learned-dimension ratio $\rho$.
Empty entries indicate results not yet collected.
}
\label{tab:simulation_mmd}
\renewcommand{\arraystretch}{1.08}
\setlength{\tabcolsep}{3pt}

\scalebox{0.58}{
\begin{tabular}{l ccc ccc ccc ccc}
\toprule
& \multicolumn{3}{c}{\textbf{Blobs}}
& \multicolumn{3}{c}{\textbf{Bands}}
& \multicolumn{3}{c}{\textbf{Waves}}
& \multicolumn{3}{c}{\textbf{Crosshatch}} \\
\cmidrule(lr){2-4} \cmidrule(lr){5-7} \cmidrule(lr){8-10} \cmidrule(lr){11-13}
$\boldsymbol{p_{\mathrm{miss}}}$
& \textbf{0\%} & \textbf{20\%} & \textbf{40\%}
& \textbf{0\%} & \textbf{20\%} & \textbf{40\%}
& \textbf{0\%} & \textbf{20\%} & \textbf{40\%}
& \textbf{0\%} & \textbf{20\%} & \textbf{40\%} \\
\midrule
CSDM ($\rho=1.44\%$)
& $1.206$ & $1.389$ & $1.402$
& $1.071$ & $1.293$ & $1.335$
& $0.913$ & $1.136$ & $1.252$
& $0.988$ & $0.993$ & $1.102$ \\

CSDM ($\rho=64\%$)
& $0.836$ & $0.897$ & $0.943$
& $0.851$ & $0.906$ & $0.988$
& $0.495$ & $0.561$ & $0.607$
& $0.503$ & $0.599$ & $0.684$ \\

SMG-Core ($\rho=1.44\%$)
& $0.814$ & $0.970$ & $0.995$
& $0.702$ & $0.747$ & $0.804$
& $0.276$ & $0.314$ & $0.380$
& $0.220$ & $0.289$ & $0.319$ \\

SMG-Core ($\rho=64\%$)
& $0.312$ & $0.364$ & $0.411$
& $0.323$ & $0.324$ & $0.365$
& $0.127$ & $0.249$ & $0.299$
& $0.100$ & $0.192$ & $0.237$ \\

\rowcolor{coreflowbg}
\textbf{CoreFlow} ($\rho=1.44\%$)
& $\mathbf{0.004}$ & $\mathbf{0.010}$ & $\mathbf{0.017}$
& $\mathbf{0.006}$ & $\mathbf{0.117}$ & $\mathbf{0.212}$
& $\mathbf{0.084}$ & $\mathbf{0.158}$ & $\mathbf{0.101}$
& $\mathbf{0.073}$ & $\mathbf{0.082}$ & $\mathbf{0.095}$ \\
\bottomrule
\end{tabular}
}
\end{table}

\subsection{Ablation Studies: Stage-I Representation with PCA-Flow}
\label{appendix:ablation PCA}

To isolate the effect of our matrix-structured Stage-I representation, we further consider a
PCA-Flow ablation that follows the same two-stage low-dimensional flow idea as CoreFlow, but
replaces the learned row/column subspaces with a standard PCA subspace on flattened matrices.
Given complete training matrices \(M_i\in\mathbb{R}^{m_1\times m_2}\), let
\(x_i=\operatorname{vec}(M_i)\in\mathbb{R}^{D}\), where \(D=m_1m_2\), and define
\[
\mu=\frac{1}{N}\sum_{i=1}^N x_i,\qquad
X_c =
\begin{bmatrix}
(x_1-\mu)^\top\\
\cdots\\
(x_N-\mu)^\top
\end{bmatrix}.
\]
We compute the economy SVD \(X_c=A\Sigma V^\top\) and retain the leading
\(d_{\rm PCA}\) right singular vectors
\[
W_{\rm PCA}=[v_1,\ldots,v_{d_{\rm PCA}}]\in\mathbb{R}^{D\times d_{\rm PCA}},
\qquad
d_{\rm PCA}\le \min\{R^2,N-1\},
\]
with an additional rounding to the nearest perfect square when using the U-Net velocity field.
Each matrix is then encoded as a PCA score
\[
z_i=W_{\rm PCA}^\top(x_i-\mu)\in\mathbb{R}^{d_{\rm PCA}},
\]
normalized as \(\tilde z_i=(z_i-\bar z)/\sigma_z\), and the same CNF architecture is trained in
this PCA-score space by flow matching:
\[
\epsilon\sim\mathcal{N}(0,I),\qquad
t\sim{\rm Unif}(0,1),\qquad
x_t=(1-t)\epsilon+t\tilde z_i,
\]
\[
\mathcal{L}_{\rm PCA\text{-}Flow}(\theta)
=
\mathbb{E}_{i,t,\epsilon}
\left[
\left\|
v_\theta(x_t,t)-(\tilde z_i-\epsilon)
\right\|_2^2
\right].
\]
At generation time, we sample \(\epsilon\sim\mathcal{N}(0,I)\), solve the learned ODE from
\(t=0\) to \(t=1\) to obtain a generated normalized PCA score \(\hat{\tilde z}\), de-normalize it by
\(\hat z=\sigma_z\hat{\tilde z}+\bar z\), and reconstruct the matrix by
\[
\hat M
=
\operatorname{mat}\left(\mu+W_{\rm PCA}\hat z\right).
\]

\begin{table}[h!]
\centering
\caption{
\small Ablation comparison on LSPF between CoreFlow and PCA-Flow using three metrics:
SVRelL2 ($\times 10^{-2}$), FrobMeanDiff, and AbsEntryMeanDiff ($\times 10^{-2}$). This table is for the complete-data case
($p_{\mathrm{miss}}=0$).
CoreFlow uses the matrix-structured representation $M\approx USV^\top$ with $\rho=49\%$,
while PCA-Flow replaces Stage I with flattened PCA and trains the same CNF on PCA scores.
}
\label{tab:lspf_metrics}
\renewcommand{\arraystretch}{1.12}
\setlength{\tabcolsep}{6pt}

\scalebox{0.82}{
\begin{tabular}{l ccc}
\toprule
& \multicolumn{3}{c}{\textbf{LSPF}} \\
\cmidrule(lr){2-4}
\textbf{Method}
& \textbf{SVRelL2}
& \textbf{FrobMeanDiff}
& \textbf{AbsEntryMeanDiff} \\
\midrule

PCA-Flow ($\rho=49\%$)
& $\mathbf{2.54_{\pm 0.23}}$
& $0.34_{\pm 0.15}$
& $2.14_{\pm .08}$ \\

\rowcolor{coreflowbg}
\textbf{CoreFlow} ($\rho=49\%$)
& $2.88_{\pm 0.19}$
& $\mathbf{0.24_{\pm .09}}$
& $\mathbf{1.87_{\pm .11}}$ \\

\bottomrule
\end{tabular}
}
\end{table}

\begin{table}[h!]
\centering
\caption{
\small Ablation comparison on simulation datasets between CoreFlow and PCA-Flow using three metrics:
SVRelL2 ($\times 10^{-2}$), FrobMeanDiff, and AbsEntryMeanDiff ($\times 10^{-2}$).
This table is for the complete-data case
($p_{\mathrm{miss}}=0$).
CoreFlow uses the matrix-structured representation $M\approx USV^\top$ with $\rho=1.44\%$,
while PCA-Flow replaces Stage I with flattened PCA and trains the same CNF on PCA scores.
}
\label{tab:simulation_pcaflow_ablation}
\renewcommand{\arraystretch}{1.12}
\setlength{\tabcolsep}{3.5pt}

\scalebox{0.56}{
\begin{tabular}{l ccc ccc ccc ccc}
\toprule
& \multicolumn{3}{c}{\textbf{Blobs}}
& \multicolumn{3}{c}{\textbf{Bands}}
& \multicolumn{3}{c}{\textbf{Waves}}
& \multicolumn{3}{c}{\textbf{Crosshatch}} \\
\cmidrule(lr){2-4} \cmidrule(lr){5-7} \cmidrule(lr){8-10} \cmidrule(lr){11-13}
\textbf{Method}
& {\tiny \textbf{SVRelL2}} & {\tiny \textbf{FrobMeanDiff}} & {\tiny \textbf{AbsEntryMeanDiff}}
& {\tiny \textbf{SVRelL2}} & {\tiny \textbf{FrobMeanDiff}} & {\tiny \textbf{AbsEntryMeanDiff}}
& {\tiny \textbf{SVRelL2}} & {\tiny \textbf{FrobMeanDiff}} & {\tiny \textbf{AbsEntryMeanDiff}}
& {\tiny \textbf{SVRelL2}} & {\tiny \textbf{FrobMeanDiff}} & {\tiny \textbf{AbsEntryMeanDiff}} \\
\midrule

PCA-Flow ($\rho=1.44\%$)
& $1.77_{\pm 0.60}$ & $0.55_{\pm 0.25}$ & $2.94_{\pm 0.08}$
& $24.45_{\pm 1.26}$ & $0.33_{\pm 0.21}$ & $0.71_{\pm 0.12}$
& $7.74_{\pm 0.98}$ & $2.39_{\pm 2.16}$ & $1.76_{\pm 0.76}$
& $8.07_{\pm 0.92}$ & $1.71_{\pm 1.29}$ & $1.56_{\pm 0.37}$ \\

\rowcolor{coreflowbg}
\textbf{CoreFlow} ($\rho=1.44\%$)
& $\mathbf{0.43_{\pm 0.12}}$ & $\mathbf{0.09_{\pm .06}}$ & $\mathbf{0.98_{\pm .04}}$
& $\mathbf{4.45_{\pm 0.74}}$ & $\mathbf{0.08_{\pm .06}}$ & $\mathbf{0.29_{\pm .03}}$
& $\mathbf{3.47_{\pm 0.17}}$ & $\mathbf{0.55_{\pm .38}}$ & $\mathbf{0.97_{\pm .09}}$
& $\mathbf{3.23_{\pm 0.50}}$ & $\mathbf{0.96_{\pm .31}}$ & $\mathbf{0.93_{\pm .22}}$ \\

\bottomrule
\end{tabular}
}
\end{table}

\textbf{Key findings.} PCA-Flow is a useful ablation because it follows the same high-level principle as CoreFlow. In complete-data settings, PCA-Flow can obtain similar,
though slightly weaker generation quality, suggesting that low-dimensional flow modeling itself is
important. However, PCA-Flow is a flattened and geometry-agnostic representation: it treats each
matrix as a vector and reconstructs only through a global linear PCA subspace. Consequently, it does
not preserve the row/column factorization structure \(USV^\top\) and does not explicitly separate shared
row and column subspaces from sample-specific core variation. Besides, it cannot naturally handle incomplete training matrices because PCA requires complete vectors before the projection
\(W_{\rm PCA}^\top(x-\mu)\) is well-defined. By contrast, CoreFlow learns a matrix-structure-aware
representation \(M\approx USV^\top\), trains the CNF on the induced core \(S=U^\top M V\), and
naturally supports missing entries through the masked Stage-I objective and alternating completion.
Overall, the ablation study PCA-Flow confirms the value of low-dimensional flow modeling, while the stronger CoreFlow
architecture is essential for geometry-preserving generation and robustness under missingness.

\end{document}